\documentclass{lmcs}
\pdfoutput=1
\usepackage[utf8]{inputenc}

\usepackage{lastpage}
\lmcsdoi{19}{4}{11}
\lmcsheading{}{\pageref{LastPage}}{}{}%
{Jun.~22,~2022}{Nov.~17,~2023}{}

\keywords{categorial grammars, linear logic, Lambek calculus, diagrammatic reasoning}
 \usepackage{amsmath, amssymb,cmll} %,amsthm,txfonts
\usepackage{amsthm,xfrac,nicefrac}
\usepackage{hyperref}
\usepackage{bussproofs}
\usepackage{enumitem}
\usepackage{pgf}
\usepackage{tikz}
\usepackage{subcaption}
\usepackage{float}
\restylefloat{figure}
\newcommand{\be}{\begin{equation}}
\newcommand{\ee}{\end{equation}}

\begin{document}
\title{Making first order linear logic\texorpdfstring{\\}{ }a generating grammar}
\author[S.~Slavnov]{Sergey Slavnov\lmcsorcid{0000-0001-8825-7318}}
\email{slavnovserge@gmail.com}

 \begin{abstract}
It is known that different categorial grammars have surface representation in a fragment of first order multiplicative linear logic ({\bf MLL1}). We show that the fragment of interest is equivalent to the recently introduced {\it extended tensor type calculus} ({\bf ETTC}).
 {\bf ETTC} is a calculus of specific typed terms, which represent tuples of strings,
 more precisely bipartite graphs decorated with strings. Types are derived from linear logic formulas, and rules
 correspond to concrete operations on these string-labeled graphs, so that they can be conveniently visualized. 
 This provides the above mentioned fragment of {\bf MLL1} that is relevant for language modeling not only with some alternative syntax and intuitive geometric representation, but also with an intrinsic deductive system, which has been absent.
 
 In this work we consider a non-trivial notationally enriched variation of the previously introduced {\bf ETTC}, which allows more concise and transparent computations. We present both a cut-free sequent calculus and a natural deduction formalism.

\end{abstract}

\maketitle

\section{Introduction}
The best known examples of  categorial  grammars are  Lambek grammars, which are based on Lambek calculus ({\bf LC}) \cite{Lambek}, i.e. noncommutative intuitionistic  linear logic (for background on linear logic see \cite{Girard_TCS},\cite{Girard}, for categorial grammars \cite{MootRetore}). These, however, have somewhat limited expressive power, and a lot of extensions/variations have been proposed, using discontinuous tuples of strings and $\lambda$-terms, commutative and non-commutative logical operations, modalities etc. Let us mention displacement grammars \cite{Morrill_Displacement}, abstract categorial grammars  ({\bf ACG}, also called $\lambda$-grammars and linear grammars) \cite{deGroote}, \cite{Muskens}, \cite{PollardMichalicek} and hybrid type logical categorial grammars \cite{KubotaLevine}.

It has been known for a while that different grammatical formalisms, such as those just mentioned, have surface representation in the familiar
{first order multiplicative intuitionistic linear logic} ({\bf MILL1}), which provides them in this way with some common ground.
 In the  seminal work \cite{MootPiazza}  it was shown that   {\bf LC} translates into {\bf MILL1} as a conservative fragment; later  similar representations were obtained for other systems \cite{Moot_extended}, \cite{Moot_inadequacy} (however, in the case of displacement calculus, the representation  is {\it not conservative}, see \cite{Moot_comparing}). This applies equally well to the {\it classical}  setting of first order multiplicative linear logic ({\bf MLL1}), which is a conservative extension of {\bf MILL1} anyway.

  The translations are based on interpreting  first order variables as denoting {\it string positions}. Basically, this induces (at least on the intuitive level) an interpretation of certain well-formed formulas and sequents as expressing relations between strings,  in particular, those encoded in formal grammars.
  Unfortunately, this intuitive  interpretation (to the author's knowledge) has never been formulated explicitly and rigorously, i.e. as a specific {\it semantics} for (a fragment of) {\bf MILL1}. Therefore the underlying structures for constructions of the above cited works,  arguably, remain somewhat unclear.

   It should be emphasised that  only a small  fragment of {\bf MILL1}, let us call it {\it linguistically well-formed},
    is actually  used in the above translations. It is this fragment that obtains in this way the above discussed ``linguistic'' interpretation in terms of relations between  strings. As for the whole {\bf MILL1}, the ``linguistic'' meaning of general formulas and sequents seems, at least, obscure.
    On the other hand, standard deductive systems  of linear logic are not {\it intrinsic} to this fragment.
  Typically, when deriving, say, the {\bf MILL1} translation of an {\bf LC} sequent, in general, we have to use at intermediate steps sequents  that have no
clear linguistic meaning, if any. It can be said that representation of formal grammars makes first order linear logic an auxiliary tool, but hardly something like a generating grammar. Derivations do not translate  step-by-step  to operations generating language elements.

(We should note, however, that the first order logic formalism has, in general, wider usage in linguistic applications. Apart from encoding string positions, first order constructions allow representing grammatical features, locality
domains, tree depth levels etc. We do not discuss the latter applications in our paper.)

The system of {\it extended tensor type calculus} ({\bf ETTC}) based on (classical) propositional {\bf MLL}, recently proposed in \cite{Slavnov_tensor} (elaborating on \cite{Slavnov_cowordisms}), was directly designed for linguistic interpretation, namely for concrete geometric representation of {\bf ACG} extended with constructions emulating noncommutative operations of {\bf LC}.
 {\bf ETTC} is a calculus of specific typed terms, which represent tuples of strings,
 more precisely bipartite graphs decorated with strings. Types are derived from {\bf MLL} formulas, and rules
 correspond to concrete operations on these string-labeled graphs, so that they can be conveniently visualized. Typed terms become language elements, from which words are gradually assembled using these geometric operations.

It was shown in \cite{Slavnov_tensor} that {\it tensor grammars} based on {\bf ETTC} include both Lambek grammars and {\bf ACG} as conservative fragments. Moreover, the representation of {\bf LC} types in {\bf ETTC} is very similar to that in {\bf MILL1}. This led to the question (raised explicitly in that paper) about relations between the two systems.
  In this work we show that {\bf ETTC} is equivalent precisely to the above discussed linguistically well-formed fragment of  {\bf MLL1}. (Actually, we consider a non-trivial notationally enriched variation of {\bf ETTC} rather than the one proposed in \cite{Slavnov_tensor}. The enriched system allows more concise and transparent computations.)
Thus, it turns out  that {\bf ETTC} provides the linguistic fragment  with an alternative syntax, together with an intrinsic deductive system and intuitive pictorial representation. It can be said that, in some sense, {\bf ETTC} does ``make the linguistic fragment a generating grammar'', as announced in the title. Speaking less ambitiously, it allows a different perspective on the {\bf MILL1} representations, which might be clarifying and useful.

  In this paper we will work on the basis of classical linear logic, mostly because it suggests
  some
  nice symmetric notation, but all constructions automatically restrict to the intuitionistic setting.

\section{Background: systems of linear logic}
\begin{figure}[htb]
\centering
%%%%%%%%%%%%%%%%%%%%%%%%%%%%%%%%%%%%%%%%%%%%%%%%%%%%%%%%%%%%%%%%%%%%%%%%%%%%%%%%%%%
\subfloat[{\bf{MLL1} language}
\label{MLL1 formulas}
]
{
$
\begin{array}{c}
\mathit{Pred}_-=\{\overline{p}|~p\in \mathit{Pred}_+\},\quad p\mbox{ is }n\mbox{-ary}\Rightarrow{\overline{p}}\mbox{ is }n\mbox{-ary},\quad \mathit{Pred}=\mathit{Pred}_+\cup \mathit{Pred}_-,\\[.05cm]
\mathit{At}=\{p(e_1,\ldots,e_n)|~e_1,\ldots,e_n\in \mathit{Var},~p\in \mathit{Pred},~p\mbox{ is }n\mbox{-ary}\},\\[.05cm]
\mathit{Fm}::=\mathit{At}|(\mathit{Fm}\otimes \mathit{Fm})|(\mathit{Fm}\parr \mathit{Fm})|\forall x\mathit{Fm}|\exists x\mathit{Fm},~
x\in \mathit{Var}\\[.15cm]
\overline{p(e_1,\ldots,e_n)}=\overline{p}(e_n,\ldots,e_1),\quad
\overline{\overline{p(e_1,\ldots,e_n)}}=p(e_n,\ldots,e_1)\mbox{ for }P\in Pred_+,\\[.1cm]
\overline{A\otimes B}=\overline{B}\parr \overline{A},\quad \overline{A\parr B}=\overline{B}\otimes \overline{A},\quad
\overline{\forall x A}=\exists x(\overline{A}),\quad\overline{\exists x A}=\forall x(\overline{A}).
\end{array}$
}\\[.4cm]
%%%%%%%%%%%%%%%%%%%%%%%%%%%%%%%%%%%%%%%%%%%%%%%%%%%%%%%%%%%%%%%%%%%%%%%%%%%%%%%%%%%%%%
\subfloat[{{\bf MLL1} sequent calculus}
\label{MLL1}
]
    {
    $
        \begin{array}{c}
            \cfrac{A\in At}{\vdash \overline{A},A}~(\mbox{Id})
            \quad
            \cfrac{\vdash \Gamma,A\quad\vdash
            \overline{A},\Theta}{\vdash\Gamma,\Theta} ~(\mbox{Cut})
                        \\[.25cm]
            \cfrac{\vdash \Gamma,A,B}
            {\vdash\Gamma,A\parr B}~
            (\parr)
            ~%\quad
            \cfrac{\vdash\Gamma, A
            ~%\quad
            \vdash
            B,\Theta}{\vdash\Gamma,A\otimes B,\Theta}~{ }
            (\otimes)
            ~%\quad
            \cfrac{\vdash\Gamma, A[\nicefrac{v}{x}],~v\not\in FV(\Gamma)}{\vdash \Gamma,\forall x A}~(\forall)
            ~%\quad
            \cfrac{\vdash\Gamma, A[\nicefrac{v}{x}]}{\vdash \Gamma,\exists x A}~(\exists)
        \end{array}
    $
    }
%%%%%%%%%%%%%%%%%%%%%%%%%%%%%%%%%%%%%%
\\[.4cm]
\subfloat[{{\bf LC} language}
\label{LC_rules}
]
{
$
\mathit{Fm}::=\mathit{Prop}|(\mathit{Fm}\backslash \mathit{Fm})|(\mathit{Fm}/\mathit{Fm})| (\mathit{Fm}\bullet \mathit{Fm})
$
}\\[.4cm]
%%%%%%%%%%%%%%%%%%%
\subfloat[{Translating {\bf LC} to {\bf MILL1}}
\label{LC2MILL}
]
{
%{\textwidth}
%\centering
$
\begin{array}{c}
\begin{array}{lcr}
||p||^{(l,r)}=p(l,r)\mbox{ for }p\in \mathit{Prop},
&&
||B/ A||^{(l,r)}=\forall x||A||^{(r,x)}\multimap ||B||^{(l,x)},
\\[.05cm]
 ||A\bullet B||^{(l,r)}=\exists x||A||^{(l,x)}\otimes ||B||^{(x,r)},
 &&
 ||A\backslash B||^{(l,r)}=\forall x||A||^{(x,l)}\multimap ||B||^{(x,r)}.
 \end{array}
\\[.1cm]
\begin{array}{rcl}
A_1,\ldots, A_n\vdash_{\bf LC} B&
\Longleftrightarrow&||A_1||^{(c_0,c_1)},\ldots,||A_n||^{(c_{n-1},c_n)}\vdash_{\bf MILL1} ||B||^{(c_0,c_n)}.
\end{array}
\end{array}
$
}
\caption{Systems of linear logic}
\end{figure}
%{\it first order multiplicative language}
Language
and sequent calculus
of {\it first order multiplicative linear logic} ({\bf MLL1})
are shown in Figures~\ref{MLL1 formulas},~ \ref{MLL1}.
We  are given a set   $\mathit{Pred}_+$ of {\it positive predicate symbols} with assigned  arities and a  set $\mathit{Var}$  of  variables. The sets $\mathit{Pred}_-$ and $\mathit{Pred}$ of {\it negative} and of all predicate symbols respectively are defined.
The sets of first order {\it atomic}, respectively, linear multiplicative formulas are denoted as $\mathit{At}$, respectively,  $\mathit{Fm}$.
  Connectives $\otimes$ and
 $\parr$  are called respectively {\it tensor} (also {\it times}) and  {\it cotensor} (also {\it par}).
{\it Linear negation} $\overline{(.)}$ is not a connective, but is {\it definable}
by induction on formula construction, as in Figure~\ref{MLL1 formulas}.
Note  that, somewhat non-traditionally, we follow the convention that negation flips tensor/cotensor factors,  typical for {\it noncommutative} systems. This does not change the logic (the formulas $A\otimes B$ and $B\otimes A$ are provably equivalent), but is more consistent with our intended geometric interpretation.

 A {\it context} ${\Gamma}$ is a finite multiset of formulas, and a {\it sequent} is an expression of the form $\vdash\Gamma$, where $\Gamma$ is a context. The set of free variables of $\Gamma$ is denoted as $FV(\Gamma)$. In this work we will consider several systems of sequent calculus, and, in order to avoid confusion, we will sometimes use abbreviations like $\vdash_{\bf MLL1}\Gamma$ to say that the sequent $\vdash\Gamma$ is derivable in ${\bf MLL1}$, similarly for other systems. We also emphasize that a commutative system like {\bf MLL1}  can be given an alternative {\it ordered} formulation, where contexts are defined as finite sequences (rather than multisets) of formulas and rules are supplemented with the {\it explicit Exchange} rule          $\cfrac{\vdash \Gamma,A,B,\Theta}
            {\vdash \Gamma,B,A,\Theta}~
            ({\rm{Ex}})$.

{\it First order multiplicative intuitionistic logic} ({\bf MILL1}) is the fragment of {\bf MLL1} restricted to sequents of the form $\vdash B,\overline{A_1},\ldots,\overline{A_n}$, written in {\bf MILL1} in a two-sided form as
${A_1},\ldots,{A_n}\vdash B$, where ${A_1},\ldots,{A_n}, B$ are {\it multiplicative intuitionistic formulas}, which means constructed using only positive predicate symbols, quantifiers, tensor and {\it linear implication} defined by
$A\multimap B=\overline{A}\parr B$. (Note that we can define linear implication differently as $A\multimap B=B\parr\overline{A}$ and get an equivalent system.)

The language  of {\it Lambek calculus} ({\bf LC}), i.e. noncommutative multiplicative intuitionistic linear logic,  is summarised in Figure~\ref{LC_rules}.  Noncommutative tensor is denoted as $\bullet$, and the two {\it directed} implications caused by noncommutativity are denoted as slashes. Formulas of {\bf LC} are often called {\it types}. A context in {\bf LC} is,  by definition, a {\it sequence}  of types, and a sequent is an expression of the form $\Gamma\vdash F$, where $\Gamma$ is a context and $F$ is a type.  In  {\bf LC} there is no analogue of the Exchange rule or  ``unordered'' formulation, i.e. the system is genuinely noncommutative.
We will not reproduce the rules which can be found, for example, in \cite{Lambek}.

Given a finite alphabet $T$ of {\it terminal symbols}, a {\it Lambek grammar} $G$ (over $T$) is a pair $G=(\mathit{Lex},S)$, where $\mathit{Lex}$, the   {\it lexicon}
is a finite set
of {\it lexical entries}, which are expressions $t:A$, where $t\in T$ and $A$ is a type, and $S$ is a selected atomic type. The {\it language} $L(G)$ of $G$ is the set of words
\be\label{Lambek_language}
 L(G)=\bigg\{t_1\ldots t_n\bigg|
 \begin{array}{l}
 ~t_1:A_1,\ldots, t_n:A_n\in Lex,
 \\
 A_1,\ldots,A_n\vdash_{\bf LC} S
 \end{array}\bigg.\bigg\},
 \ee
where $t_1\ldots t_n$ stands for the concatenation of strings $t_1,\ldots,t_n$. More details on categorial grammars can be found in \cite{MootRetore}.

 Given two variables  $l, r$, the {\it first order translation} $||F||^{(l,r)}$
 of an {\bf LC} formula  $F$ parameterized by $l, r$
  is shown in Figure~\ref{LC2MILL} ({\bf LC} propositional symbols are treated as binary predicate symbols). This provides a conservative embedding of {\bf LC} into {\bf MILL1}  \cite{MootPiazza}.

\section{MILL1 grammars and linguistic fragment}\label{first order grammar section}
\subsection{MILL1 grammars}
Translation from {\bf LC} suggests defining {\bf MILL1} grammars similar to Lambek grammars.

Let  a finite alphabet $T$ of terminal symbols be  given.
Assume for convenience that  our first order language contains  two special variables (or constants) $l,r$.
\begin{defi}\label{first order grammar}
Given a {\it terminal alphabet} $T$,
a {\it  {\bf MILL1} lexical entry} (over $T$) is a pair $(w,A)$, where $w\in T^*$ is nonempty, and $A$ is a {\bf MILL1} formula with one occurrence of $l$ and one occurrence of $r$ and no other free variables.
For the formula $A$ occurring in a lexical entry we will write
$A[x;y]=A[\nicefrac{x}{l}][\nicefrac{y}{r}]$
 (so that $A=A[l;r]$).

 A  {\bf MILL1} {\it grammar} $G$ (over $T$) is a pair $(\mathit{Lex},s)$, where $\mathit{Lex}$ is a finite set of {\bf MILL1} lexical entries, and $s$ is a binary predicate symbol.
 The {\it language $L(G)$ generated by }$G$ or, simply, the {\it language of} $G$ is defined by
 \be\label{MILL1_language}
 L(G)=\bigg\{w_1\cdots w_n\bigg|
 \begin{array}{l}
 ~(w_1,A_1),\ldots, (w_m,A_n)\in Lex,
 \\
 ~A_1[c_0;c_1],\ldots,A_n[c_{n-1};c_n]\vdash_{\bf MILL1} s(c_0,c_n)
 \end{array}\bigg.\bigg\}.
 \ee
 \end{defi}
It seems clear that, under such a definition, Lambek grammars translate precisely to {\bf MILL1} grammars generating the same language (compare (\ref{MILL1_language}) with (\ref{Lambek_language})).

It has been shown \cite{Moot_extended}, \cite{Moot_inadequacy} that {\bf MILL1} allows representing more complex systems such as displacement calculus, abstract categorial grammars, hybrid type-logical grammars. This suggests that the above definition should be generalized to allow more complex lexical entries, corresponding to word tuples rather than just words. (To the author's knowledge, no explicit definition of a {\bf MILL1} grammar has been given in the literature, but this concept is implicit in the above cited works).

Now  let us isolate the fragment of first order logic that is actually used in  translations of linguistic systems.

\subsection{Linguistic fragment}
\begin{defi}
A   language of {\bf MLL1} is {\it linguistically marked} if each  $n$-ary predicate symbol $p$ is equipped with  {\it valency} $v(p)\in{\mathbb{N}}^2$, $v(p)=(v_l(p),v_r(p))$,  where $v_l(p)+v_r(p)=n$, such that  $v(\overline{p})=(v_r(p),v_l(p))$. When $v(p)=(k,m)$ we say that first $k$ occurrences $x_1,\ldots,x_k$ in the atomic formula $p(x_1,\ldots,x_n)$  are {\it left} occurrences  or have {\it left polarity}, and $x_{k+1},\ldots,x_n$ are  {\it right} occurrences, with  {\it right polarity}.
\end{defi}
For a variable occurrence $x$ in a compound formula in a linguistically marked language $F$ we define the polarity of $x$ in $F$ by induction.
If $F=A\Box B$, where $\Box\in\{\otimes,\parr\}$,  and $x$ occurs in $A$, respectively $B$, then the polarity of $x$ in $F$ is the same as the polarity of $x$ in $A$, respectively $B$. If $F=QxA$, where $Q\in\{\forall,\exists\}$, then the polarity of a variable occurrence $x$ in $F$ is the same as the polarity of $x$ in $A$,
If $\Gamma=A_1,\ldots,A_n$ is a context, and $x$ is a variable occurrence in $A_i$ for some $i\in\{1,\ldots,n\}$, then the polarity of $x$ in the context $\Gamma$ and in the sequent $\vdash\Gamma$ is the same as the polarity of $x$ in $A_i$. For a {\bf MILL1} sequent $\Gamma\vdash F$, the polarity of its variable occurrences is determined by the translation to {\bf MLL1}.

\begin{defi}
Given a linguistically marked {\bf MLL1} language,
a  {\it formula, context} or {\it sequent}  is {\it linguistically marked} if every quantifier binds exactly one left and one right variable occurrence.
An {\bf MLL1} {\it derivation} $\pi$ is {\it linguistically marked} if all sequents occurring in $\pi$ are linguistically marked.

 A linguistically marked {\it formula} or {\it context} is {\it linguistically well-formed} if, furthermore, it has at most one left and at most one right occurrence of any free variable. A linguistically marked  {\it sequent} is {\it linguistically well-formed} if each of its free variables has {\it exactly} one left and one right occurrence.
A {\bf MILL1} {\it grammar} is {\it well-formed} if the formula in every lexical entry is linguistically well-formed with  $l$ and $r$ occurring with left and right polarity respectively.
\end{defi}
\begin{prop}\label{lingustically marked derivation}
Any cut-free derivation of a linguistically marked {\bf MLL1} sequent is linguistically marked. $\Box$
\end{prop}
Evidently, if we understand {\bf MILL1} as a fragment of {\bf MLL1}, then the translation of  {\bf LC}  in Figure~\ref{LC2MILL}  has precisely the linguistically well-formed fragment as its target (with $v(p)=(1,1)$ for all $p\in \mathit{Pred}$), and Lambek grammars translate to linguistically well-formed ones. Similar observations apply to translations in \cite{Moot_extended}, \cite{Moot_inadequacy}.

However, the standard sequent calculus formulation of {\bf MLL1} is  not {\it intrinsic} to the linguistic fragment. For an illustration, a basic {\bf LC} derivable sequent $A,B\vdash A\bullet B$ translates to {\bf MILL1} as $A(c_0,c_1),B(c_1,c_2)\vdash\exists x A(c_0,x)\otimes B(x,c_2)$. The latter, obviously, is derivable in ${\bf MLL1}$ by the $(\exists)$ rule applied to the sequent $A(c_0,c_1),B(c_1,c_2)\vdash A(c_0,c_1)\otimes B(c_1,c_2)$, which itself is not in the linguistic fragment. Thus, in order to derive a linguistically well-formed sequent in {\bf MLL1} we have to use ``linguistically ill-formed'' ones at intermediate steps.

Our goal is to provide the linguistic fragment with an {\it intrinsic} deductive system.

\subsubsection{Occurrence nets}
\begin{figure}[ht]
\centering
\begin{tikzpicture}
\begin{scope}[shift={(-3,0)}]
 \node at (0,0) {
$
\vdash a(e,t),\overline{a}(t,e)\otimes a(e,s),\overline{a}(s,e)$
};
\hspace*{11pt} %move lines because of font change
\scalebox{1.15}{
 \draw[thick,-](-1.,-0.2) to  [out=-30,in=180] (-.75,-.3) to  [out=0,in=-150](-0.5,-0.2);
 \draw[thick,-](1,-0.2) to  [out=-30,in=180] (1.25,-.3) to  [out=0,in=-150](1.5,-0.2);
 \draw[thick,-](-1.3,-0.2) to  [out=-30,in=180] (-.75,-.4) to  [out=0,in=-150](-.2,-0.2);
 \draw[thick,-](.7,-0.2) to  [out=-30,in=180] (1.25,-.4) to  [out=0,in=-150](1.8,-0.2);
}
 \end{scope}
 \node [above] at (0,0){$(\exists)$};
 \node at (0,-.1) {$\Longrightarrow$};
 %\node at (0,0) {$(\exists)$};
 \begin{scope}[shift={(3.5,0)}]

 \node at (0,0) {
$
\vdash a(e,t),\exists x(\overline{a}(t,x)\otimes a(x,s)),\overline{a}(s,e)$
};
\hspace*{-15pt}
\scalebox{1.15}{
 \draw[thick,-](-1.4,-0.2) to  [out=-30,in=180] (-.75,-.3) to  [out=0,in=-150](-0.3,-0.2);
 \draw[thick,-](1.0,-0.2) to  [out=-30,in=180] (1.45,-.3) to  [out=0,in=-150](1.7,-0.2);
 \draw[thick,-](-1.6,-0.2) to  [out=-30,in=180] (.25,-.5) to  [out=0,in=-150](2.0,-0.2);
}

 \end{scope}
\end{tikzpicture}
%\caption{fuck}
%}
\caption{Occurrence net example}
\label{occ_net}
\end{figure}
Let us introduce some convenient format and terminology for derivations of linguistically well-formed sequents.
\begin{defi}
An {\it occurrence net} of a linguistically marked {\bf MLL1} sequent $\vdash\Gamma$ is a
perfect matching $\sigma$ between left and right free occurrences in $\Gamma$, such that each pair ({\it link}) in $\sigma$ consists of occurrences of the same variable.
\end{defi}
Basically, occurrence nets are rudiments of proof-nets. To each cut-free linguistically marked derivation $\pi$ with conclusion $\vdash\Gamma$   we will assign an occurrence net $\sigma(\pi)$, which is an occurrence net of $\vdash\Gamma$. (Note that for a linguistically well-formed sequent there is only one occurrence net possible.)
\begin{defi} The {\it occurrence net $\sigma(\pi)$  of a cut-free linguistically marked derivation $\pi$} is constructed by induction on $\pi$.
\begin{itemize}
\item If $\pi$ is the axiom $\vdash \overline{X}, X$, where
$X= {p}(e_1,\ldots,e_n)$,
  the net is defined by matching each occurrence  $e_i$ in $X$  with the occurrence $e_i$ in $\overline{X}=\overline{p}(e_n,\ldots,e_1)$.
  \item If $\pi$ is obtained from a derivation $\pi'$ by the $(\parr)$ rule, then $\sigma(\pi)=\sigma(\pi')$.
    \item If $\pi$ is obtained from derivations $\pi_1$, $\pi_2$ by the $(\otimes)$ rule, $\sigma(\pi)=\sigma(\pi_1)\cup\sigma(\pi_2)$.
\item If $\pi$ is obtained from some $\pi'$ by the $(\forall)$ rule applied to a  formula $A'=A[
\nicefrac{v}{x}]$ and introducing the formula $\forall x A$, then
$\sigma(\pi)=\sigma(\pi')\setminus\{(v_l,v_r)\}$, where $v_l,v_r$ are the two occurrences of $v$ in $A'$.
 (The variable $v$ has no free occurrences in the premise other than those in $A'$, and, since all sequents are linguistically marked, there must be precisely one left and one right occurrence of $v$ in $A'$. It follows that $(v_l,v_r)\in\sigma(\pi')$.)
\item If $\pi$ is obtained from $\pi'$ by the $(\exists)$ rule applied to a  formula $A'=A[\nicefrac{v}{x}]$ and introducing the formula $\exists x A$, there are two cases depending on $\sigma(\pi')$. Let $v_l$ and $v_r$ be, respectively, the  left and the right occurrence of $v$ in $A'$ corresponding to the two occurrences of $x$ in $A$ bound by the existential quantifier.
     \begin{itemize}
     \item
If $(v_l, v_r)\in\sigma(\pi')$, then $\sigma(\pi)=\sigma(\pi')\setminus\{(v_l,v_r)\}$.
\item Otherwise let $v_l'$, $v_r'$ be such that $(v_l', v_r),(v_l, v_r')\in\sigma(\pi')$.
Then
\[\sigma(\pi)=(\sigma(\pi')\setminus\{(v_l',v_r),(v_l,v_r')\})\cup\{(v_l',v_r')\}.
\]
\end{itemize}
\end{itemize}
\end{defi}
 (If we see occurrence nets geometrically as bipartite graphs, then the $(\parr)$ rule does not change the graphs, the $(\otimes)$ rule takes the disjoint union of two graphs. The $(\forall)$ rule erases the link corresponding to the two occurrences that become bound by the universal quantifier.
 The $(\exists)$ rule has two cases: it  erases a link in the first case and
glues two links  into one in the second case.
An  example for the second case of the $(\exists)$ rule is shown in Figure~\ref{occ_net}.)

In what follows we will use a number of times the operation of replacing a free variable occurrence in an expression with another variable. So we introduce some notation generalizing familiar notation for substitution.
Let $\Phi$ be a context or a formula, let $x$ be a free variable occurrence in $\Phi$ and $v\in\mathit{Var}$.
 Then $\Phi[\nicefrac{v}{x}]$ is the expression obtained from $\Phi$ by replacing $x$ with $v$.
 We will also use the notation
  $\Phi[\nicefrac{v_1}{x_1},\ldots,\nicefrac{v_n}{x_n}]=\Phi[\nicefrac{v_1}{x_1}]\ldots[\nicefrac{v_n}{x_n}]$ for iterated substitutions, where it is assumed implicitly that $x_1,\ldots,x_n$ are pairwise distinct occurrences, so that the substitutions commute with each other.

Finally, we  will use for an induction parameter the {\it size of a cut-free derivation} defined in the following natural way. If the derivation $\pi$ is an axiom then the size $\mathit{size}(\pi)$ of $\pi$ is $1$. If $\pi$ is obtained from a derivation $\pi'$ by a single-premise rule then $\mathit{size}(\pi)=\mathit{size}(\pi')+1$. If $\pi$ is obtained from derivations $\pi_1$, $\pi_2$ by a two-premise rule then
$\mathit{size}(\pi)=\mathit{size}(\pi_1)+\mathit{size}(\pi_2)+1$.
\begin{prop}\label{MLL1 alpha-equiv derivations}
If $\Gamma$, $\Gamma'$ are linguistically marked contexts differing from each other only by renaming bound variables and the sequent $\vdash\Gamma$ is {\bf MLL1} derivable with a cut-free derivation $\pi$, then $\vdash\Gamma'$ is derivable with a cut-free derivation $\pi'$ of the same size and with the same occurrence net, $\mathit{size}(\pi')=\mathit{size}(\pi)$, $\sigma(\pi')=\sigma(\pi)$.
\end{prop}
\begin{proof} Induction on $\pi$. \end{proof}
\begin{prop}\label{exists'}
Let $\pi$ be a linguistically marked cut-free {\bf MLL1} derivation with conclusion $\vdash\Gamma$, and assume that $(e_l,e_r)\in\sigma(\pi)$.
Let $v\in\mathit{Var}$ be such that
$e_l,e_r$ are not in the scope of a quantifier $Qv$, $Q\in\{\forall,\exists\}$,
and let $\widehat\Gamma=\Gamma[\nicefrac{v}{e_l},\nicefrac{v}{e_r}]$.
 Then $\vdash\widehat\Gamma$ is derivable in {\bf MLL1} with a linguistically marked cut-free derivation $\widehat\pi$ of the same size as $\pi$. Moreover, if $v_l,v_r$  are occurrences of $v$ in $\widehat\Gamma$ replacing, respectively,  the occurrences $e_l,e_r$ in $\Gamma$, then $\sigma(\widehat\pi)=(\sigma(\pi)\setminus\{(e_l,e_r)\})\cup \{(v_l,v_r)\}$.
\end{prop}
\begin{proof} Induction on $\mathit{size}(\pi)$.
The most involved step is when  $\pi$ is obtained from a derivation $\pi'$ of some sequent $\vdash\Gamma'$ by the $(\exists)$ rule and  $(e_l,e_r)\not\in\sigma(\pi')$.

This means that   $\Gamma'$
contains some formula $A$, while
$\Gamma$
contains the formula
$\exists xA'$,
where
$A=A'[\nicefrac{e}/{x}]$, and the occurrences $e_l'$, $e_r'$ of $e$ in $A$ that correspond to the two  occurrences of $x$ in $A'$ are linked in $\sigma(\pi')$ to, respectively,  $e_r$, $e_l$, i.e. $(e_l,e_r'),(e_l',e_r)\in\sigma(\pi')$.  Note that
\be\label{unquantified formula}
A'=A[\nicefrac{x}{e_l'},\nicefrac{x}{e_r'}].
\ee

Observe from (\ref{unquantified formula}) that the context $\widehat\Gamma=\Gamma[\nicefrac{v}{e_l},\nicefrac{v}{e_r}]$ can be obtained as follows.

Let $\Theta=\Gamma'[\nicefrac{v}{e_l'},\nicefrac{v}{e_r},\nicefrac{v}{e_l},\nicefrac{v}{e_r'}]$,
and let $v_l',v_r,v_l,v_r'$ be the occurrences of $v$ in $\Theta$ replacing, respectively, the occurrences $e_l',e_r,e_l,e_r'$ in $\Gamma'$.
Let $B$ be the formula in $\Theta$ that corresponds to $A$ in $\Gamma'$. Put
\be\label{quantified formula}
\widehat A=\exists xB[\nicefrac{x}{v_l'},\nicefrac{x}{v_r'}].
\ee
Then $\widehat\Gamma$ is obtained from $\Theta$ by replacing $B$ with $\widehat A$.

Moreover, the sequent $\vdash\widehat\Gamma$ is derivable from $\vdash\Theta$ by the $(\exists)$ rule. It follows that, if the occurrences $e_l',e_r,e_l,e_r'$ in $\Gamma'$ are not in the scope of a quantifier $Qv$, $Q\in\{\forall,\exists\}$, then,  by the induction hypothesis (applied twice), we have that $\vdash_{\bf MLL1}\Theta$, hence  $\vdash_{\bf MLL1}\widehat\Gamma$. However, the above condition  may fail, and then we cannot apply the induction hypothesis  directly. Therefore, in a general case, we need some more work.

By renaming bound variables if necessary, we can obtain from $\Gamma'$ a linguistically marked context $\Gamma''$ such that the occurrences $e_l',e_r,e_l,e_r'$ are not in the scope of a quantifier $Qv$, $Q\in\{\forall,\exists\}$. By the preceding proposition the sequent $\vdash\Gamma''$ is derivable with a cut-free derivation $\pi''$, where $\mathit{size}(\pi'')=\mathit{size}(\pi')$, $\sigma(\pi'')=\sigma(\pi')$.
Then we put $\Theta''=\Gamma''[\nicefrac{v}{e_l'},\nicefrac{v}{e_r},\nicefrac{v}{e_l},\nicefrac{v}{e_r'}]$.  Now the induction hypothesis can be applied for sure and the sequent $\vdash\Theta''$ is derivable with a linguistically marked cut-free derivation $\rho$, and $\mathit{size}(\rho)=\mathit{size}(\pi'')=\mathit{size}(\pi')$. Moreover,
\be\label{intermediate occurrence net}
\sigma(\rho)=(\sigma(\pi')\setminus\{(e_l',e_r),(e_l,e_r')\})\cup\{(v_l',v_r),(v_l,v_r')\},
\ee
where $v_l',v_r,v_l,v_r'$ are the occurrences of $v$ in $\Theta''$ replacing, respectively, the occurrences $e_l'$, $e_r$, $e_l$, $e_r'$ in $\Gamma''$.

Let $A''$ be the formula in $\Gamma''$ corresponding to $A$ in $\Gamma'$,
and
 $B''$ be the formula in $\Theta''$ obtained from $A''$ in $\Gamma''$. Let $\widehat B=\exists xB''[x/v_l',x/v_r']$.
     The formula $\widehat A$ in  (\ref{quantified formula})   may differ from $\widehat B$ only by renaming bound variables.
Let $\widehat\Theta$ be the context obtained from $\Theta''$ by replacing $B''$ with $\widehat B$.
Then $\widehat\Theta$ and $\widehat\Gamma$, as well, may differ only by renaming bound variables. But $\vdash\widehat\Theta$ is derivable from $\vdash\Theta''$ by the $(\exists)$ rule, i.e. the derivation $\widehat\rho$ of $\vdash\widehat\Theta$ is obtained from $\rho$ by attaching the $(\exists)$ rule. Since $\mathit{size}(\rho)=\mathit{size}(\pi')$, it follows that
$\mathit{size}(\widehat\rho)=\mathit{size}(\pi)$, and
the occurrence net $\sigma(\widehat\rho)$ is computed from (\ref{intermediate occurrence net}). Applying Proposition~\ref{MLL1 alpha-equiv derivations} once more, we obtain the desired result. \end{proof}

\subsubsection{Linguistic  derivations}
\begin{defi}
Let $\vdash\Gamma$
be a  linguistically well-formed {\bf MLL1} sequent containing a formula $A$ and let $s,t\in Var$, $s\not=t$, have, respectively, a left and a right free occurrences in $A$, denoted as $s_l$ and $t_r$. Let $s_r$  be the {\it unique right} occurrence of $s$ in $\Gamma$, and $t_l$ be the {\it unique left} occurrence of $t$.
Let $x, v\in \mathit{Var}$ be such that $x$ does not occur in $A$ freely, and $s_r,t_l$ do not occur in $\Gamma$ in the scope of a quantifier $Qv$, $v\in\{\forall,\exists\}$.

Let  $A'=\exists xA[\nicefrac{x}{t_r},\nicefrac{x}{s_l}]$, and the context $\Gamma'$ be obtained from $\Gamma$ by replacing $A$ with $A'$. Finally, let
 $\widehat\Gamma=\Gamma'[\nicefrac{v}{t_l},\nicefrac{v}{s_r}]$.
   Then {\it the sequent $\vdash\widehat\Gamma$ is obtained from $\vdash\Gamma$ {\it by the $(\exists')$ rule}}.
   \end{defi}
   An example of the $(\exists')$ rule is the following
\be\label{exists' example}
\cfrac
{
\vdash \overline{a}(t,z), a(y,s)\parr(\overline{a}(s,y)\otimes a(z,t))
}
{
\vdash \overline{a}(v,z),\exists x(a(y,v)\parr(\overline{a}(x,y)\otimes a(z,x)))
}
~(\exists'),
\ee
where the valency $v(a)=(1,1)$.
In the notation of the above definition, the sequents in (\ref{exists' example}) have the following structure:
\[
\begin{array}{rl}
 A=a(y,s)\parr(\overline{a}(s,y)\otimes a(z,t)),
& \Gamma=\overline{a}(t,z), a(y,s)\parr(\overline{a}(s,y)\otimes a(z,t)),\\
A'=\exists x(a(y,s)\parr(\overline{a}(x,y)\otimes a(z,x))),
& \Gamma'=\overline{a}(t,z), \exists x(a(y,s)\parr(\overline{a}(x,y)\otimes a(z,x))),\\
\widehat A=\exists x(a(y,v)\parr(\overline{a}(x,y)\otimes a(z,x))),
& \widehat\Gamma=\overline{a}(v,z), \exists x(a(y,v)\parr(\overline{a}(x,y)\otimes a(z,x))),
\end{array}
\]
where we  also indicated the formula $\widehat A$ in the context $\widehat\Gamma$ that corresponds to the formula $A$ in the context $\Gamma$.
   \begin{prop}
   If $\vdash\Gamma$ is an {\bf MLL1} derivable linguistically well-formed sequent and the sequent $\vdash\widehat\Gamma$ is obtained from $\vdash\Gamma$ by the $(\exists')$ rule, then $\vdash\widehat\Gamma$ is linguistically well-formed and {\bf MLL1} derivable.
   \end{prop}
   \begin{proof}
   Let
    the notation be as in the definition above.  By Proposition~\ref{lingustically marked derivation} any  cut-free derivation $\pi$ of $\vdash\Gamma$ is linguistically marked so it has an occurrence net $\sigma(\pi)$. Moreover,
   since $\vdash\Gamma$ is linguistically well-formed, there  are only two occurrences of  $s$ and $t$ in $\Gamma$, which implies
    $(s_l,s_r),(t_l,t_r)\in\sigma(\pi)$.

    By renaming, if necessary, bound variables of $A$  we can obtain from $\Gamma$ a linguistically well-formed context $\Theta$ such that $s_l,t_r$ in $\Theta$ are not in the scope of a quantifier $Qv$, $Q\in\{\forall,\exists\}$. By Proposition~\ref{MLL1 alpha-equiv derivations}, the sequent $\vdash\Theta$ is cut-free derivable with a linguistically marked derivation $\rho$ such that $\sigma(\rho)=\sigma(\pi)$, i.e. $(s_l,s_r),(t_l,t_r)\in\sigma(\rho)$.
    It follows from
   Proposition~\ref{exists'} that
    $\vdash_{\bf MLL1}\Theta'$, where $\Theta'=\Theta[\nicefrac{v}{t_l},\nicefrac{v}{t_r},\nicefrac{v}{s_l},\nicefrac{v}{s_r}]$.

    Let $B$ be the formula in $\Theta$ corresponding to $A$ in $\Gamma$ and $B'$ be the formula in $\Theta'$ corresponding to $B$ in $\Theta$. Finally, let $v_l,v_r',v_l',v_r$ be the occurrences of $v$ in $\Theta'$ replacing the occurrences $t_l, t_r,s_l,s_r$ in $\Theta$ respectively, and let $\widehat{ B}=\exists xB'[\nicefrac{x}{v_l'},\nicefrac{x}{v_r'}]$. It follows from the definition that the formulas $\widehat B$ and $\widehat A$, where
    $\widehat A$ is the formula in $\widehat\Gamma$ corresponding to $A'$ in $\Gamma'$, may differ only by renaming bound variables.
    Let $\widehat\Theta$ be obtained from $\Theta'$ by replacing $B'$ with $\widehat B$. Then the sequent $\vdash\widehat{\Theta}$ is derivable from $\vdash\Theta'$ by the $(\exists)$ rule.
    But $\widehat\Theta$, again, may differ from $\widehat\Gamma$
    only by renaming bound variables, and the statement follows from Proposition~\ref{MLL1 alpha-equiv derivations}.
    That $\vdash\widehat\Gamma$ is linguistically well-formed is obvious from counting free left and right occurrences. \end{proof}

   Note that on the level of  occurrence nets, seen as bipartite graphs, the $(\exists')$ rule does the same gluing as $(\exists)$; only vertex labels are changed.
\begin{defi}
Derivations of linguistically well-formed sequents using rules of {\bf MLL1} and the $(\exists')$ rule  and involving only linguistically well-formed sequents   are {\it linguistic derivations}.
\end{defi}
\begin{lem}\label{linguistically well-formed fragment}
Any cut-free {\bf MLL1} derivation $\pi$ of a  linguistically well-formed sequent $\vdash\Gamma$  translates to a linguistic derivation.
\end{lem}
\begin{proof}
By Proposition~\ref{lingustically marked derivation} the derivation $\pi$ is linguistically marked. We use
   induction on  $\mathit{size}(\pi)$.

   The main case is when the last step in  $\pi$ is  the $(\exists)$ rule applied to a cut-free linguistically marked derivation $\pi'$ of some non-linguistically well-formed $\vdash\Gamma'$.
   This means that $\Gamma'$ has {\it four} free occurrences of some free variable $v$ (two of which are renamed when the quantifier is introduced in $\Gamma$).
   Let $v_l',v_r'$ be, respectively, the left and the right occurrences of $v$ in $\Gamma'$ that become bound in $\Gamma$, and let $v_l,v_r$ be, respectively, the remaining left and right occurrence of $v$ in $\Gamma'$ (inherited by $\Gamma$). We have that $\Gamma'$ contains a formula $A$ in which $v_l', v_r'$ are located and  $\Gamma$ is obtained from $\Gamma'$ by replacing $A$ with the formula $A'=\exists xA[\nicefrac{x}{v_r'},\nicefrac{x}{v_l'}]$.

   The derivation $\pi'$ still has an occurrence net $\sigma(\pi')$. If $(v_l',v_r')\in\sigma(\pi')$, then, by Proposition~\ref{exists'}, the sequent $\vdash\Gamma''$, where $\Gamma''=\Gamma'[\nicefrac{e}{v_l'},\nicefrac{e}{v_r'}]$  and $e\in\mathit{Var}$ is fresh, is derivable with the derivation of the same size as $\pi'$. But $\vdash\Gamma$ is equally well derivable from $\vdash\Gamma''$ by the ($\exists$) rule, and the  statement follows from the induction hypothesis.  Otherwise we have $(v_l,v_r'),(v_l',v_r)\in\sigma(\pi')$. Let $u,w\in\mathit{Var}$ be fresh and let $\Gamma''=\Gamma'[\nicefrac{u}{v_l},\nicefrac{u}{v_r'},\nicefrac{w}{v_l'},\nicefrac{w}{v_r}]$. Then $\vdash\Gamma''$ is, again by Proposition~\ref{exists'}, derivable with a derivation of the same size as $\pi'$. But $\vdash\Gamma$ can be obtained from $\vdash\Gamma''$ by the ($\exists'$) rule, and the statement follows. \end{proof}

Thus, adding the $(\exists')$ rule, we obtain a kind of intrinsic deductive system for the linguistically well-formed fragment. It might seem though that the   usual syntax of first order sequent calculus is not very natural for such a system. It is not even clear how to write the $(\exists')$ rule in the sequent calculus format concisely. Arguably, some other representation might be desirable.

\section{Tensor type calculus}\label{tensor types section}
We assume that we are given   an infinite  set $\mathit{Ind}$ of  {\it indices}. They will be used in all kinds of syntactic objects (terms, types, typing judgements) that we consider.

Indices will have upper or lower occurrences,  which may
be  free or bound. (The exact meaning of free and bound occurrences for different syntactic objects will be defined below.) Following the practice of first order logic, we will write $e^{[\nicefrac{i}{j}]}$, respectively $e_{[\nicefrac{i}{j}]}$, for the expression  obtained from $e$ by replacing the upper, respectively lower,
free
occurrence(s)
 of
 the
  index $j$ with    $i$, implicitly stating by this notation that the substituted occurrences of $i$ are free in the resulted expression.
We will write ${I^\bullet}(e)$ and ${I_\bullet}(e)$ for the sets of   upper and lower indices of $e$  respectively, and $I(e)$ for the pair $(I^\bullet(e),I_\bullet(e))$; similarly, $FI(e)=(FI^\bullet(e),FI_\bullet(e))$ for sets of free indices.

For   pairs of index sets  $W_i=(U_i,V_i)$, $i=1,2$, we write binary operations componentwise, e.g.
 $W_1\cup W_2=(U_1\cup U_2,V_1\cup V_2)$. We write $W_1\perp W_2$ for $U_1\cap U_2=V_1\cap V_2=\emptyset$. We also write $(i,j)\perp (U,V)$ for $i\not\in U$, $j\not\in V$ and $(i,j)\in (U,V)$ for $i\in U$, $j\in V$. Finally, for $W=(U,V)$, we write $W^\dagger=(V,U)$ and $|W|=U\cup V$.

\subsection{Tensor terms}\label{lang}
\begin{defi}\label{tensor terms}
Given  an alphabet $T$ of {\it terminal symbols}, {\it tensor terms} (over $T$) are the elements of  the free {\it commutative} monoid (written multiplicatively with the unit denoted as 1) generated by the sets
\be\label{generators}
\{[w]_i^j|~i,j\in\mathit{Ind},w\in T^*\}\mbox{ and }\{[w]|~w\in T^*\}
\ee
satisfying the constraint that any index has at most one lower and one upper occurrence.

Generating set (\ref{generators}) elements of the form $[w]_i^j$ are {\it elementary regular tensor terms} and those of the form $[w]$ are elementary {\it singular tensor terms}.
\end{defi}
  (The adjective ``tensor'' will often be omitted in the following.)

  It follows from the definition, that an index in a term cannot occur more than twice.
  \begin{defi} An index  occurring in a term $t$ is {\it free} in $t$ if it occurs in $t$ exactly once.
  An index occurring in  $t$ twice (once as an upper one and once as a lower one) is {\it bound} in $t$. A term is {\it $\beta$-reduced} if it has no bound indices.
  Terms $t,t'$ that can be obtained from each other by renaming bound indices are {\it $\alpha$-equivalent}, $t\equiv_\alpha t'$.
  \end{defi}
\begin{defi}\label{beta-red} {\it $\beta$-Reduction} of tensor terms is generated by the relations
\be\label{tensor term relations}
\begin{array}{c}
[u]^j_i\cdot[v]^k_j\to_\beta[uv]^k_i,\quad [w]^i_i\to_\beta[w],\quad[\epsilon]\to_\beta1,
\\
 {[} a_1\ldots a_n ]
 \to_\beta [a_na_1\ldots a_{n-1}]\mbox{ for }a_1,\ldots,a_n\in T,
\end{array}
\ee
where $\epsilon$ denotes the empty word. Terms related by $\beta$-reduction are {\it $\beta$-equivalent}, notation: $t\equiv_\beta s$.
\end{defi}
The meaning of $\beta$-reduction will be discussed shortly in Section~\ref{terms geometry}.

  It is easy to see that  free indices are  invariant under $\beta$-equivalence and that any term is $\beta$-equivalent to a $\beta$-reduced one. Note also that $\alpha$-equivalent terms are automatically $\beta$-equivalent, because their $\beta$-reduced forms coincide.
  \begin{defi}\label{misc. term definitions}
  A $\beta$-reduced term is {\it regular} if it is the product of elementary regular terms. A general tensor term is regular if it is $\beta$-equivalent to a $\beta$-reduced regular term. A tensor term that is not regular is {\it singular}.

  The elementary term  $[\epsilon]^j_i$  is    denoted  as $\delta^j_i$, more generally, we write $\delta^{i_1\ldots i_k}_{j_1\ldots j_k}$ for the term $\delta^{i_1}_{j_1}\cdots\delta^{i_k}_{j_k}$. When all indices $i_1,\ldots,i_k,j_1,\ldots,j_k$ are pairwise distinct, the term
  $\delta^{i_1\ldots i_k}_{j_1\ldots j_k}$ is called {\it Kronecker delta}.

  A term $t$ is {\it lexical} if it is regular and not $\beta$-equivalent to some term of the form $\delta^i_j\cdot t'$ with $(i,j)\in FI(t)$.
   \end{defi}
   Multiplication by  Kronecker deltas amounts to renaming indices:
if $t$ is a  term  with
$(i,j)\in FI(t)$  and
$(i',j')\perp I(t)$,
then $\delta_{i}^{i'}\cdot t\equiv_\beta t^{[\nicefrac{i'}{i}]}$ and
$\delta_{j'}^{j}\cdot t\equiv_\beta t_{[\nicefrac{j'}{j}]}$. For the case of $i=j$, we have from (\ref{tensor term relations}) that $\delta^i_i\equiv_\beta 1$.

As for lexical terms they can be equivalently characterized as $\beta$-equivalent to regular $\beta$-reduced ones of the form $[w_1]^{i_1}_{j_1}\cdots[w_k]^{i_k}_{j_k}$ where all $w_1,\ldots,w_k$ are nonempty.

\subsubsection{Geometric representation}\label{terms geometry}
\begin{figure}%[ht]
\centering
\subfloat[ $ {[ab]_i^j[cd]_j^k\to_\beta[abcd]_i^k}$ ]
{
%\quad\quad\quad\quad
\quad
\begin{tikzpicture}
\begin{scope}[shift={(0,-.3)}]
\draw[thick,->](-.3,.3)--(.3,.3);
\node[below] at(-.3,.3){$i$};
\node[below] at(.3,.3){$j$};
\node [above]at (0,.3) {$\mbox{ab}$};
\end{scope}

\draw[thick,dashed,-](.3,0)
to [out=-90,in=180] (.55,.25) to  [out=0,in=-90](.8,0);

\begin{scope}[shift={(1.1,-.3)}]
\draw[thick,->](-.3,.3)--(.3,.3);
\node[below] at(-.3,.3){$j$};
\node[below] at(.3,.3){$k$};
\node [above]at (0,.3) {$\mbox{cd}$};
\end{scope}

\begin{scope}[shift={(1.2,0)}]
\node at (.7,0){\large $\to_\beta$};
\begin{scope}[shift={(.5,0)}]
\draw[thick,->](.7,0)--(1.3,.0);
\node [above]at (1.,0) {$\mbox{abcd}$};
\node[below] at(.7,0){$i$};
\node[below] at(1.3,0){$k$};
\end{scope}
\end{scope}
\end{tikzpicture}
\quad
}
\subfloat[${[abc]^i_i\to_\beta[abc]}$]
%$\to_\beta[bca]}$]
{
\quad
\begin{tikzpicture}
\begin{scope}[shift={(0,-.2)}]
\draw[thick,->](-.3,0)--(.3,0);
\draw [thick,dashed,-](.3,.0)to [out=90,in=0] (0,.5) to  [out=180,in=90](-.3,0);
\node[below] at(-.3,0){$i$};
\node[below] at(.3,0){$i$};
\node [font=\footnotesize, above]at (0,0) {$\mbox{abc}$};
\end{scope}
\node at (.7,-.2){\large $\to_\beta$};
\begin{scope}[shift={(1.6,0)}]
          \draw[thick,-](-.5,0) to  [out=-90,in=180] (0,-.5) to  [out=0,in=-90](.5,0)
           to  [out=90,in=0] (0,.5) to  [out=180,in=90](-.5,0);
           \node at (0,-.2) {$\mbox{abc}$};
\end{scope}
\end{tikzpicture}
\quad
}
\subfloat[$\delta^i_i\to 1$]
{
%\quad\quad\quad\quad
\begin{tikzpicture}
\begin{scope}[shift={(5,0)}]
\draw[thick,-](-.5,0) to  [out=-90,in=180] (0,-.5) to  [out=0,in=-90](.5,0)
           to  [out=90,in=0] (0,.5) to  [out=180,in=90](-.5,0);
           \node at (.9,-.1){$\to_\beta$};
           \node at (1.3,0){\large${\emptyset}$};
           %\node at (.8,-1){$\delta^i_i\to 1$};
\end{scope}
\end{tikzpicture}
}
\caption{Meaning of $\beta$-reductions}
\label{reductions}
\end{figure}
We think of regular terms  as bipartite graphs having indices as vertices  and edges labeled with words, the direction of edges being from lower indices to upper ones.

A regular term  $[w]_i^j$, $i\not=j$, corresponds to a single edge from $i$ to $j$ labeled with $w$,
the product of two terms without common indices is  the disjoint union of the corresponding graphs,
and a term  with repeated indices corresponds to the graph obtained by gluing edges along matching vertices.  The  unit $1$ corresponds to the empty graph. As for singular terms, such as $[w]^i_i$,  they correspond to closed loops (with no vertices) labeled with {\it cyclic} words (this explains the last relation in (\ref{tensor term relations})). These arise when edges are glued cyclically.
Singular terms are pathological, but we need them for consistency of definitions.
Note however our convention that when there are no labels, loops evaporate ($\delta^i_i\equiv_\beta 1$), so that the singularity does not arise.
The correspondence between terms and edge-labeled graphs is  illustrated in Figure~\ref{reductions}.  We emphasize that this geometric representation is an {\it invariant} of $\beta$-equivalence.

Finally, we note that a Kronecker delta is a bipartite graph with  no edge labels, and a lexical term, vice versa, is a graph whose every edge has a nonempty label.

 (We also remark that in the system introduced in \cite{Slavnov_tensor}, which preceded this work, empty loops were not factored out. This was adequate for problems considered there.)

\subsubsection{Remarks on binding and multiplication}
We
emphasize that binding in terms is {\it global}, not restricted to some ``scope''. If $t$ is a subterm (a factor) in a term $s$ and $t$ has a  bound index $i$, then $i$ is bound everywhere in $s$. In particular, if $t$ and $t'$ both have a bound index $i$ then the expression $tt'$ simply is not a term (it has four occurrences of $i$). In general, because term multiplication is only partially defined (concretely, the expression $ts$ is a term if $I(t)\perp I(s)$), we have to be careful when commuting multiplication with $\beta$- or $\alpha$-equivalence. It can be that $t\equiv_\beta t'$, but for some term $s$ the product $ts$ is well-defined, while $t's$ is not a term at all because of index collisions.

As a compensation,  multiplication of terms is strictly associative whenever defined: if $(ts)k$ is a term then $t(sk)$ is also a term and the two are strictly equal. Also, if $t\to_\beta t'$, $s\to_\beta s'$, $k\to_\beta k'$ and the expression $tsk$ is a term then $t's'k'$ is also a term and
$tsk\to_\beta t's'k'$.  Passing to equivalence classes we have the following: if $t\equiv_\beta t'$, $s\equiv_\beta s'$, $k\equiv_\beta k'$ and both expressions $tsk$, $t's'k'$ are well-defined terms, then $tsk\equiv_\beta t's'k'$; the same property holds for $\alpha$-equivalence.

(On the other hand, one could argue that some amount of non-associativity, in the case of linguistic applications,  might be an {\it advantage} rather than a drawback. Such an observation might suggest possible directions for refinements and modifications for the system introduced in  this paper.)

\subsection{Tensor formulas and types}
\begin{defi}
Given a set $\mathit{Lit}_+$ of   {\it positive literals}, where every  element $p\in \mathit{Lit}_+$ is assigned a {\it valency} $v(p)\in \mathbb{N}^2$,
the set $\widetilde{\mathit{Fm}}$ of {\it tensor pseudoformulas} is built  according to the grammar in Figure~\ref{ETTC_lang},
where  $\mathit{Lit}_-$ and  $\mathit{Lit}$ are, respectively, the set of {\it negative literals} and of all literals. The convention for negative literals is that $v(\overline{p})=(m,n)$ if  $v(p)=(n,m)$.
The set $\widetilde{\mathit{At}}$ is the set atomic pseudoformulas.
\end{defi}
\begin{figure}
  \centering
{
$\begin{array}{c}
\mathit{Lit}_-=\{\overline{p}|~p\in \mathit{Lit}_+\}, \quad\overline{\overline{p}}=p\mbox{ for }p\in \mathit{Lit}_+, \quad \mathit{Lit}=\mathit{Lit}_+\cup \mathit{Lit}_-.\\[.2cm]
\widetilde{\mathit{At}}=\{p_{j_1\ldots j_n}^{i_1\ldots i_m}|~p\in \mathit{Lit},v(p)=(m,n),i_1,\ldots, i_m,j_1,\ldots,j_n\in\mathit{Ind}\}.\\[.2cm]
\widetilde{\mathit{Fm}}::=\widetilde{\mathit{At}}|(\widetilde{\mathit{Fm}}\otimes\widetilde{\mathit{Fm}})|
(\widetilde{\mathit{Fm}}\parr\widetilde{\mathit{Fm}})|\nabla^i_j\widetilde{\mathit{Fm}}
|\triangle^i_j\widetilde{\mathit{Fm}},~
i,j\in\mathit{Ind},~i\not=j.\\[.2cm]
\overline{p^{i_1,\ldots,i_n}_{j_1,\ldots,j_m}}=\overline{p}_{{i_n},\ldots, {i_1}}^{{j_m},\ldots, {j_1}}.
\\[.2cm]
\overline{A\otimes B}=\overline{B}\wp\overline{A},\quad
\overline{A\wp B}=\overline{B}\otimes\overline{A},\quad
\overline{\nabla_i^jA}= \triangle^i_j\overline{A},\quad
\overline{\triangle_i^jA}= \nabla^i_j\overline{A}.
\end{array}
$
}  % Where the multiple \quad by accident? If you insist on it, let us know
\caption{{\bf ETTC} Language}
\label{ETTC_lang}
\end{figure}
 Duality $\overline{(.)}$ is not a connective or operator, but is definable.
The symbols $\nabla,\triangle$ are {\it binding operators}.
They bind indices exactly in the same way as quantifiers bind  variables.
The operator $Q\in\{\nabla,\triangle\}$ in front of an expression $Q^i_jA$ has $A$ as its scope and binds all lower occurrences of $i$, respectively, upper occurrences of $j$ in $A$ that are not already bound by some other operator.
\begin{defi}
A  pseudoformula $A$ is {\it well-formed} when  any index has at most one  free upper and one free lower occurrence in $A$, and every binding operator binds exactly one lower and one upper index occurrence.  A well-formed pseudoformula $A$ is a {\it tensor formula} or a {\it pseudotype}. If moreover  $FI^\bullet(A)\cap FI_\bullet(A)=\emptyset$ then $A$ is a {\it tensor type}.
\end{defi}
Note that definitions of free and bound indices for tensor terms and tensor formulas are {\it different}. In particular, unlike  terms, general tensor formulas (that are not tensor types) may have repeated free indices (i.e. an index may have both an upper and a  lower free occurrence). Also, unlike the case of terms, binding in tensor formulas is {\it local}, visible only in the scope of a binding operator.

   Tensor formulas that  can be obtained  from each other by renaming bound indices are {\it $\alpha$-equivalent}. The set of atomic types (i.e. coming from  atomic pseudoformulas) is denoted as $\mathit{At}$.

\begin{defi}
A {\it tensor  pseudotype context}  $\Gamma$, or, simply, a {\it tensor context} is a finite multiset of pseudotypes such that for any two distinct $A,B\in \Gamma$ we have that
$FI(A)\perp FI(B)$.
 The pseudotype context $\Gamma$ is a {\it type context} if $FI^\bullet(\Gamma)\cap FI_\bullet(\Gamma)=\emptyset$, where  $FI(\Gamma)=\bigcup\limits_{A\in\Gamma}FI(A)$.

 An {\it ordered tensor context} $\Gamma$ is defined identically with the difference that $\Gamma$ is a sequence rather than a multiset.
\end{defi}
Again, in a pseudotype context an index can have two free occurrences, once as a lower one and once as an upper one. For example the expression $a^i,\overline{a_i}$ is a legitimate pseudotype context, but not a type context.

\subsection{Sequents and typing judgements}
\begin{defi}\label{tensor sequent}
A   {\it pseudotyping judgement}  $\Sigma$ is an expression of the form $ t:: \Gamma$, where
 $\Gamma$ is a tensor   context
  and $t$ is a tensor term such that
  \be\label{judgement}
   I(t)\perp FI(\Gamma),\mbox{ } FI(t)\cup FI(\Gamma)=(FI(t)\cup FI(\Gamma))^\dagger.
    \ee
   When $\Gamma$ is a type (not just pseudotype) context, we say that $\Sigma$ is a {\it tensor typing judgement}.

   A {\it tensor sequent} is an expression of the form $\vdash\Sigma$, where $\Sigma$ is a pseudotyping judgement.

   An {\it ordered pseudotyping judgement} and {\it ordered tensor sequent} are defined identically with the difference that the tensor context should be ordered.
  \end{defi}
  (In the paper \cite{Slavnov_tensor} preceding this work we used a slightly different notation for tensor sequents with the term to the left of the turnstile.)

  Spelling out defining relation (\ref{judgement}): if we erase from $\Sigma$ all bound indices of $\Gamma$, then every remaining index has exactly one upper and one lower occurrence. When $\Sigma$ is a  {tensor typing judgement}, we have $FI(t)=FI(\Gamma)^\dagger$, i.e. every  index occurring in $\Sigma$ freely has exactly one free occurrence in $\Gamma$ and one in $t$. In this case we read $\Sigma$  as ``$t$ has type $\Gamma$''. A pseudotyping judgement is not a genuine typing judgement if there are repeated free indices  in the tensor context (i.e. some index in the tensor context  has both an upper and a lower free occurrence).
   When $t=1$ we write the  judgement $\Sigma$  simply as $\Gamma$. When $\Gamma$ consists of a single formula $F$, we write $\Sigma$ as $t:F$.
      When  $t$ is $\beta$-reduced or lexical we say that $\Sigma$  is, respectively, $\beta$-reduced or lexical.

      The ordered and unordered contexts and sequents correspond to two possible formulations of sequent calculus. We will generally use the unordered version, but for geometric representation of tensor sequents and sequent rules we need to consider explicit Exchange rule. The definitions below apply to both versions.

\begin{defi}\label{tensor sequent relations}
If $t,t'$ are terms with $t\equiv_\beta t'$ then the pseudotyping judgements $t::\Gamma$ and $t'::\Gamma$ are {\it $\beta$-equivalent}, $t::\Gamma\equiv_\beta t'::\Gamma$.

 {\it $\alpha$-Equivalence  of pseudo-typing judgements} is generated by the following:
\[
\begin{array}{c}
t\equiv_\alpha  t', F\equiv_\alpha F'\Rightarrow t::\Gamma,F\equiv_\alpha t'::\Gamma,F',\\
i\in FI^\bullet(t)\cup FI^\bullet(\Gamma), ~j\mbox{ fresh }\Rightarrow
t::\Gamma\equiv_\alpha (t::\Gamma)^{[\nicefrac{j}{i}]}_{[\nicefrac{j}{i}]}.
\end{array}
\]
{\it $\eta$-Expansion of pseudo-typing judgements} is the transitive closure of the relation defined by
\[
i\in FI^\bullet(\Gamma)\cap FI_\bullet(\Gamma),~j\mbox{ fresh }\Rightarrow
t::\Gamma\to_\eta\delta^i_{j} t::\Gamma^{[\nicefrac{j}{i}]},~t::\Gamma\to_\eta\delta_i^jt::\Gamma_{[\nicefrac{j}{i}]}.
\]
{\it $\eta$-Reduction of judgements} is the opposite relation of $\eta$-expansion,  {\it $\eta$-equivalence} ($\equiv_\eta$) is the symmetric closure of $\eta$-expansion.

Tensor sequents $\vdash\Sigma$, $\vdash\Sigma'$ are, respectively $\alpha$-, $\beta$- or $\eta$-equivalent iff $\Sigma$ and $\Sigma'$ are.
\end{defi}
      $\eta$-Expansion removes   from the tensor context
        free index repetitions
       by  renaming them and, simultaneously,   adds Kronecker deltas to the term as a compensation. Typing judgements have no repeated free indices in the tensor context and  are, thus,  {\it $\eta$-long} in the sense that they cannot be $\eta$-expanded. A general pseudotyping judgement could be thought as a shorthand notation for its $\eta$-expansion; the sequent $\vdash a^i,\overline a_i$ ``morally'' is a short for $\vdash \delta^j_i::a^i,\overline a_j$. Lexical pseudotyping judgements, on the contrary, are {\it $\eta$-reduced} in the sense that they are not $\eta$-expansions of anything, for having no Kronecker deltas in the term. Lexical  typing judgements, at once ${\eta}$-long and $\eta$-reduced, will have natural interpretation as nonlogical axioms.
   It is easy to see that any pseudotyping judgement has (many $\alpha\beta$-equivalent) $\eta$-long and $\eta$-reduced forms.

\subsubsection{Geometric representation}

 \begin{figure}[htb]
\centering
{
\quad\quad\quad\quad
%\scalebox{1}[.8]
{
         \begin{tikzpicture}[xscale=2,yscale=2]
    \begin{scope}[shift={(.5,0)}]
        %\draw  (0,0) circle [radius=0.05];
        \draw[thick,->](1.25,0) to  [out=-90,in=0] (.625,-.5) to  [out=180,in=-90](0,0);

        \node at(.625,-.4){$\gamma\delta$};
        %\node at(.75,.35){a};
        \node[font=\footnotesize] at(.44,-.15){$\beta\alpha$};
        \draw[thick,->](1,0) to  [out=-90,in=0] (.75,-.25) to  [out=180,in=-90](0.5,0);
        \draw[thick,<-](0.75,0) to  [out=-90,in=0] (.5,-.25) to  [out=180,in=-90](.25,0);

        \draw [fill] (.25,0) circle [radius=0.05];
        \draw [fill] (1,0) circle [radius=0.05];
        \draw [fill] (1.25,0) circle [radius=0.05];

        \node [above]at(0,0){$\overbrace{i}$};
        %\draw[dashed,-](-.05,0)--(-.05,.22)--(0.05,.22)--(0.05,0);
        %\node[above] at (0,.5)  {${\overbrace{}}$};
                \node at (0,.37)  {${a}$};
        %\node [above]at(0.5,0){$\overbrace{\quad\quad\quad}$};
        \node [above]at(0.25,0){$j$};
        \node [above]at(0.5,0){$\overbrace{~\quad k\quad}$};
        \node [above]at(0.75,0){$l$};

                \node at (0.5,.4)  {${b}$};

        \node [above]at(1,0){$k$};
        \node [above]at(1.125,0.1){$\overbrace{~\quad\quad~}$};
        \node [above]at(1.25,0){$m$};

                \node at (1.125,.35)  {${c}$};

        \end{scope}
        \node at (2.5,-.15){\Large $\Leftrightarrow$};
         \node[right] at (3.1,-.15) {$[\gamma\delta]_m^i\cdot [\beta\alpha]_j^l:: a_i\otimes b^j_{kl}, c^{km}$};

        \end{tikzpicture}
}
}
\caption{Tensor sequent}
\label{Typing_judgement_geometrically}
\end{figure}
 If tensor terms can be thought as edge-labeled graphs, then {\it ordered} pseudotyping judgements or tensor sequents correspond to particular pictorial representations of these graphs.
 Especially natural the representation is for genuine {\it typing} judgements, so we discuss it first.

Let $\Sigma=t:: \Gamma$, where $t$ is $\beta$-reduced, be an ordered typing judgement, i.e. such that $FI(t)=FI(\Gamma)^\dagger$. (When  the term is not $\beta$-reduced, we replace it with its $\beta$-reduced form.) We interpret free indices  of $\Gamma$ as vertices. For a pictorial representation we equip  the set of vertices with a particular ordering corresponding to their positioning in $\Gamma$. Say, indices occurring in the same formula are ordered from left to right, from   top to bottom (first come the upper indices, then come the lower ones), and the whole set of free indices occurring in $\Gamma$ is ordered lexicographically according to the ordering of formulas in $\Gamma$, i.e. indices occurring in $A$ come before indices occurring in $B$ if $A$ comes before $B$ in $\Gamma$. We  depict them  aligned, say, horizontally in this order.

 The edge-labeled graph on these vertices is constructed as follows. Free indices  in $\Gamma$  are in bijection with free indices/vertices of $t$ and we connect them with labeled edges corresponding to factors of $t$. That is, for any index $\mu\in FI^\bullet(\Gamma)$ we have that $\mu\in FI_\bullet(t)$ and there is unique $\nu\in FI^\bullet(t)$ such that $t$ has the form $t=[w]^\nu_\mu t'$, so that $t$, seen as a graph, contains and edge from $\mu$ to $\nu$ labeled with the word $w$. It follows that  $\nu\in FI_\bullet(\Gamma)$, and we draw an edge from $\mu$ to $\nu$ with the label $w$. In this way every index/vertex in $FI(\Gamma)$ becomes adjacent to a (unique) edge.

  The constructed graph is a specific geometric representation of the graph corresponding to $t$, the representation being induced by a particular ordering of vertices.  Note that the direction of edges is from upper indices of $\Gamma$ to lower ones (upper indices of $\Gamma$ correspond to lower indices of $t$ and vice versa). Also, note that bound indices of $\Gamma$ are not in the picture.

 When $\Sigma=t:: \Gamma$ is only a pseudo-typing judgement, i.e. there are repeated free indices in $\Gamma$, we treat it as a short expression for its $\eta$-long expansion $t'::\Gamma'$. Any pair of repeated free index occurrences in $\Gamma$ corresponds to a Kronecker delta, i.e. an edge in $t'$, connecting two {\it distinct} vertices and carrying  no label. In this case, for geometric representation of $\Sigma$, we take as  vertices the set of {\it free index occurrences} rather than indices in $\Gamma$ and order them in the same way a above. In the picture, we connect every pair of repeated free indices/vertices of $\Gamma$ with an edge carrying no label and directed from the upper occurrence to the lower one. The prescription for the remaining indices has already been described.
    An example of a concrete  pseudotyping judgement  representation is given  in Figure~\ref{Typing_judgement_geometrically}.

 Observe that $\alpha\beta\eta$-equivalent pseudo-typing judgements or sequents have identical geometric representation (up to vertex labeling). Usually it is convenient to erase indices from the picture, avoiding notational clutter.
  For example, the pseudotyping judgement
  $[xy]_{m'}^{i'}\cdot [ba]_j^l:: a_{i'}\otimes b^j_{k' l}, c^{k'm'}$
   is $\alpha$-equivalent to the one in Figure~\ref{Typing_judgement_geometrically},
  while
 $[xy]_m^i\cdot [ba]_j^l\cdot\delta_k^\mu::a_i\otimes b^j_{\mu l}, c^{km}$ is
 an
 %its
  $\eta$-expansion
  of the latter.
 Also, we note that lexical typing judgements correspond to pictures where every edge has a non-empty label.

\subsection{Sequent calculus}
\subsubsection{Rules}
We defined tensor sequents in order to define tensor grammars, which generate languages from a given lexicon of typing judgements.   But at first, we introduce the underlying purely logical system of {\it extended tensor type calculus} ({\bf ETTC}).
(The title ``extended'', introduced in \cite{Slavnov_tensor},  refers to usage of binding operators, which extend plain types of {\bf MLL}.)
The system of {\bf ETTC}, being cut-free, does not use any non-logical axioms and is independent from  terminal alphabets.

  Our default formulation  involves {\it unordered} tensor sequents. For geometric representation of the rules one needs an ordered formulation with  the explicit rule for exchange; this will be discussed in the next subsection.
\begin{defi}
The system of {\it extended tensor type calculus} ({\bf ETTC}) is given by the rules in Figure~\ref{ETTC}, where sequents are unordered, and it is assumed  that all expressions are well-formed,
  i.e. there are no forbidden index repetitions and upper occurrences match lower ones.
  \end{defi}

  The requirement that all expressions in Figure~\ref{ETTC} must be well-formed is not to be overlooked. It imposes severe restrictions on the rule premises. These are spelled out in Figure~\ref{restrictions} (recall our shorthand notation for pairs of sets introduced in the beginning of this section).
  \begin{figure}%[htb]
\centering
\subfloat[Rules
\label{ETTC}
]
{
    $
    \begin{array}{c}
        \cfrac{A\in \mathit{At}}{\vdash A,\overline{A}}({\rm{Id}})
    ~~~
    \cfrac{\vdash t::\Gamma,A~ \vdash s::
    \overline{A},\Theta}{\vdash ts::\Gamma,\Theta}({\rm{Cut}})
    ~~~
    \cfrac{\vdash t::\Gamma,~t\equiv_\beta t'}{\vdash t'::\Gamma}~(\equiv_\beta)
        \\[.4cm]
    \cfrac{ \vdash t:: \Gamma}{\vdash \delta^{i}_{j} t::\Gamma^{[\nicefrac{j}{i}]}}(\alpha\eta^\to)
        ~~~
        \cfrac{\vdash\delta^{i}_{j} t:: \Gamma^{[\nicefrac{j}{i}]}}{ \vdash t::\Gamma}(\alpha\eta^\leftarrow)
        ~~~
        \cfrac{\vdash t::\Gamma}{\vdash\delta^{j}_{i} t:: \Gamma_{[\nicefrac{j}{i}]}}(\alpha\eta_\to)
        ~~~
    \cfrac{\vdash\delta^{i}_{j} t:: \Gamma_{[\nicefrac{i}{j}]}}{t \vdash \Gamma}(\alpha\eta_\leftarrow)
        \\[.4cm]
     \cfrac{\vdash t::
        \Gamma,A,B}{\vdash t::
        \Gamma,A\parr B}
                (\parr)
       ~~~
        \cfrac{\vdash t::\Gamma, A~\vdash
        s::B,\Theta}{\vdash ts::\Gamma,A\otimes B,\Theta}(\otimes)
        ~~~
        \cfrac{
            \vdash\delta_\nu^\mu t ::\Gamma,A
            }
            {
            \vdash t::\Gamma,\nabla^\mu_\nu A
            }(\nabla)
        ~~~
        \cfrac{
            \vdash t::A,\Gamma
            }
            {
            \vdash\delta^\nu_\mu t:: \triangle_\nu^\mu A,\Gamma
            }(\triangle)
    \end{array}
        $
}
%%%%%%%%%%%%%%%%%%%%%%%%%%%%%%%%%%%%%%%%%%%%%%%%%%%%%%%%%%%%%%%%%%
\\[.3cm]
\subfloat[Restrictions on rules\label{restrictions}]
{
$
\begin{array}{c}
(\alpha\eta^\to):~
i\in FI^\bullet(\Gamma),~ j\not\in |I(t)|\cup|FI(\Gamma)|.
\quad\quad(\alpha\eta_\to):~
i\in FI_\bullet(\Gamma),~ j\not\in |I(t)|\cup|FI(\Gamma)|.\\[.1cm]
(\otimes):~
|I(t)|\cap |I(s)|=|FI(\Gamma,A)|\cap |FI(B,\Theta)|=\emptyset,~ I(t)\perp FI(B,\Theta),~ I(s)\perp FI(\Gamma,A).\\[.1cm]
(\mbox{Cut}):~I(t)\perp I(s),~ FI(\Gamma)\perp FI(\Theta),~FI(t)\cap FI(s)^\dagger\subseteq FI(\Gamma)^\dagger\cap FI(\Theta)\subseteq FI(A).\\[.1cm]
(\nabla),(\triangle):~(\nu,\mu)\in FI(A).
\end{array}
$
}
\caption{{\bf ETTC} sequent calculus}
\end{figure}
\begin{prop}
The rules in Figure~\ref{ETTC} transform well-formed tensor sequents to well-formed tensor sequents iff the premises satisfy the conditions in Figure~\ref{restrictions}.
\end{prop}
\begin{proof}
\hfill
\begin{itemize}
\item $(\alpha\eta^\to)$, $(\alpha\eta_\to)$: Sufficiency is obvious. The rule adds two matching occurrences of the fresh index $j$, and one free occurrence of $i$ in the tensor context gets moved to the term without changing the polarity. Let us check necessity. For definiteness, consider the $(\alpha\eta^\to)$ rule, the other case being identical.

There are two possibilities for the term in the conclusion of the rule to be well-formed. Either $i\in FI_\bullet(t)$,
$i\not\in I^\bullet(t)$ or
  $i\not\in |I(t)|$. If the first possibility holds, then it must be that $i\in FI^\bullet(\Gamma)$ for the premise to be well-formed. Assume that the second possibility holds.  Then it must be that $i\in FI_\bullet(\Gamma^{[\nicefrac{j}{i}]})=FI_\bullet(\Gamma)$ for the conclusion of the rule to be well-formed. But $i\in FI_\bullet(\Gamma)$ and $i\not\in |I(t)|$ implies, again, $i\in FI^\bullet(\Gamma)$ or the premise is not well-formed. Thus, in both cases the  condition $i\in FI^\bullet(\Gamma)$ must hold.

  Consider the second condition, $j\not\in |I(t)|\cup|FI(\Gamma)|$.  We note that the conclusion of the rule has a free upper occurrence of $j$ located in the tensor context (the one replacing $i$ in $\Gamma$) and an occurrence of $j$ in the term located in the factor $\delta^i_j$. If $j\in |FI(\Gamma)|$  then there is a third occurrence of $j$ in the conclusion located in the tensor context, and if  $j\in|I(t)|$ then there is a third occurrence of $j$ in the conclusion located in the term. In both cases the conclusion is not well-formed.
\item $(\otimes):$ Sufficiency is very easy. We also need to check that all conditions are indeed necessary. Consider, for example, the first one that
$|I(t)|\cap |I(s)|=\emptyset$.
Assume, for a contradiction, that  $i\in |I(t)|\cap |I(s)|$. For the term $ts$ in the conclusion to be well-formed, it must be that $i$ occurs in $t$ and in $s$ freely (without repetitions) and with opposite polarities. It follows that $i$ also has free occurrences in the contexts $\Gamma,A$ and $B,\Theta$, otherwise the premises of the rule are not well-formed pseudo-typing judgements. But then the conclusion of the rule is not well-formed, because $(i,i)\in I(ts)\cap FI(\Gamma,A\otimes B,\Theta)$. The other  conditions are obtained similarly.
\item $(\mbox{Cut})$: Similar to the $(\otimes)$ case.
\item $(\nabla),(\triangle)$: The condition is necessary simply for the formula in the conclusion to be well-formed. For sufficiency we observe that there are also  an {\it upper} and a {\it lower} free occurrence of $\mu$ and  $\nu$ respectively in the premise, and a simple analysis shows that indices in the conclusion match correctly. \qedhere
\end{itemize} \end{proof}

\subsubsection{Geometric meaning}
\begin{figure}[htb]
\centering
%%%%%%%%%%%%%%%%%%%%%%%%%%%%%%%%%%%%%%%%%%%%%%%%%%%%%%%%%%%%%%%%%%%
%\\
\subfloat[Identity/structure\label{structural rules picture}]
{
    \scalebox{.8}
      {
         \tikz[xscale=.5]
            {
            \begin{scope}[shift={(2,1.25)}]
             \draw[thick,-](0.5,0) to  [out=-90,in=180] (2,-.75) to  [out=0,in=-90](3.5,0);
          \draw[thick,-](0,0) to  [out=-90,in=180] (2,-1) to  [out=0,in=-90](4,0);
          \node[above] at (.25,0){$A$};
          \node[above] at (3.75,0){$\overline{A}$};
          \node at (4.75,-.5){$({\rm{Id}})$};

                \draw [fill] (.1,0) circle [radius=0.01];
                \draw [fill] (.3,0) circle [radius=0.01];
                \draw [fill] (.2,0) circle [radius=0.01];
                \draw [fill] (.4,0) circle [radius=0.01];

                \draw [fill] (3.6,0) circle [radius=0.01];
                \draw [fill] (3.7,0) circle [radius=0.01];
                \draw [fill] (3.8,0) circle [radius=0.01];
                \draw [fill] (3.9,0) circle [radius=0.01];
            \end{scope}
                %%%%%%%%%%%%%%%%%%%%%%%%%%%%%%%%%%%%%%%%%%%5

              \begin{scope}[shift={(12.25,0)}]
             \begin{scope}[shift={(1,0)}]

             \draw[draw=black](-5,.25)rectangle(-3.5,.75);\node at(-4.25,.5){$t$};
              \draw[draw=black](-3,.25)rectangle(-1.5,.75);\node at(-2.25,.5){$s$};
             \draw[thick,-](-4.75,.75)--(-4.75,1.25);

             \node[above] at(-4.75,1.25) {$\Gamma$};
             \draw[thick,-](-3.75,.75)--(-3.75,1.25);
             \node[above] at(-3.75,1.25) {$A$};

             \draw[thick,-](-1.75,.75)--(-1.75,1.25);

             \node[above] at(-1.75,1.25) {$\Theta$};
             \draw[thick,-](-2.75,.75)--(-2.75,1.25);
             \node[above] at(-2.75,1.25) {$\overline{A}$};
             \node at (-.75,.5) {$\Rightarrow$};
             \end{scope}

            \begin{scope}[shift={(-3.75,0)}]
                 \draw[draw=black](5,.25)rectangle(6.5,.75);\node at(5.75,.5){$t$};
             \draw[thick,-](5.25,.75)--(5.25,1.25);

             \node[above] at(5.25,1.25) {$\Gamma$};
             \draw[thick,-](5.75,.75)to[out=90,in=180](6.25,1.25)
             to[out=0,in=90](7.75,.75);
             %\node[above] at(2.25,2) {$A$};

                     \draw[draw=black](7.,0.25)rectangle(8.5,.75);\node at(8,.5){$s$};

             \draw[thick,-](8.25,.75)--(8.25,1.25);

             \node[above] at(8.25,1.25) {$\Theta$};
             %\draw[thick,-](2.25,.75)--(2.25,.75.25)--(1.25,.75.75)--(1.25,2);
            % \node[above] at(1.25,2) {$B$};
             \node at (9.5,.75) {$(\mbox{Cut})$};
             \end{scope}
             \end{scope}
             %%%%%%%%%%%%%%%%%%%%%%%%%%%%%%%%%%%%%%%%%%%%%%%%%%%%%%%%%
             \begin{scope}[shift={(25,0)}]
             \draw[draw=black](-5,.25)rectangle(-1.5,.75);\node at(-3.25,.5){$t$};
             \draw[thick,-](-4.75,.75)--(-4.75,1.5);

             \node[above] at(-4.75,1.5) {$\Gamma$};
             \draw[thick,-](-3.75,.75)--(-3.75,1.5);
             \node[above] at(-3.75,1.5) {$A$};

             \draw[thick,-](-1.75,.75)--(-1.75,1.5);

             \node[above] at(-1.75,1.5) {$\Theta$};
             \draw[thick,-](-2.75,.75)--(-2.75,1.5);
             \node[above] at(-2.75,1.5) {$B$};
             \node at (-.75,.5) {$\Rightarrow$};

                     \draw[draw=black](0,.25)rectangle(3.5,.75);\node at(1.75,.5){$t$};
             \draw[thick,-](.25,.75)--(.25,1.5);

             \node[above] at(.25,1.5) {$\Gamma$};
             \draw[thick,-](1.25,.75)--(1.25,1)--(2.25,1.25)--(2.25,1.5);
             \node[above] at(2.25,1.5) {$A$};

             \draw[thick,-](3.25,.75)--(3.25,1.5);

             \node[above] at(3.25,1.5) {$\Theta$};
             \draw[thick,-](2.25,.75)--(2.25,1)--(1.25,1.25)--(1.25,1.5);
             \node[above] at(1.25,1.5) {$B$};
             \node at (4.5,.75) {$(\mbox{Ex})$};
             \end{scope}
             }
     }
}\\[.1cm]
%%%%%%%%%%%%%%%%%%%%%%%%%%%%%%%%%%%%%%%%%%%%%%%%%%%%%%%%%%%%%%%%%%%%%%%%%%%%%%%%%%%
\subfloat[{$(\nabla)$ rule}
        \label{nabla_geometrically}
   ]
{
         \begin{tikzpicture}[xscale=1.5,yscale=1.5]
         \draw[draw=black](1.2,-.2)rectangle(-.5,-.5);\node at(.35,-.35){$t$};

         \begin{scope}[shift={(.1,0)}]
            \begin{scope}[shift={(-.3,.2)}]
                \draw [fill] (-.1,0) circle [radius=0.01];
                \draw [fill] (-0.05,0) circle [radius=0.01];
                \draw [fill] (0,0) circle [radius=0.01];
                \draw [fill] (0.05,0) circle [radius=0.01];
                        \draw [fill] (0.1,0) circle [radius=0.01];
                %\draw[thick,dashed, -](-.2,0)--(-.2,.25) -- (.2,.25)-- (.2,0);
                        \node at (0,0.1) [above] {$\Gamma$};
                        \node at (0,0.1)  {$\overbrace{~~\quad~~}$};
            \end{scope}

                \draw[thick, -](-.5,.2)--(-.5,-.2);
                \draw[thick, -](-.1,.2)--(-.1,-.2);

         \end{scope}

                \begin{scope}[shift={(.5,.2)}]

                \draw[thick,-](-.25,0.2)--(-.25,-.4);
                \draw[thick,-](.55,0.2)--(.55,-.4);

                \draw [fill] (-.1,0) circle [radius=0.01];
                \draw [fill] (-.15,0) circle [radius=0.01];
                \draw [fill] (0.15,0) circle [radius=0.01];
                \draw [fill] (0.2,0) circle [radius=0.01];
                \draw [fill] (0.1,0) circle [radius=0.01];
                \draw [fill] (.4,0) circle [radius=0.01];
                \draw [fill] (.45,0) circle [radius=0.01];

                        \node at (.15,0.3)  {$\overbrace{\quad\quad\quad}$};
                        \node at (.15,.3) [above] {$A$};

                \begin{scope}[shift={(0,-.05)}]
                    \draw [fill] (0.3,0) circle [radius=0.01];
                    \draw [fill] (0,0) circle [radius=0.05];
                    \node at (0,0) [above] {$\nu$};
                    \node at (0.3,.0) [above] {$\mu$};
                    \draw[thick,->](0,0) to  [out=-90,in=180] (.15,-.25) to  [out=0,in=-90](0.3,0);
                \end{scope}

                \end{scope}

\node at(1.6,-.35){$\Rightarrow$};

            \begin{scope}[shift={(2.5,0)}]
         \draw[draw=black](1.2,-.2)rectangle(-.5,-.5);\node at(.35,-.35){$t$};

         \begin{scope}[shift={(.1,0)}]
            \begin{scope}[shift={(-.3,.2)}]
                \draw [fill] (-.1,0) circle [radius=0.01];
                \draw [fill] (-0.05,0) circle [radius=0.01];
                \draw [fill] (0,0) circle [radius=0.01];
                \draw [fill] (0.05,0) circle [radius=0.01];
                        \draw [fill] (0.1,0) circle [radius=0.01];
                \node at (0,0.1) [above] {$\Gamma$};
                        \node at (0,0.1)  {$\overbrace{~~\quad~~}$};
            \end{scope}

                \draw[thick, -](-.5,.2)--(-.5,-.2);
                \draw[thick, -](-.1,.2)--(-.1,-.2);

         \end{scope}

                \begin{scope}[shift={(.5,.2)}]

                \draw[thick,-](-.25,0.2)--(-.25,-.4);
                \draw[thick,-](.55,0.2)--(.55,-.4);

                \draw [fill] (-.1,0) circle [radius=0.01];
                \draw [fill] (-.15,0) circle [radius=0.01];
                \draw [fill] (0.15,0) circle [radius=0.01];
                \draw [fill] (0.2,0) circle [radius=0.01];
                \draw [fill] (0.1,0) circle [radius=0.01];
                \draw [fill] (.4,0) circle [radius=0.01];
                \draw [fill] (.45,0) circle [radius=0.01];

                        \node at (.15,.3) [above] {$\nabla_\nu^\mu A$};
                        \node at (.15,0.3)  {$\overbrace{\quad\quad\quad}$};

                \begin{scope}[shift={(0,-.05)}]
                    %\draw [fill] (0.3,0) circle [radius=0.01];
%                    \draw [fill] (0,0) circle [radius=0.05];
                    %\node at (0,0) [above] {$\nu$};
%                    \node at (0.3,.01) [above] {$\mu$};
%                    \draw[thick,->](0,0) to  [out=-90,in=180] (.15,-.25) to  [out=0,in=-90](0.3,0);
                \end{scope}

                \end{scope}

                \end{scope}

        \end{tikzpicture}
 }\\[.1cm]
 %%%%%%%%%%%%%%%%%%%%%%%%%%%%%%%%%%%%%%%%%%%
\subfloat[{$(\triangle)$ rule}
        \label{delta_geometrically}
 ]
 {
          \begin{tikzpicture}[xscale=1.5,yscale=1.5]
         \draw[draw=black](1.2,-.2)rectangle(-.5,-.5);\node at(.35,-.35){$t$};

         \begin{scope}[shift={(.1,0)}]
            \begin{scope}[shift={(-.3,.2)}]
                \draw [fill] (-.1,0) circle [radius=0.01];
                \draw [fill] (-0.05,0) circle [radius=0.01];
                \draw [fill] (0,0) circle [radius=0.01];
                \draw [fill] (0.05,0) circle [radius=0.01];
                        \draw [fill] (0.1,0) circle [radius=0.01];
                \node at (0,0.1) [above] {$\Gamma$};
                        \node at (0,0.1)  {$\overbrace{~~\quad~~}$};
            \end{scope}

                \draw[thick, -](-.5,.2)--(-.5,-.2);
                \draw[thick, -](-.1,.2)--(-.1,-.2);

         \end{scope}

                \begin{scope}[shift={(.5,.2)}]

                \draw[thick,-](-.25,.2)--(-.25,-.4);
                \draw[thick,-](.55,.2)--(.55,-.4);

                \draw[thick,->](0,-.05)--(0,-.4);
                \draw[thick,<-](.3,-.05)--(.3,-.4);

                \draw [fill] (-.1,0) circle [radius=0.01];
                \draw [fill] (-.15,0) circle [radius=0.01];
                \draw [fill] (0.15,0) circle [radius=0.01];
                \draw [fill] (0.2,0) circle [radius=0.01];
                \draw [fill] (0.1,0) circle [radius=0.01];
                \draw [fill] (.4,0) circle [radius=0.01];
                \draw [fill] (.45,0) circle [radius=0.01];

                \node at (.15,0.3)  {$\overbrace{\quad\quad\quad}$};
                        \node at (.15,.3) [above] {$A$};

                \begin{scope}[shift={(0,-.05)}]
                    \draw [fill] (0.3,0) circle [radius=0.01];
                    \draw [fill] (0,0) circle [radius=0.05];
                    \node at (0,0) [above] {$\nu$};
                    \node at (0.3,.0) [above] {$\mu$};
                    %\draw[thick,->](0,0) to  [out=-90,in=180] (.15,-.25) to  [out=0,in=-90](0.3,0);
                \end{scope}

         \end{scope}

\node at(1.6,-.35){$\Rightarrow$};

            \begin{scope}[shift={(2.5,0)}]
         \draw[draw=black](1.2,-.2)rectangle(-.5,-.5);\node at(.35,-.35){$t$};

         \begin{scope}[shift={(.1,0)}]
            \begin{scope}[shift={(-.3,.2)}]
                \draw [fill] (-.1,0) circle [radius=0.01];
                \draw [fill] (-0.05,0) circle [radius=0.01];
                \draw [fill] (0,0) circle [radius=0.01];
                \draw [fill] (0.05,0) circle [radius=0.01];
                        \draw [fill] (0.1,0) circle [radius=0.01];
               \node at (0,0.1) [above] {$\Gamma$};
                        \node at (0,0.1)  {$\overbrace{~~\quad~~}$};
            \end{scope}

                \draw[thick, -](-.5,.2)--(-.5,-.2);
                \draw[thick, -](-.1,.2)--(-.1,-.2);

         \end{scope}

                \begin{scope}[shift={(.5,.2)}]

                \draw[thick,-](-.25,.2)--(-.25,-.4);
                \draw[thick,-](.55,.2)--(.55,-.4);

                \draw[thick,-](0,-.05)--(0,-.4);
                \draw[thick,-](.3,-.05)--(.3,-.4);

                \draw [fill] (-.1,0) circle [radius=0.01];
                \draw [fill] (-.15,0) circle [radius=0.01];
                \draw [fill] (0.15,0) circle [radius=0.01];
                \draw [fill] (0.2,0) circle [radius=0.01];
                \draw [fill] (0.1,0) circle [radius=0.01];
                \draw [fill] (.4,0) circle [radius=0.01];
                \draw [fill] (.45,0) circle [radius=0.01];

                \node at (.15,0.3)  {$\overbrace{\quad\quad\quad}$};
                        \node at (.15,.3) [above] {$\triangle_\nu^\mu A$};

                \begin{scope}[shift={(0,-.05)}]
                    %\draw [fill] (0.3,0) circle [radius=0.01];
                    %\draw [fill] (0,0) circle [radius=0.05];
                    %\node at (0,0) [above] {$\nu$};
                    %\node at (0.3,.01) [above] {$\mu$};
                    \draw[thick,-](0,0) to  [out=90,in=180] (.15,.2) to  [out=0,in=90](0.3,0);
                \end{scope}

         \end{scope}
                \end{scope}

        \end{tikzpicture}
        }
%%%%%%%%%%%%%%%%%%%%%%%%%%%%%%%%%%%%%%%%%%%%%%%%%%%%%%%%
\caption{Pictures for {\bf ETTC}}
\label{ETTC pictures}
\end{figure}
In order to give geometric representation of the {\bf ETTC} rules we need to consider {\it ordered} sequents and enrich the system with the rule of exchange
\[\cfrac{\vdash t::\Gamma,A,B,\Theta}{\vdash t::\Gamma,B,A,\Theta}~({\rm{Ex}}).\]
 Clearly, this results in an equivalent formulation.
Then, using geometric representation of tensor sequents, the rules of {\bf ETTC} can be illustrated as in Figure~\ref{ETTC pictures}.

The identity/structure group  $\{({\rm{Id}}), ({\rm{Cut}}), ({\rm{Ex}}))\}$  is schematically illustrated  in Figure~\ref{structural rules picture}. As for the multiplicative group, the $(\otimes)$ rule puts two graphs together in the disjoint union and the $(\parr)$ rule  does nothing in the picture. The $(\equiv_\beta)$ and $(\alpha\eta)$ rules do not change the picture either (except for vertex labels); they could be thought of as {\it coordinate transformations}.

Finally, the $(\triangle)$ rule  glues together  two  indices/vertices, and the $(\nabla)$ rule is applicable only in the case when the corresponding indices/vertices are connected with an edge carrying no label. Then this edge  is erased from the picture completely (the information about the erased edge is stored in the introduced type). This is illustrated  in Figures~\ref{nabla_geometrically},~\ref{delta_geometrically} (the lines of dots denote that there might be other vertices around).

\subsubsection{Towards first order translation}
We will discuss the exact relationship of {\bf ETTC} with the linguistic fragment of {\bf MLL1} in due course, but the outline of their correspondence is the following. Indices in tensor formulas correspond to  variables in predicate formulas, with upper/lower polarities of indices corresponding to left/right polarities of variables. Multiplicative connectives translate to themselves, while the binding operators $\nabla/\triangle$ of {\bf ETTC} correspond to the $\forall/\exists$ quantifiers in the linguistic fragment. A quantifier in the linguistic fragment binds exactly two variable occurrences of opposite polarities. These two bound occurrences in a predicate formula translate to two {\it different} indices bound by the corresponding operator in the tensor formula: $\forall xA(x,x)$ translates to $\nabla^i_jA^j_i$.

Tensor sequents may also contain terms, and these do not translate to the first order language directly. However, we note that Kronecker deltas behave much like (some rudimentary versions of) equalities to the left of the turnstile. Typically, the tensor sequent $\vdash \delta^{i_1\ldots i_k}_{j_1\ldots j_k}::\Gamma$ could be thought intuitively as the sequent
\[
j_1=i_1,\ldots,j_k=i_k\vdash\Gamma
\]
 in a first order language with equality. At least, the $(\alpha\eta)$ rules are consistent with such an interpretation. And the $(\equiv_\beta)$ rule, in the absence of terminal symbols, amounts to the equivalences $\delta^i_j\delta^j_k\equiv_\beta\delta^i_k$ and $\delta^i_i\equiv_\beta 1$, which, basically correspond to transitivity and reflexivity of equality (if we translate the term $1$ as the empty context). This intuition, probably, could be developed more formally.

\subsubsection{Properties}
The crucial property of cut-elimination in {\bf ETTC} will be established in a separate section below. Here we collect some simple observations about the system.
      \begin{prop}\label{derivability closed under equivalence}
 In {\bf ETTC}:
 \begin{enumerate}
 [(i)]
 \item the sequent $\vdash F,\overline F$ is derivable for any formula $F$;
 \item $\alpha$-equivalent formulas are provably equivalent, i.e. if $F\equiv_\alpha F'$ then the sequent $\vdash \overline F, F'$ is derivable;
 \item $\beta\eta$-equivalent sequents are cut-free derivable from each other;
 \item $\alpha$-equivalent sequents are derivable from each other.
 \end{enumerate}
  \end{prop}
 \begin{figure}%[htb]
\centering
%%%%%%%%%%%%%%%%%%%%%%%%%%%%%%%%%%%%%%%%%%%%%%%%%%%%%%%%%%%%%%%%%%%%%%%%%%%%%%
\subfloat[Claim {\it(ii)}
\label{alpha equiv formulas are equivalent figure}
]
{$\quad$
\def\fCenter{\ \vdash\ }
\Axiom$\fCenter \overline A,A$
\RightLabel{($\alpha\eta^\to$)}
\UnaryInf$ \fCenter \delta^\nu_y::\overline A,A^{[\nicefrac{y}{\nu}]}$
\RightLabel{($\alpha\eta_\to$)}
\UnaryInf$ \fCenter \delta^{x\nu}_{\mu y}::\overline A,A^{[\nicefrac{y}{\nu}]}_{[\nicefrac{x}{\mu}]}$
\RightLabel{($\triangle$)}
\UnaryInf$ \fCenter \delta^{\mu x\nu}_{\nu\mu y}::\triangle^\nu_\mu\overline A,A^{[\nicefrac{y}{\nu}]}_{[\nicefrac{x}{\mu}]}$
\RightLabel{($\equiv_\beta$)}
\UnaryInf$ \fCenter \delta^{x}_{y}::\triangle^\nu_\mu\overline A,A^{[\nicefrac{y}{\nu}]}_{[\nicefrac{x}{\mu}]}$
\RightLabel{($\nabla$)}
\UnaryInf$ \fCenter \triangle^\nu_\mu\overline A,\nabla^x_y(A^{[\nicefrac{y}{\nu}]}_{[\nicefrac{x}{\mu}]})$
\DisplayProof
}
%%%%%%%%%%%%%%%%%%%%%%%%%%%%%%%%%%%%%%%%%%%%%%%%%%%%%%%%%%%%%%%%%%%%%%%%%%%%%%
\subfloat[Claim {\it(iv)}
\label{Claim iv}
]
{
$\quad
\begin{array}{c}
\def\fCenter{\ \vdash\ }
\Axiom$\fCenter t::\Gamma\quad(i,i)\in FI(\Gamma)$
\RightLabel{($\alpha\eta^\to$)}
\UnaryInf$ \fCenter \delta^i_jt::\Gamma^{[\nicefrac{j}{i}]}$
\RightLabel{($\alpha\eta_\leftarrow$)}
\UnaryInf$ \fCenter t::\Gamma^{[\nicefrac{j}{i}]}_{[\nicefrac{j}{i}]}$
\DisplayProof
\\
\\
\def\fCenter{\ \vdash\ }
\Axiom$\fCenter t::\Gamma\quad i\in FI_\bullet(t)\cap FI^\bullet(\Gamma)$
\RightLabel{($\alpha\eta^\to$)}
\UnaryInf$ \fCenter \delta^i_jt::\Gamma^{[\nicefrac{j}{i}]}$
\RightLabel{($\equiv_\beta$)}
\UnaryInf$ \fCenter t_{[\nicefrac{j}{i}]}::\Gamma^{[\nicefrac{j}{i}]}$
\DisplayProof
\end{array}
$
}
\caption{To the proof of Proposition~\ref{derivability closed under equivalence}}
\label{derivability closed under equivalence figure}
\end{figure}
 \begin{proof} Claim (i) is easy. Claim (ii) is established by induction on $A$, $A'$, the base cases being $F=Q^\mu_\nu A$, $F'=Q^x_y (A^{[\nicefrac{y}{\nu}]}_{[\nicefrac{x}{\mu}]})$, $Q\in\{\nabla,\triangle\}$. A derivation for this case is shown in Figure~\ref{alpha equiv formulas are equivalent figure}.
  For claim (iii) note that $\beta$-equivalent sequents are derivable from each other by the $(\equiv_\beta)$ rule, and if $\Sigma'$ is a one-step $\eta$-expansion of $\Sigma$ then $\Sigma'$ is derivable from $\vdash\Sigma$ by the $(\alpha\eta^\to)$ rule, and $\vdash\Sigma$ from $\vdash\Sigma'$ by the $(\alpha\eta_\leftarrow)$ rule.

  Finally, consider claim (iv). There are two relations generating $\alpha$-equivalence of sequents. The first one involves sequents of the form $\vdash t::\Gamma,F$ and $\vdash t'::\Gamma,F'$, where $t\equiv_\alpha t'$, $F\equiv_\alpha F'$. The statement follows from (ii), (iii). The second relation involves sequents of the form $\vdash\Sigma$ and $\vdash\Sigma^{[\nicefrac{j}{i}]}_{[\nicefrac{j}{i}]}$. This has two subcases: either the index $i$ has two free occurrences in the pseudotype context in $\Sigma$, or it has one free occurrence in the pseudotype context and one  in the term. Derivations for both situations are shown in Figure~\ref{Claim iv}. \end{proof}

\subsubsection{Implicational types}
 \begin{figure}[htb]
\centering
\subfloat[{Implication type}
\label{implication_general_form}
]
{
\quad\quad\quad
\scalebox{.8}
{
\begin{tikzpicture}
\begin{scope}[xscale=1]

\draw[thick,<-](0,0) to  [out=-90,in=180] (.75,-.6) to  [out=0,in=-90](1.5,0);
\draw[thick,->](-.4,0) to  [out=-90,in=180] (.75,-1) to  [out=0,in=-90](1.9,0);
        \draw [fill] (-.4,0) circle [radius=0.05];
        \draw [fill] (1.5,0) circle [radius=0.05];
        \node [above] at (-0.2,0) {$\overline{A}$};
        \node [above] at (1.75,0) {$B$};

        \node  at (0.75,-.8) {string1};
        \node [above] at (0.75,-.6) {string2};
\end{scope}
\end{tikzpicture}
}
\quad~~
}
\subfloat[{Implication ''subtypes''}
\label{subtypes}
]
{
\quad\quad
\scalebox{.8}
{
\begin{tikzpicture}
\begin{scope}[xscale=1]
%\begin{scope}[shift={(-24,0)}]
\draw[thick,<-](0,0) to  [out=-90,in=180] (.75,-.6) to  [out=0,in=-90](1.5,0);
\draw[thick,->](-.4,0) to  [out=-90,in=180] (.75,-1) to  [out=0,in=-90](1.9,0);
        \draw [fill] (-.4,0) circle [radius=0.05];
        \draw [fill] (1.5,0) circle [radius=0.05];
        \node [above] at (-0.2,0) {$\overline{A}$};
        \node [above] at (1.75,0) {$B$};

        \node  at (0.75,-.8) {string};

\end{scope}
\end{tikzpicture}
}
%\quad
\scalebox{.8}
{
\begin{tikzpicture}
\begin{scope}[xscale=1]
%\begin{scope}[shift={(-24,0)}]
\draw[thick,<-](0,0) to  [out=-90,in=180] (.75,-.6) to  [out=0,in=-90](1.5,0);
\draw[thick,->](-.4,0) to  [out=-90,in=180] (.75,-1) to  [out=0,in=-90](1.9,0);
        \draw [fill] (-.4,0) circle [radius=0.05];
        \draw [fill] (1.5,0) circle [radius=0.05];
        \node [above] at (-0.2,0) {$\overline{A}$};
        \node [above] at (1.75,0) {$B$};

        \node [above] at (0.75,-.6) {string
        };
\end{scope}

\end{tikzpicture}
}
}\\[.2cm]
%%%%%%%%%%%%%%%%%%%%%%%%%%%%%%%%
\subfloat[{Implication subtypes action}
\label{slash action}
]
{
%\scalebox{.5}[.5]
{
%$
%\begin{array}{c}
 %\tikz
 \begin{tikzpicture}
 [xscale=.4,yscale=.5]
 %{
 \begin{scope}[shift={(-.4,0)}]

    \begin{scope}%[xscale=-1]
    \begin{scope}%[shift={(-3.2,0)}]

    \begin{scope}[shift={(-.7,0)}]%HERE .5!
            \draw[thick,<-](0,0) to  [out=-90,in=180] (0.4,-0.8) to  [out=0,in=-90] (0.8,0);
             \draw [fill] (0.8,0) circle [radius=0.1];
              \node[above ] at (0.4,0) {$A$};
               \node[above] at (0.4,-.8) {$\mbox{s}$};

    \end{scope}

    \begin{scope}[shift={(-.45,0)}]
             \draw[thick,<-](1.1,0) to  [out=-90,in=180] (2.5,-1.25) to  [out=0,in=-90] (3.9,0);
        \draw [fill] (3.9,0) circle [radius=0.1];
        \node[above] at (2.5,-1.35) {$\mbox{t}$};
          \node[above] at (1.3,0) {$\overline{A}$};
          \node[above] at (3.6,0) {$B$};
        \draw[thick,->](1.7,0) to  [out=-90,in=180] (2.5,-0.6) to  [out=0,in=-90] (3.3,0);
        \draw [fill] (1.7,0) circle [radius=0.1];

    \end{scope}

    \end{scope}
    \end{scope}

    {
    \node at(4.3,-.5){$\Longrightarrow$};%4.8
    }
 \end{scope}

\begin{scope}%[xscale=-1]
    \begin{scope}%[shift={(-14.3,0)}]%13.7!!!!!!!!!!!
        \begin{scope}[shift={(6.2,0)}]

            \begin{scope}[shift={(-1.5,0)}]%1.5!!!!!!!!!!!!!!!!!!
             \draw[thick,-](0,0) to  [out=-90,in=180] (0.4,-0.8) to  [out=0,in=-90] (0.8,0);
                  \node[above] at (0.4,-.8) {$\mbox{s}$};

            \end{scope}

            \draw[thick](0,0) to  [out=90,in=0] (-.35,.5) to  [out=180,in=90] (-.7,0);
            \draw[thick](0.6,0) to  [out=90,in=0] (-.75,1) to  [out=180,in=90] (-1.5,0);%0.3

            \begin{scope}[shift={(-1.1,0)}]

                 \draw[thick,-](1.1,0) to  [out=-90,in=180] (2.5,-1.25) to  [out=0,in=-90] (3.9,0);

                \node[above] at (2.5,-1.35) {$\mbox{t}$};

                  \node[above] at (3.6,0) {$B$};
                        \draw[thick,->](1.7,0) to  [out=-90,in=180] (2.5,-0.6) to  [out=0,in=-90] (3.3,0);
                \draw [fill] (3.9,0) circle [radius=0.1];
            \end{scope}
        \end{scope}
    \end{scope}
\end{scope}

\begin{scope}[shift={(.2,0)}]
    \node [font=\large] at(10,-.5) {
    $=$
    };

    \begin{scope}[xscale=-1]
\begin{scope}[shift={(-24,0)}]
    \begin{scope}[shift={(9.5,.25)}]

        \begin{scope}
         \draw[thick,<-](2.4,-.3) to  [out=-90,in=180] (3.1,-1.25) to  [out=0,in=-90] (3.8,-.3);
                \draw [fill] (3.8,-.3) circle [radius=0.1];

        \node[above] at (3.1,-1.3) {$\mbox{ts}$};

        \end{scope}

    \end{scope}
    \end{scope}
    \end{scope}
\end{scope}
%%%%%%%%%%%%%%%%%%%%%%%%%%%%%%%%%%%%%%%%%%%%%%%%%%%%%%%%%%%%%%%%%%%%
 \begin{scope}[shift={(0,-2.5)}]
 \begin{scope}[shift={(-.4,0)}]

    \begin{scope}[xscale=1]%-1
    \begin{scope}%[shift={(-3.2,0)}]

    \begin{scope}[shift={(-.7,0)}]%HERE .5!
            \draw[thick,<-](0,0) to  [out=-90,in=180] (0.4,-0.8) to  [out=0,in=-90] (0.8,0);
             \draw [fill] (0.8,0) circle [radius=0.1];
              \node[above ] at (0.4,0) {$A$};
               \node[above] at (0.4,-.8) {$\mbox{s}$};

    \end{scope}

    \begin{scope}[shift={(-.45,0)}]
             \draw[thick,<-](1.1,0) to  [out=-90,in=180] (2.5,-1.25) to  [out=0,in=-90] (3.9,0);
        \draw [fill] (3.9,0) circle [radius=0.1];
        \node[above] at (2.5,-.95) {$\mbox{t}$};
          \node[above] at (1.35,0) {$\overline{A}$};
          \node[above] at (3.65,0) {$B$};
        \draw[thick,->](1.6,0) to  [out=-90,in=180] (2.5,-0.9) to  [out=0,in=-90] (3.4,0);
        \draw [fill] (1.6,0) circle [radius=0.1];

    \end{scope}

    \end{scope}
    \end{scope}

    {
    \node at(4.3,-.5){$\Longrightarrow$};%4.8
    }
 \end{scope}

\begin{scope}%[xscale=-1]
    \begin{scope}%[shift={(-14.3,0)}]%13.7!!!!!!!!!!!
        \begin{scope}[shift={(6.2,0)}]

            \begin{scope}[shift={(-1.5,0)}]%1.5!!!!!!!!!!!!!!!!!!
             \draw[thick,-](0,0) to  [out=-90,in=180] (0.4,-0.8) to  [out=0,in=-90] (0.8,0);
                  \node[above] at (0.4,-.8) {$\mbox{s}$};

            \end{scope}

            \draw[thick](.0,0) to  [out=90,in=0] (-.35,.5) to  [out=180,in=90] (-.7,0);
            \draw[thick](0.5,0) to  [out=90,in=0] (-.75,1) to  [out=180,in=90] (-1.5,0);%0.3

            \begin{scope}[shift={(-1.1,0)}]

                 \draw[thick,-](1.1,0) to  [out=-90,in=180] (2.5,-1.25) to  [out=0,in=-90] (3.9,0);

                \node[above] at (2.5,-.95) {$\mbox{t}$};

                  \node[above] at (3.7,0) {$B$};
                        \draw[thick,->](1.6,0) to  [out=-90,in=180] (2.5,-0.9) to  [out=0,in=-90] (3.4,0);
                \draw [fill] (3.9,0) circle [radius=0.1];
            \end{scope}
        \end{scope}
    \end{scope}
\end{scope}

\begin{scope}[shift={(.2,0)}]
    \node [font=\large] at(10,-.5) {
    $=$
    };

    \begin{scope}[xscale=-1]
\begin{scope}[shift={(-24,0)}]
    \begin{scope}[shift={(9.5,.25)}]

        \begin{scope}
         \draw[thick,<-](2.4,-.3) to  [out=-90,in=180] (3.1,-1.25) to  [out=0,in=-90] (3.8,-.3);
            \draw [fill] (3.8,-.3) circle [radius=0.1];

        \node[above] at (3.1,-1.3) {$\mbox{st}$};

        \end{scope}

    \end{scope}
    \end{scope}
    \end{scope}
\end{scope}
\end{scope}
\end{tikzpicture}
%}
%\end{array}
%$
}
}
\caption{Geometry of implicational types}
\end{figure}
The implicational types of {\bf LC} can be translated to {\bf ETTC} as
\be\label{translating lambek types}
(b/a)^i_j=\nabla^\mu_\nu(b^i_\mu\wp \overline{a}_j^\nu ),\quad
(a\backslash b)^i_j=\nabla^\mu_\nu(\overline{a}_\mu^i\wp b_j^\nu )
\ee
   (compare with translation to {\bf MILL1} in Figure~\ref{LC2MILL}).
   Let us dicuss the geometric meaning of this.

   Terms of (pseudo-)type $\nabla^\mu_\nu A$ encode a particular  ``subtype'' of  $A$: consisting of terms whose  indices/vertices corresponding to $\mu$ and $\nu$  are connected with an edge carrying no label.
    Let $A, B$  be   types with exactly one free upper and one free lower index.  Then regular terms of type $A$ or $B$ are, essentially, strings. As for the implication type $A\multimap B$, which we encode in {\bf ETTC} as $\overline{A}\parr B$ for definiteness (the alternative encoding is $B\parr\overline{A}$), its terms are rather  pairs of strings, as Figure~\ref{implication_general_form} suggests. There are two ``degenerate subtypes'',
shown in Figure~\ref{subtypes}, where one of the two strings is empty. (They are degenerate, for example, for having no lexical terms). These are precisely  ``subtypes'' of the kind encoded by $\nabla$.
 Now, it is easy to see that terms of the first ``subtype'' act  on type $A$ terms by multiplication (concatenation) on the left, and terms of the second one act by multiplication on the right (see Figure~\ref{slash action}).
 That is, the two ``degenerate subtypes'' of the general, ``undirected'' implication  type  correspond precisely to two implicational types of {\bf LC}. Note, however, that the translations in (\ref{translating lambek types}) use both possible encodings of the undirected implication, one for the left {\bf LC} implication and the other one for the right implication. This choice seems most natural.

\section{Cut-elimination}
\subsection{\texorpdfstring{$\eta$}{η}-Long fragment}
   \begin{figure}[htb]
      \centering
      \subfloat[${\bf ETTC}_{\eta\to}$ rules
\label{ETTC eta-long}]
{
      $
      %\begin{array}{c}
      \cfrac{p\in Lit}
      {
\vdash
%\delta_{i_1}^{k_n}\cdots\delta_{i_n}^{k_1}\delta^{j_1}_{l_m}\cdots\delta^{j_m}_{l_1}
\delta_{i_1\ldots i_n l_m\ldots l_1}^{k_n\ldots k_1j_1\ldots j_m}
:: p^{i_1\ldots i_n}_{j_1\ldots j_m},\overline{p}^{l_1\ldots l_m}_{k_1\ldots k_n}
}({\rm{Id}}_{\eta_\to})
~
\cfrac{\vdash \delta^i_j t::\Gamma,A_{[\nicefrac{i}{\mu}]}^{[\nicefrac{j}{\nu}]}}
{\vdash  t::\Gamma,\nabla^\mu_\nu A}(\nabla_{\equiv_\alpha})
~
\cfrac{\vdash t::A_{[\nicefrac{i}{\mu}]}^{[\nicefrac{j}{\nu}]},\Gamma}
{\vdash \delta^j_i t::\triangle^\mu_\nu A,\Gamma}(\triangle_{\equiv_\alpha})
%\end{array}
$
}
\\[.2cm]
%%%%%%%%%%%%%%%%%%%%%%%%%%%%%%%%%%%%%%%%%%
 \subfloat[{Emulating $(\nabla_{\equiv_\alpha})$ in ${\bf ETTC}$}
\label{emulating nabla_alpha}
]
{
   $
   \def\fCenter{\ ::\ }
\def\ScoreOverhang{.1pt}
\Axiom$\vdash \delta^i_j t\fCenter\Gamma,A_{[\nicefrac{i}{\mu}]}^{[\nicefrac{j}{\nu}]}$
\RightLabel{($\alpha\eta^\to$)}
\UnaryInf$\vdash \delta^{ji}_{\nu j} t\fCenter\Gamma,A_{[\nicefrac{i}{\mu}]}$
\RightLabel{($\alpha\eta_\to$)}
\UnaryInf$\vdash \delta^{\mu ji}_{i\nu j} t\fCenter\Gamma,A$
\RightLabel{($\equiv_\beta$)}
\UnaryInf$\vdash \delta^\mu_\nu t\fCenter\Gamma,A$
\RightLabel{($\nabla$)}
\UnaryInf$\vdash t\fCenter\Gamma,\nabla^\mu_\nu A$
\DisplayProof
\quad\quad
$
}
%%%%%%%%%%%%%%%%%%%%%%%%%%%%%%%%%%%%%%%%%%
 \subfloat[{Emulating $(\triangle_{\equiv_\alpha})$ in ${\bf ETTC}$}
\label{emulating triangle_alpha}
]
{
   $
   \quad\quad
   \def\fCenter{\ ::\ }
\def\ScoreOverhang{.1pt}
\Axiom$\vdash t\fCenter A_{[\nicefrac{i}{\mu}]}^{[\nicefrac{j}{\nu}]},\Gamma$
\RightLabel{($\alpha\eta^\to$)}
\UnaryInf$\vdash \delta^{j}_{\nu} t\fCenter A_{[\nicefrac{i}{\mu}]},\Gamma$
\RightLabel{($\alpha\eta_\to$)}
\UnaryInf$\vdash \delta^{\mu j}_{i\nu} t\fCenter A,\Gamma$
\RightLabel{($\triangle$)}
\UnaryInf$\vdash \delta^{\nu\mu j}_{\mu i\nu} t\fCenter \triangle_{\nu}^{\mu} A,\Gamma$
\RightLabel{($\equiv_\beta$)}
\UnaryInf$\vdash \delta^j_i t\fCenter \triangle_{\nu}^{\mu}A,\Gamma$
\DisplayProof
$
}
\caption{${\bf ETTC}_{\eta\to}$}
%\label{ETTC eta-long}
\end{figure}
The {\it $\eta$-long} fragment ${\bf ETTC}_{\eta\to}$ is {\bf ETTC} restricted to $\eta$-long sequents, i.e. to typing judgements. The fragment can be given a separate axiomatization.
\begin{defi}
The system ${\bf ETTC}_{\eta\to}$  is obtained from  {\bf ETTC}  by removing the $(\alpha\eta)$ rules and replacing the $({\rm{Id}})$ axioms with their $\eta$-long forms $({\rm{Id}}_{\eta_\to})$ and the $(\nabla)$, $(\triangle)$ rules with their closures $(\nabla_{\equiv_\alpha})$, $(\triangle_{\equiv_\alpha})$ under $\alpha$-equivalence shown in Figure~\ref{ETTC eta-long}.
\end{defi}
\begin{prop}\label{ETTC is a fragment}
If a typing judgement $\Sigma$ is derivable in ${\bf ETTC}_{\eta\to}$ then it is derivable in ${\bf ETTC}$. If the ${\bf ETTC}_{\eta\to}$ derivation of $\Sigma$ is cut-free, then $\Sigma$ is cut-free derivable in ${\bf ETTC}$.
\end{prop}
\begin{proof} The $({\rm{Id}_{\eta_{\to}}})$ axioms of ${\bf ETTC}_{\eta\to}$ are {\bf ETTC} derivable from $({\rm{Id}})$ axioms using $(\alpha\eta)$ rules. The $(\nabla_{\equiv_\alpha})$, $(\triangle_{\equiv_\alpha})$  rules are cut-free emulated in {\bf ETTC} as shown in Figures \ref{emulating nabla_alpha}, \ref{emulating triangle_alpha}. \end{proof}

%%%%%%%%%%%%%%%%%%%%%%%%%%%%5
(It was precisely the $\eta$-long fragment that was introduced and discussed in \cite{Slavnov_tensor} under the name {\bf ETTC}, with the exception that in \cite{Slavnov_tensor} we did  not have the term reduction $\delta^i_i\to_\beta 1$. Also, in \cite{Slavnov_tensor} types were defined as quotiented under $\alpha$-equivalence, so that the rules $(\nabla)_{\equiv_\alpha}$, $(\triangle)_{\equiv_\alpha}$ in those definitions would be identical to the original $(\nabla)$, $(\triangle)$ versions. The enriched {\bf ETTC} of this work has the advantage of more concise notation; derivations in the ${\bf ETTC}_{\eta\to}$ tend to become overloaded with indices and barely readable.)
%%%%%%%%%%%%%%%%%%%%%%%%%%%%%%%%%%%%%%%%%%%%%%%%%%%%5

The $\eta$-long fragment is notationally cumbersome, but is easy to analyze. In particular, cut-elimination in ${\bf ETTC}_{\eta\to}$ is more or less routine. For the full system {\bf ETTC} of this work, cut-elimination is trickier because of $(\alpha\eta)$ rules.

\subsection{Cut-elimination for ${\bf ETTC}_{\eta\to}$}
\begin{prop}\label{alpha-equivalent derivations}
If $\Sigma$, $\Sigma'$ are $\alpha$-equivalent typing judgements and $\pi$ is an ${\bf ETTC}_{\eta\to}$ derivation of $\vdash\Sigma$  then there exists an ${\bf ETTC}_{\eta\to}$ derivation $\pi'$ of $\vdash\Sigma'$, which is obtained from $\pi$ by renaming some indices.
\end{prop}
\begin{proof} Easy induction on  $\pi$. \end{proof}

We define the {\it size of a derivation} in ${\bf ETTC}_{\eta\to}$ as follows. If the derivation $\pi$ is an axiom then the size $\mathit{size}(\pi)$ of $\pi$ is $1$. If $\pi$ is obtained from a derivation $\pi'$ by a single-premise rule $\mathit{size}(\pi)=\mathit{size}(\pi')+1$. If $\pi$ is obtained from derivations $\pi_1$, $\pi_2$ by a two-premise rule then
$\mathit{size}(\pi)=\mathit{size}(\pi_1)+\mathit{size}(\pi_2)+1$.

%Note that the $(\equiv_\beta)$ rules do not count. (Essentially, in the analysis below we identify $\beta$-equivalent tensor terms or sequents.)

Given  an ${\bf ETTC}_{\eta\to}$ derivation with two subderivations $\pi_1$, $\pi_2$ of the sequents
$\vdash t::\Gamma,A$ and $\vdash s::\overline A, \Theta$ respectively followed with the Cut rule, we say that this application of a cut rule   is a {\it principal cut} if $\pi_1$, respectively  $\pi_2$, ends with a rule introducing the formula $A$, respectively $\overline A$. An application of the Cut rule that is not principal is a {\it side cut}.

\begin{lem}\label{side cuts}
There is an algorithm transforming an ${\bf ETTC}_{\eta\to}$ derivation $\pi$ of a sequent $\vdash\Sigma$ to a derivation $\pi'$ of the same sequent, where   $\mathit{size}(\pi')=\mathit{size}(\pi)$ and all cuts in $\pi'$ principal.
\end{lem}
\begin{proof}
Let $\pi_1,\ldots,\pi_n$ be all subderivations of $\pi$ ending with side cuts and let $N=\sum\mathit{size}(\pi_i)$. We will use $N$ as the induction parameter for proving termination of the algorithm.

The algorithm consists in permuting a side cut with the preceding rule introducing a side formula. We will consider only one case that seems most involved, namely, permutation with the $(\nabla_{\equiv\alpha})$ rule when there is a possibility of index collision. Other cases are similar or easier.

Let a side cut be between  sequents $\vdash\Sigma_1$, $\vdash\Sigma_2$ of the forms
\[
\vdash t::\nabla^\mu_\nu B,\Phi,A,\quad \vdash s:: \overline A,\Theta
\]
 respectively, where $\vdash\Sigma_1$ is obtained
 by the $(\nabla_{\equiv_\alpha})$ rule
from the sequent $\vdash\Sigma_0$ of the form
$\vdash \delta^i_jt:: B_{[\nicefrac{i}{\mu}]}^{[\nicefrac{j}{\nu}]},\Phi,A$.
 %: $\cfrac{\vdash \sigma:: B,\Gamma',A}{\vdash \delta^\beta_\alpha\sigma::\triangle^\alpha_\beta B,\Gamma',A}$,
 %while the indices $i,j$ occur in $\Sigma_2$ freely (which means $i,j\in FI(\Theta)$, since they cannot occur in $A$ and $\overline A$ or $\Sigma_1$ is %not $\eta$-long).

 The transformation of derivations is shown in Figure~\ref{side cut figure}, where
 in order to avoid possibility of forbidden index repetitions, we use Proposition~\ref{alpha-equivalent derivations} and replace the derivation of $\vdash\Sigma_0$ with a derivation with the same size and an $\alpha$-equivalent conclusion  obtained from $\Sigma_0$ by renaming $i$, $j$ with fresh indices.
  %We replace the cut between $\Sigma_1$, $\Sigma_2$ with a cut between $\Sigma_0'$, $\Sigma_2$, and apply the $(\triangle_{\equiv_\alpha})$ rule to the conclusion to obtain  the conclusion of the original cut.
  The induction parameter $N$ for the new derivation is smaller at least by 1. \end{proof}
      \begin{figure}
      \centering
      \subfloat[{Transforming a side cut}
\label{side cut figure}
]
{
   $
   \begin{array}{c}
  i',j'\mbox{ fresh},~B'=B^{[\nicefrac{j'}{\nu}]}_{[\nicefrac{i'}{\mu}]}.
  \\[.25cm]
%(\nu,\mu)\in FI(\sigma),~ (\mu,\nu)\in FI(B),~\\
\def\fCenter{\ ::\ }
\def\ScoreOverhang{.1pt}
\Axiom$\vdash \delta^i_j t\fCenter B_{[\nicefrac{i}{\mu}]}^{[\nicefrac{j}{\nu}]},\Phi,A$
\RightLabel{($\nabla_{\equiv_\alpha}$)}
\UnaryInf$ \vdash t\fCenter\nabla^\mu_\nu B,\Phi,A$
\Axiom$\vdash s\fCenter \overline A,\Theta$
\RightLabel{(${\rm{Cut}}$)}
\insertBetweenHyps{\hskip 3pt}
\BinaryInf$ \vdash t s\fCenter \nabla^\mu_\nu B,\Phi,\Theta$
\DisplayProof
\Rightarrow
\def\fCenter{\ ::\ }
\def\ScoreOverhang{.1pt}
\Axiom$\vdash \delta^{i'}_{j'} t\fCenter  B',\Phi,A$
\Axiom$\vdash s\fCenter \overline A,\Theta$
\RightLabel{(${\rm{Cut}}$)}
\insertBetweenHyps{\hskip 3pt}
\BinaryInf$ \vdash \delta^{i'}_{j'} t s\fCenter  B',\Phi,\Theta$
\RightLabel{($\nabla_{\equiv_\alpha}$)}
\UnaryInf$ \vdash t s\fCenter \nabla^{\mu}_{\nu} B,\Phi,\Theta$
\DisplayProof
\end{array}
   $
   }
   \\[.3cm]
   %%%%%%%%%%%%%%%%%%%%%%%%%%%%%%%%%
   %%%%%%%%%%%%%%%%%%%%%%%%%%%%%%%%%%%
   \subfloat[{Principal cut $\triangle/\nabla$}
   \label{cut triangle/nabla}]
   {
   $\begin{array}{c}
   i',j'\mbox{ fresh}, t'=t^{[\nicefrac{j'}{l}]}_{[\nicefrac{i'}{k}]}, A'=A^{[\nicefrac{j'}{\nu}]}_{[\nicefrac{i'}{\mu}]}.
   \\[.25cm]
   \def\fCenter{\ ::\ }
\def\ScoreOverhang{.1pt}
\Axiom$\vdash \delta_{j}^{i}s\fCenter \Gamma,A_{[\nicefrac{i}{\mu}]}^{[\nicefrac{j}{\nu}]}$
\RightLabel{($\nabla_{\equiv_\alpha}$)}
\UnaryInf$ \vdash s\fCenter \Gamma,\nabla^\mu_\nu A$
\Axiom$\vdash t\fCenter \overline A^{[\nicefrac{k}{\mu}]}_{[\nicefrac{l}{\nu}]}, \Theta$
\RightLabel{($\triangle_{\equiv_\alpha}$)}
\UnaryInf$ \vdash \delta^{k}_{l} t\fCenter \triangle^\nu_\mu \overline A,\Theta$
\RightLabel{(${\rm{Cut}}$)}
\insertBetweenHyps{\hskip 3pt}
\BinaryInf$ \vdash s\delta^{k}_{l} t\fCenter \Gamma,\Theta$
\DisplayProof
\Rightarrow
\def\fCenter{\ ::\ }
\def\ScoreOverhang{0pt}
\Axiom$\vdash \delta^{i'}_{j'}s\fCenter \Gamma,A'$
\Axiom$\vdash t'\fCenter \overline{A'}, \Theta$
\RightLabel{(${\rm{Cut}}$)}
\insertBetweenHyps{\hskip 3pt}
\BinaryInf$ \vdash \delta_{j'}^{i'} st'\fCenter \Gamma,\Theta$
\DisplayProof
\end{array}
$
}
   \caption{Cut-elimination for ${\bf ETTC}_{\eta\to}$}
   %\label{side cut figure}
 \end{figure}
 \begin{lem}\label{principal cuts}
 There is an algorithm transforming an ${\bf ETTC}_{\eta\to}$ derivation $\pi$ of a sequent $\Sigma$ with cuts into an ${\bf ETTC}_{\eta\to}$ derivation $\pi'$ of some sequent $\Sigma'$ such that $\Sigma'\equiv_{\beta}\Sigma$ and $\mathit{size}(\pi')<\mathit{size}(\pi)$, provided  that all cuts in $\pi$ are principal.
 \end{lem}
 \begin{proof} If there is a cut
 \[
 \cfrac{\vdash
  \delta_{i_1\ldots i_n l_m\ldots l_1}^{k_n\ldots k_1j_1\ldots j_m}
 :: p^{i_1\ldots i_n}_{j_1\ldots j_m},\overline{p}^{l_1\ldots l_m}_{k_1\ldots k_n
 }~
 \vdash
   \delta_{k_n\ldots k_1 t_m\ldots t_1}^{s_n\ldots s_1l_m\ldots l_1}
  :: p^{k_n\ldots k_1}_{l_m\ldots l_1},\overline{p}^{t_1\ldots t_m}_{s_1\ldots s_n}}
 {\vdash
 \delta_{i_1\ldots i_n l_m\ldots l_1k_n\ldots k_1t_m\ldots t_1}^{k_n\ldots k_1j_1\ldots j_ms_n\ldots s_1l_m\ldots l_1}
 ::
  p^{i_1\ldots i_n}_{j_1\ldots j_m},\overline{p}^{t_1\ldots t_m}_{s_1\ldots s_n}}
  \]
 between $({\rm{Id}}_{\eta\to})$ axioms then its conclusion is $\beta$-equivalent to the $({\rm{Id}}_{\eta\to})$ axiom
 \[
 \vdash \delta_{i_1\ldots i_nt_m\ldots t_1}^{s_n\ldots s_1j_1\ldots j_m}:: p^{i_1\ldots i_n}_{j_1\ldots j_m},\overline{p}^{t_1\ldots t_m}_{s_1\ldots s_n}.
 \]

 The transformation of a principal cut between a $\triangle$- and a $\nabla$-formula is shown in Figure~\ref{cut triangle/nabla}, where, similarly to  Proposition~\ref{side cuts}, we use Proposition~\ref{alpha-equivalent derivations} and replace derivations of the premises with derivations of $\alpha$-equivalent sequents. The conclusion of the new derivation is $\beta$-equivalent to the conclusion of the original one, they are obtained from each other by renaming bound indices in the terms. (That is, the terms are $\alpha$-equivalent, but $\alpha$-equivalent terms are automatically $\beta$-equivalent).

The case of a $\otimes-$ and a $\parr-$ formula is indistinguishable from the familiar case of multiplicative linear logic. \end{proof}
\begin{cor}
The system ${\bf ETTC}_{\eta\to}$ is cut-free.
\end{cor}
\begin{proof} Combining Lemma~\ref{side cuts} and Lemma~\ref{principal cuts} we prove by induction on the size of an ${\bf ETTC}_{\eta\to}$ derivation $\pi$ of a sequent $\vdash\Sigma$ that $\pi$ transforms to a cut-free derivation of some $\beta$-equivalent sequent $\vdash\Sigma'$. But $\vdash\Sigma'$ is cut-free derivable from $\vdash\Sigma$ in ${\bf ETTC}_{\eta\to}$ by the $(\equiv_\beta)$ rule. \end{proof}

 \subsection{Conservativity of ${\bf ETTC}_{\eta\to}$}
\begin{defi}
If $\Gamma$ is a pseudotype context, the {\it degeneracy} $D(\Gamma)$ of $\Gamma$ is the index set
\[
D(\Gamma)=\{i
|~i\in FI^\bullet(\Gamma)\cap FI_\bullet(\Gamma)\}.
\]
If $\Sigma$ is a pseudotyping judgement of the form $t::\Gamma$, the {\it degeneracy} $D(\Sigma)$ of $\Sigma$ is the  set $D(\Sigma)=D(\Gamma)$.
A typing judgement $\widetilde\Sigma$ is an {\it $\eta$-long representative of $\Sigma$} if $\widetilde\Sigma$ has the form
\[
\delta^{\mu_1\ldots\mu_n}_{\nu_1\ldots\nu_n}t::\Gamma^{[\nicefrac{\nu_1}{i_1},\ldots,\nicefrac{\nu_n}{i_n}]}_{[\nicefrac{\mu_1}{i_1},\ldots,
\nicefrac{\mu_n}{i_n}]},
\]
where $\{i_1,\ldots,i_n\}=D(\Sigma)$.
\end{defi}
We will also use the following notation for pairs of contexts: if $\Gamma,\Theta$ are pseudotype contexts then
\be\label{intersecting indices}
D(\Gamma;\Theta)=\{i| i\in FI^\bullet(\Gamma)\cap FI_\bullet(\Theta)\},
\ee
so that
\be\label{pair degeneracy}
D(\Gamma,\Theta)=D(\Gamma)\cup D(\Theta)\cup D(\Gamma;\Theta)\cup D(\Theta;\Gamma).
\ee
\begin{lem}\label{conservativity of eta-long fragment}
If $\Sigma$ is a pseudotyping judgement and  $\vdash_{\bf ETTC}\Sigma$, then for any $\eta$-long representative $\widetilde\Sigma$ of $\Sigma$ it holds that $\vdash_{{\bf ETTC}_{\eta\to}}\widetilde\Sigma$.
\end{lem}
\begin{proof} Induction on the {\bf ETTC} derivation of $\Sigma$.

For the base we have that $\eta$-long representatives of $({\rm{Id}})$ axioms of {\bf ETTC} are $({\rm{Id}}_{\eta_\to})$ axioms of ${\bf ETTC}_{\eta\to}$.
If $\Sigma$ is obtained from a sequent $\Sigma'$ by an $(\alpha\eta)$ or the $(\equiv_\beta)$ rule then any $\eta$-long representative of $\Sigma$ is $\alpha$-equivalent to some $\eta$-long representative of $\Sigma'$ and the statement follows from the induction hypothesis. The case of the $(\otimes)$ or the $(\parr)$ rule is very easy.

Let $\vdash\Sigma$ of the form $\vdash \delta^\mu_\nu t::\triangle^\mu_\nu A,\Gamma$ be obtained from the sequent $\vdash\Sigma'$ of the form
$\vdash t::A,\Gamma$, where $(\nu,\mu)\in FI(A)$, by the $(\triangle)$ rule. There are different possibilities on the structure of $D(\Sigma')$, we will consider the most involved case, when $\mu,\nu\in D(\Sigma')$. In this case we have $\mu,\nu\not\in D(\Sigma)$ (since repetitions of $\mu,\nu$ in the context are eliminated by the $\triangle$ operator) and, in fact, $D(\Sigma')=D(\Sigma)\cup\{\mu,\nu\}$. Any $\eta$-long representative $\widetilde\Sigma$ of $\Sigma$ has the form
\[
\delta^{\mu_1\ldots\mu_n}_{\nu_1\ldots\nu_n}\cdot\delta^\mu_\nu t::(\triangle^\mu_\nu A,\Gamma)^{[\nicefrac{\nu_1}{i_1},\ldots,\nicefrac{\nu_n}{i_n}]}_{[\nicefrac{\mu_1}{i_1},\ldots,
\nicefrac{\mu_n}{i_n}]},
\]
where $\{i_1,\ldots,i_n\}=D(\Sigma)$. It follows that
 the typing judgement $\widetilde{\Sigma'}$ of the form
\[
\delta^{\mu_1\ldots\mu_n}_{\nu_1\ldots\nu_n}\cdot\delta^{\mu'\nu}_{\mu\nu'} t::( A^{[\nicefrac{\nu'}{\nu}]}_{[\nicefrac{\mu'}{\mu}]},\Gamma)^{[\nicefrac{\nu_1}{i_1},\ldots,\nicefrac{\nu_n}{i_n}]}_{[\nicefrac{\mu_1}{i_1},\ldots,
\nicefrac{\mu_n}{i_n}]}
\]
is an $\eta$-long representative of $\Sigma'$, and, by the induction hypothesis, $\vdash_{{\bf ETTC}_{\eta\to}}\widetilde{\Sigma'}$. But $\vdash\widetilde{\Sigma}$ is ${\bf ETTC}_{\eta\to}$ derivable from $\vdash\widetilde{\Sigma'}$ by the $(\triangle_{\equiv_\alpha})$ rule followed by the $(\equiv_\beta)$ rule.
Other possibilities of the $(\triangle)$ rule are similar or easier. The case of the $(\nabla)$ rule is analogous to the $(\triangle)$ rule.

The most involved case is that of the Cut rule. Let $\vdash\Sigma$ of the form $\vdash ts::\Gamma,\Theta$ be obtained from $\vdash\Sigma_1$, $\vdash\Sigma_2$ of the forms
$\vdash t::\Gamma,A$, $\vdash s::\overline A,\Theta$ respectively by the Cut rule.
\begin{figure}
  \centering
  $
  \begin{array}{c}
  D(\Gamma)=\{i_1,\ldots,i_k\},~D(\Theta)=\{j_1,\ldots,j_l\},
     ~D({\Gamma;\Theta})=\{k_1,\ldots,k_m\},
~D({\Theta;\Gamma})=\{l_1,\ldots,l_n\}.
\\[.15cm]
\delta_\Gamma=\delta^{x_1\ldots x_k}_{y_1\ldots y_k},
\quad
\delta_\Theta=\delta^{p_1\ldots p_l}_{q_1\ldots q_l},
\quad
\delta_{\Gamma\Theta}=\delta^{r_1\ldots r_m}_{s_1\ldots s_m},
\quad
\delta_{\Theta\Gamma}=\delta^{u_1\ldots u_n}_{v_1\ldots v_n},
\\[.15cm]
\Gamma'=\Gamma^{[\nicefrac{y_1}{i_1},\ldots,\nicefrac{y_k}{i_k},\nicefrac{s_1}{k_1},\ldots,\nicefrac{s_m}{k_m}]}_{[\nicefrac{x_1}{i_1},
\ldots,\nicefrac{x_k}{i_k},\nicefrac{u_1}{l_1},\ldots,\nicefrac{u_n}{l_n}]},
\quad
 \Theta'=\Theta^{[\nicefrac{q_1}{j_1},\ldots,\nicefrac{q_l}{j_l},\nicefrac{v_1}{l_1},\ldots,\nicefrac{v_n}{l_n}]}_
 {[\nicefrac{p_1}{j_1},\ldots,\nicefrac{p_l}{j_l},\nicefrac{r_1}{k_1},\ldots,\nicefrac{r_m}{k_m}]}.
 \\[.15cm]
 D(A)=\{\xi_1,\ldots,\xi_r\},\quad\mu_1,\nu_1,\ldots\mu_r,\nu_r,\sigma_1,\ldots,\sigma_m,\tau_1,\ldots,\tau_n\mbox{ fresh},
 \\[.15cm]
 \delta_{\Gamma A}=\delta_{s_1\ldots s_m}^{\sigma_1\ldots\sigma_m},
\quad
\delta_{\overline A\Theta}=\delta^{r_1\ldots r_m}_{\sigma_1\ldots\sigma_m},
\quad
\delta_{\Theta\overline A}=\delta_{v_1\ldots v_n}^{\tau_1\ldots\tau_n},
\quad\delta_{A\Gamma}=\delta^{u_1\ldots u_n}_{\tau_1\ldots\tau_n},
\\[.15cm]
\delta_A=\delta^{\mu_1\ldots \mu_r}_{\nu_1\ldots \nu_r},
\quad\delta_{\overline A}=\delta_{\mu_1\ldots \mu_r}^{\nu_1\ldots \nu_r},
\quad
A'=A^{[\nicefrac{\nu_1}{\xi_1},\ldots,\nicefrac{\nu_r}{\xi_r}\nicefrac{\tau_1}{l_1},\ldots,\nicefrac{\tau_n}{l_n}]}
_{[\nicefrac{\mu_1}{\xi_1},\ldots,\nicefrac{\mu_r}{\xi_r}\nicefrac{\sigma_1}{k_1},\ldots,\nicefrac{\sigma_m}{k_m}]}.
  \end{array}
  $
  \caption{To the proof of Lemma~\ref{conservativity of eta-long fragment}}\label{conservativity of eta-long fragment picture}
\end{figure}

Any $\eta$-long representative $\widetilde\Sigma$ of $\Sigma$ can be written as
\be\label{tildeSigma}
 \delta_\Gamma \delta_{\Theta}\delta_{\Gamma\Theta} \delta_{\Theta\Gamma} ts::\Gamma',\Theta',
\ee
 where $\Gamma'$, $\Theta'$ are obtained from $\Gamma,\Theta$ respectively by renaming repeated free indices, and
  $\delta_\Gamma$, $\delta_{\Theta}$, $\delta_{\Gamma\Theta}$, $\delta_{\Theta\Gamma}$ are Kronecker deltas corresponding to this renaming, respectively, in $D(\Gamma)$, $D(\Theta)$, $D({\Gamma;\Theta})$, $D({\Theta;\Gamma})$ (see Figure~\ref{conservativity of eta-long fragment picture} for explicit expressions with indices).

We are going to construct $\eta$-long representatives $\widetilde{\Sigma_1}$, $\widetilde{\Sigma_2}$ of $\Sigma_1$, $\Sigma_2$ respectively such that  $\widetilde\Sigma$, up to $\beta$-equivalence, is obtained from them  by a cut.

Note that
we have the equalities
\[
D({\Gamma;\Theta})=D({\Gamma;A})=D({\overline A;\Theta}),\quad D({\Theta;\Gamma})=D(\Theta;\overline A)=D({A;\Gamma}).
\]
 Indeed, if an index $i$ occurs freely both in $\Gamma$ and $\Theta$, then, in order for $\Sigma$ to be well formed, it should have just one free occurrence in $\Gamma$ and one in $\Theta$, and no occurrences in $ts$ whatsoever. Then in order for $\Sigma_1$, $\Sigma_2$ to be well-formed, there must be a free occurrence of $i$ in $A$ (hence in $\overline{A}$) to match the occurrence in $\Gamma$ for $\Sigma_1$ and in $\Theta$ for $\Sigma_2$.
With this observation is easy to construct desired $\widetilde{\Sigma_1}$, $\widetilde{\Sigma_2}$, respectively of the forms
\[
\delta_\Gamma\delta_{\Gamma A}\delta_{A\Gamma}\delta_At::\Gamma',A',
\quad\delta_\Theta\delta_{\Theta\overline A}\delta_{\overline A\Theta}\delta_{\overline{A}}s::\overline{A'},\Theta',
\]
where we use the same notational convention for Kronecker deltas as in (\ref{tildeSigma}), so that the terms satisfy the relations
\[
\delta_{\Gamma A}\delta_{\overline A\Theta}\equiv_\beta\delta_{\Gamma\Theta},\quad
\delta_{\Theta\overline A}\delta_{A\Gamma}\equiv_\beta\delta_{\Theta\Gamma},
\quad\delta_A\delta_{\overline{A}}\equiv_\beta1.
\]
Explicit formulas are given in Figure~\ref{conservativity of eta-long fragment picture}.

 Then the cut between
$\vdash\widetilde{\Sigma_1}$, $\vdash\widetilde{\Sigma_2}$ yields a sequent $\beta$-equivalent to $\vdash\widetilde\Sigma$. Since derivability in ${\bf ETTC}_{\eta\to}$ is closed under $\beta$-equivalence, the statement follows from the induction hypothesis. \end{proof}
\begin{cor}
${\bf ETTC}_{\eta\to}$ is a conservative fragment of {\bf ETTC}. $\Box$
\end{cor}
\begin{cor}
{\bf ETTC} is cut-free.
\end{cor}
\begin{proof} Assume that $\vdash_{\bf ETTC}\Sigma$. Let $\widetilde\Sigma$ be an $\eta$-long form of $\Sigma$. In particular, $\widetilde\Sigma$ is an $\eta$-long representative of $\Sigma$, and, using the preceding lemma and cut-elimination for  ${{\bf ETTC}_{\eta\to}}$, the sequent $\vdash\widetilde{\Sigma}$ in cut-free derivable in ${\bf ETTC}_{\eta\to}$. Then, by Proposition~\ref{ETTC is a fragment}, the sequent $\vdash\widetilde\Sigma$ is cut-free derivable in {\bf ETTC}. But $\vdash\Sigma$ is cut-free derivable from $\vdash\widetilde\Sigma$ in {\bf ETTC} using $(\alpha\eta)$ rules. \end{proof}

\section{Natural deduction and grammars}
For derivations from nonlogical axioms it seems more convenient
to allow  variables standing for tensor terms,  in the style of natural deduction
(abbreviated below as  ``n.d.''). N.d.\  formulation might also be more convenient if we want to consider  the purely intuitionistic fragment of {\bf ETTC} that is restricted to types built using only binding operators, tensor and linear implication (the latter defined as $A\multimap B=\overline A\parr B$ or $A\multimap B=B\parr \overline A$).
\begin{defi}
Given a terminal alphabet $T$ and a countable set $\mathit{var}_{\otimes}$ of {\it tensor variable symbols} of different valencies, where valency $v(x)$ of the symbol $x$ is a pair of nonnegative integers, the set $\mathit{Var}_{\otimes}$ of {\it tensor variables} is the subset of
\[
\{
x^{i_1\ldots i_m}_{j_1\ldots j_n}|
~x\in \mathit{var}_{\otimes},v(x)=(m,n),~
i_1,\ldots, i_m,j_1,\ldots,j_n\in \mathit{Ind}
\}
\]
  satisfying the constraint that there are no repeated indices.
    {\it N.d.\  tensor terms} are defined  same as tensor terms in Definition \ref{tensor terms} except that   the set  $\mathit{Var}_{\otimes}$ is added to sets
   (\ref{generators}) of generators.

    {\it$\beta$-Reduction} of n.d.\  tensor terms is defined same as
    $\beta$-reduction of tensor terms in Definition \ref{beta-red}.
    {\it Lexical n.d.\  terms} are defined as in Definition \ref{misc. term definitions} with the addition that they contain no tensor variables.
 \end{defi}
 \begin{defi}
An {\it  n.d.\  tensor pseudotyping}, respectively {\it typing, judgement}   is an expression of the form $t:C$  where $t$ is an n.d.\  tensor term and $C$  a tensor pseudotype, respectively type, satisfying the same relation (\ref{judgement}) as $t$ and $\Gamma$ in Definition \ref{tensor sequent} of usual tensor pseudotyping judgement.
Lexical,  $\beta$-reduced and $\alpha$-, $\beta$- or $\eta$-equivalent pseudotyping judgements are defined as previously in Definition \ref{tensor sequent relations}.

A {\it variable declaration} is an n.d.\  {\it typing} judgement $X:A$, where $X$ is a tensor variable. An {\it n.d.\  typing context} $\Gamma$ is a finite {\it set} (rather than a multiset) of variable declarations that have no common  variable symbols and such that  for any two distinct variables $X,Y$ in $\Gamma$ we have $I(X)\perp I(Y)$.

An  {\it n.d.\  tensor sequent} is an expression of the form $\Gamma\vdash \Sigma$, where $\Gamma$ is an n.d.\  typing context, and $\Sigma$  an n.d.\  pseudotyping judgement. $\alpha$-, $\beta$-,   $\eta$-Equivalence of n.d.\  sequents are generated by the relation
\[
\Sigma\equiv_\xi\Sigma'\Longrightarrow (\Gamma\vdash\Sigma)\equiv_\xi (\Gamma\vdash\Sigma'),\mbox{ where }\xi\in\{\alpha,\beta,\eta\}.
\]
\end{defi}
\begin{defi}
 {\it Natural deduction system} ${\bf ETTC}_{n.d.\ }$ is given by the rules in Figure~\ref{ND ETTC} with the usual implicit restriction that all expressions are well-formed.
 \end{defi}
 \begin{figure}[htb]
\centering
{
    $
    \begin{array}{lr}
        X:A\vdash X:A~(\rm {Id})
        &
        \cfrac{\Gamma\vdash t:A,~t\equiv_\beta t'}{\Gamma\vdash t':A}~(\equiv_\beta)
        \\[.4cm]
       \cfrac{\Gamma \vdash t:A}{\Gamma\vdash\delta^{i}_{j} t:A^{[j/i]}}~(\alpha\eta^\to)
                &
        \cfrac{\Gamma\vdash\delta^{i}_{j} t:A^{[j/i]}}{\Gamma\vdash t:A}~(\alpha\eta^\leftarrow)
               \\[.4cm]
        \cfrac{\Gamma \vdash t:A}{\Gamma\vdash\delta_{j}^{i} t:A_{[i/j]}}~(\alpha\eta_\to)
        &
        \cfrac{\Gamma\vdash\delta^{i}_{j} t:A_{[i/j]}}{\Gamma\vdash t:A}_(\alpha\eta_\leftarrow)
        \\[.4cm]
         \cfrac{X:A,\Gamma\vdash  tX:B}{\Gamma\vdash t: \overline{A}\parr B}
        ~(\parr{\rm{I}}_l)
        &
        \cfrac{\Gamma\vdash t:{A}\quad\Theta\vdash s:\overline{A}\parr B}{\Gamma,\Theta\vdash ts:B}~(\parr{\rm{E}}_l)
        \\[.4cm]
        \cfrac{\Gamma,X:A\vdash  tX:B}{\Gamma\vdash t:B\parr \overline{A}}
        ~(\parr{\rm{I}}_r)
        &
        \cfrac{\Gamma\vdash s:B\parr \overline{A}\quad\Theta\vdash t:{A}}{\Gamma,\Theta\vdash st:B}~(\parr{\rm{E}}_r)
        \\[.4cm]
        \cfrac{\Gamma\vdash t:A\quad\Theta \vdash s:B}{\Gamma,\Theta\vdash ts:A\otimes B}(\otimes{\rm{I}})
        &
        \cfrac{\Theta\vdash s:A\otimes B\quad \Gamma,X:A,Y:B\vdash  tXY:C}{\Gamma,\Theta\vdash ts:C}(\otimes{\rm{E}})
        \\[.4cm]
        \cfrac{\Gamma\vdash \delta^\mu_\nu\cdot t:A}{\Gamma\vdash t:\nabla^\mu_\nu A}~(\nabla{\rm{I}})
        &
        \cfrac{\Gamma\vdash t:\nabla^\mu_\nu A}{\Gamma\vdash \delta^\mu_\nu\cdot t:A}~(\nabla{\rm{E}})\\[.4cm]
        \cfrac{\Gamma\vdash  t:A}{\Gamma\vdash \delta_\mu^\nu\cdot t:\triangle^\mu_\nu A}~(\triangle{\rm{I}})
        &
        \cfrac{\Gamma\vdash   t:\triangle^\mu_\nu A\quad X:A,\Theta\vdash\delta_\mu^\nu sX:B}{\Gamma,\Theta\vdash st:B}~ (\triangle{\rm{E}})
        \end{array}
    $
}
\caption{Natural deduction for {\bf ETTC} }
\label{ND ETTC}
\end{figure}
Note that there are two pairs of introduction/elimination rules for the $\parr$ connective. They correspond to two different encodings of implication. If we are interested only in the intuitionistic fragment then one of the two pairs (depending on the chosen encoding) is not needed.
\begin{prop}\label{n.d.  derivability closed under equivalence}
In  ${\bf ETTC}_{n.d.}$, $\alpha\beta\eta$-equivalent sequents are derivable from each other.
\end{prop}
\begin{proof} Similar to Proposition~\ref{derivability closed under equivalence}. \end{proof}

\begin{prop}\label{Subst}
The system ${\bf ETTC}_{n.d.\ }$ is closed under the following {\it substitution} rule
\[
\cfrac
{
\Gamma\vdash t:A\quad X:A,\Theta\vdash X\cdot s:B
}
{
\Gamma,\Theta\vdash ts:B
}
({\rm{Sb}})
\]
\end{prop}
\begin{proof} The rule is emulated with the  $(\parr{\rm{I}}_l)$ rule followed by the $(\parr{\rm{E}}_l)$ rule. \end{proof}

We defined typing contexts in n.d.\  sequents as {\it unordered} sets. However, notationally it will be convenient to consider them as ``vectors'' with enumerated components as in the following.

Let $\overrightarrow{X}=\{X_{(1)},\ldots,X_{(n)}\}$ be  a sequence of  tensor variables and  $\Gamma=A_{(1)},\ldots,A_{(n)}$ be an ordered tensor type context
 (we put sequence numbers in  brackets in order to avoid confusion with indices).
 We denote the n.d.\  typing context
 $X_{(1)}:A_{(1)},\ldots,X_{(n)}:A_{(n)}$ as  $\overrightarrow{X}:\Gamma$. Given a term $t$, we denote  the term $tX_{(1)}\cdots X_{(n)}$ as $t\overrightarrow{X}$. Finally, we write $\overline\Gamma$ for the ordered type context $\overline{A_{(n)}},\ldots,\overline{A_{(1)}}$.  (As usual, it is implicitly assumed  that all these expressions are well-formed).
 \begin{defi}
An n.d.\  tensor sequent is {\it standard} if it has the form
%\be\label{ND derivable sequent}
$\overrightarrow{X}:\Gamma\vdash t\overrightarrow{X}:B$.
%\ee
\end{defi}
\begin{prop}
Any sequent  derivable in  ${\bf ETTC}_{n.d.\ }$ is
standard. $\Box$
\end{prop}
%%%%%%%%%%%%%%%%%%%%%%%%%%%%%%%%%%%%
\begin{prop}[``Deduction theorem'']\label{extended deduction theorem}
Consider a finite set
\[
\Xi=\{\tau_{(1)}:A_{(1)},\ldots,\tau_{(n)}:A_{(n)}\}
\]
 of  n.d.\   lexical  typing judgements
with
$FI(A_{(i)})\perp FI(A_{(j)})$
%, $FI(A_{(i)})\perp FI(A_{(j)})^\dagger$
for $i\not=j$, $i,j=1,\ldots,n$.

An n.d.\  sequent $\Sigma$ of the form $\Gamma\vdash t:B$
 is ${\bf ETTC}_{n.d.\ }$ derivable
 from  $\Xi$ using each element of $\Xi$ exactly once  iff there exist an n.d.\  sequent
 $\Sigma'$ of the form
 \be\label{Sigma'}
 X_{(1)}:A_{(1)},\ldots,X_{n}:A_{(n)},\Gamma\vdash t'X_{(1)}\ldots X_{(n)}:B
 \ee
  derivable in ${\bf ETTC}_{n.d.\ }$ without nonlogical axioms such that
  \be\label{t'}
  t'\cdot\tau_{(1)}\ldots\tau_{(n)}\equiv_\beta t.
  \ee
\end{prop}
\begin{proof} If there is $\Sigma'$ satisfying (\ref{Sigma'}), (\ref{t'}), then
$\Sigma$ is ${\bf ETTC}_{n.d.\ }$ derivable from $\Xi$ and $\Sigma'$ by a series of substitution rules using each element of $\Xi$ exactly once. Thus, if $\vdash_{{\bf ETTC}_{n.d.\ }}\Sigma'$ then $\Sigma$
 is ${\bf ETTC}_{n.d.\ }$ derivable from $\Xi$ as required.  The other direction is proven by an easy induction on the n.d.\ derivation of $\Sigma$. \end{proof}

\subsection{Translation to sequent calculus}
\begin{prop}\label{translations}
Given  a standard n.d.\  sequent $\Xi$ of the form
$\overrightarrow{X}:\Gamma\vdash t\overrightarrow{X}:A$,  the expression $\Xi'$ of the form
 $
 t::A,\overline{\Gamma}
$ is  a well-defined tensor pseudotyping judgement.
\end{prop}
\begin{proof} Let us denote $\widetilde t=t\overrightarrow{X}$ and $\Theta=A,\overline\Gamma$.

The n.d.\  type context
$\overrightarrow{X}:\Gamma$, by definition, consists of  variable declarations. In any variable declaration $Y:F$ we have that $FI(Y)=FI^\dagger(F)$, because it is an n.d.\  typing judgement, and we have $I(Y)=FI(Y)$, because $Y$, being a tensor variable, has no repeated indices. It follows that, for $FI(\overrightarrow{X})=\bigcup\limits_{Y\in\overrightarrow{X}}FI(Y)$, we have
$I(\overrightarrow{X})=FI^\dagger(\Gamma)=FI(\overline\Gamma)
$.

Since $\widetilde t$ is  a well-formed term we have
that
$I(t)\perp I(\overrightarrow{X})$, $I(Y)\perp I(Y')~\forall Y,Y'\in \overrightarrow{X}$.
Since $\Sigma$
is a well-defined n.d.\  pseudotyping judgement
we have $I(\widetilde t)\perp FI(A)$,
$FI(\overrightarrow{X})\perp FI(A)$.
It follows that
$FI(\overline\Gamma)\perp FI(A)$, $FI(F)\perp FI(F')~\forall F,F'\in \overrightarrow{\Gamma}$,
$I(t)\perp FI(\overline\Gamma)$, and $I(t)\perp FI(A)$.
 The first two conditions mean that $\Theta$ is a well-defined pseudotype context, and the last two mean that $I(t)\perp FI(\Theta)$.

In order for $\Xi'$ to be a pseudotyping judgement it remains to show that $I^\bullet(t)\cup FI^\bullet(\Theta)=I_\bullet(t)\cup FI_\bullet(\Theta)$.
Again, since $\Sigma$
is an n.d.\  pseudotyping judgement, we have that
$I^\bullet(\widetilde t)\cup FI^\bullet(A)=I_\bullet(\widetilde t)\cup FI_\bullet(A)$.
But  $I(\widetilde t)=I(t)\cup I(\overrightarrow{X})=I(t)\cup FI(\overline{\Gamma})$ and $FI(\Theta)=FI(\overline\Gamma)\cup FI(A)$. The desired equality follows.
\end{proof}
\begin{defi}
In the notation of Proposition~\ref{translations}, the tensor sequent $\vdash\Xi'$ is the {\it sequent calculus translation} of a standard n.d.\  sequent
$\Xi$.
\end{defi}

\begin{lem}\label{sequent calculus->ND}
A standard n.d.\  sequent is derivable in  ${\bf ETTC}_{n.d.\ }$ whenever its {sequent calculus translation}
is derivable in ${\bf ETTC}$.
\end{lem}
\begin{figure}
\centering
$
\begin{array}{c}
\cfrac{\vdash  \delta^i_jt::A,\overline\Theta,(\overline B)^{[\nicefrac{j}{i}]}}
{\vdash  t::A,\overline\Theta,(\overline B)}~
(\alpha\eta^\leftarrow)
\quad
\Longrightarrow
\\
   \def\fCenter{\ \vdash\ }
\def\ScoreOverhang{0pt}
\Axiom$Z: B\fCenter Z: B$
\RightLabel{($\alpha\eta_\to$)}
\UnaryInf$Z: B\fCenter \delta^j_iZ: B_{[\nicefrac{j}{i}]}
 $
\AxiomC{\it{(Ind.~hyp.)}}
\UnaryInf$Y: B_{[\nicefrac{j}{i}]},\overrightarrow{X}:{\Theta},\fCenter   \delta^i_jtY\overrightarrow{X}:A$
\RightLabel{(Sb)}
\BinaryInf$Z:B,\overrightarrow{X}:\Theta \fCenter \delta^j_i\delta^i_j tZ\overrightarrow{X}:A$
\RightLabel{($\equiv_\beta$)}
\UnaryInf$Z:B,\overrightarrow{X}:\Theta \fCenter tZ\overrightarrow{X}:A$
\DisplayProof
\\
\\
%%%%%%%%%%%%%%%%%%%%%%%%%%%%%%%%%%%%%%%%%%
%%%%%%%%%%%%%%%%%%%%%%%%%%%%%%%%%%%%%%%%%%
\cfrac{\vdash \delta^\mu_\nu t::A,\overline\Theta,\overline B}
{\vdash t::A,\overline\Theta,\nabla^\mu_\nu\overline B}
~(\nabla)
\quad
\Longrightarrow
\\
   \def\fCenter{\ \vdash\ }
\def\ScoreOverhang{.1pt}
\Axiom$Z: \triangle^\nu_\mu B\fCenter Z: \triangle^\nu_\mu B$
\AxiomC{\it{(Ind.~hyp.)}}
\UnaryInf$Y: B,\overrightarrow{X}:{\Theta}\fCenter \delta^\mu_\nu  tY\overrightarrow{X}:A$
\RightLabel{($\triangle{\rm{E}}$)}
\BinaryInf$Z: \triangle^\nu_\mu B,\overrightarrow{X}:\Theta \fCenter tZ\overrightarrow{X}:A$
\DisplayProof
\\
%%%%%%%%%%%%%%%%%%%%%%%%%%%%%%%%%%%%%%%%%%
\\
\cfrac{\vdash  t::A,\overline\Theta,\overline B}
{\vdash \delta^\mu_\nu t::A,\overline\Theta,\triangle_\mu^\nu \overline B}
~(\triangle)
\quad
\Longrightarrow
\\
   \def\fCenter{\ \vdash\ }
\def\ScoreOverhang{.1pt}
\Axiom$Z: \nabla_\nu^\mu B\fCenter Z: \nabla_\nu^\mu B $
\RightLabel{($\nabla{\rm{E}}$)}
\UnaryInf$Z: \nabla_\nu^\mu B\fCenter \delta^\mu_\nu Z:  B $
\AxiomC{\it{(Ind.~hyp.)}}
\UnaryInf$Y: B,\overrightarrow{X}:{\Theta}\fCenter   tY\overrightarrow{X}:A$
\RightLabel{(Sb)}
\BinaryInf$Z: \nabla_\nu^\mu B,\overrightarrow{X}:\Theta \fCenter t\delta^\mu_\nu Z\overrightarrow{X}:A$
\DisplayProof
\\
%%%%%%%%%%%%%%%%%%%%%%%%%%%%%%%%%%%%%%%%
\\
\cfrac{\vdash t::A,\overline\Theta,\overline B\quad\vdash s::\overline C,\overline{\Psi}}
{\vdash ts::A,\overline{\Theta},\overline B\otimes \overline C,\overline{\Psi}}
~(\otimes)
\quad
\Longrightarrow
\\
   \def\fCenter{\ \vdash\ }
\def\ScoreOverhang{.1pt}
\AxiomC{\it`{(Ind.~hyp.)}}
\UnaryInf${\overrightarrow{X}:{\Psi}}\fCenter s\overrightarrow{X}:\overline C$
\Axiom$Y:C\parr B\fCenter Y:C\parr B$
\RightLabel{($\parr{\rm{E}}_l$)}
\BinaryInf$\overrightarrow{X}:{\Psi},Y: C\parr B\fCenter s\overrightarrow{X}\cdot Y:{B}$
\AxiomC{\it{(Ind.~hyp.)}}
\UnaryInf$T: B,\overrightarrow{Z}:{\Theta}\fCenter tT\overrightarrow{Z}:A$
\RightLabel{(Sb)}
\BinaryInf$\overrightarrow{X}:{\Psi},Y:C\parr B,\overrightarrow{Z}:{\Theta}\fCenter ts\overrightarrow{X}\cdot Y\cdot\overrightarrow{Z}:A$
\DisplayProof
\end{array}
$
\caption{From ${\bf ETTC}$ to ${\bf ETTC}_{n.d.\ }$: ``side'' rule cases}
\label{seq->ND, side cases}
\end{figure}
\begin{proof} Let the n.d.\  sequent in question be
$\overrightarrow{X}:\Gamma\vdash t\overrightarrow{X}:A$. Use induction  on the {\bf ETTC} derivation of the translation $\vdash t::A,\overline\Gamma$.
A rule introducing or modifying the ``principal'' formula $A$ directly translates to the corresponding n.d.\  rule (with two variations for the $(\parr)$ rule), the same applies to the $(\equiv_\beta)$ rule.  A ``side'' rule, i.e. modifying  a formula in $\overline\Gamma$, typically, is emulated using the substitution rule and an elimination rule for the dual connective. Basic cases  of the $(\alpha\eta^\leftarrow)$, $(\nabla)$, $(\triangle)$ and $(\otimes)$ rules are collected in Figure~\ref{seq->ND, side cases}, the remaining ones are treated similarly or easier. Note that for the $(\otimes)$ rule there is also a case mirror-symmetric to that in the figure, namely, when the ``principal'' formula $A$ in the conclusion comes
from the right premise of the rule. In this situation the n.d.\  translation is obtained using the $(\parr{\rm{E}}_r)$ rule. \end{proof}
\begin{lem}\label{ND<-sequent calculus}
If a standard sequent  is derivable in  ${\bf ETTC}_{n.d.\ }$ then its {sequent calculus translation}
is derivable ${\bf ETTC}$.
\end{lem}
\begin{proof} For the axioms of ${\bf ETTC}_{n.d.\ }$ use Proposition~\ref{derivability closed under equivalence}, claim~(i). The introduction rules, $(\alpha\eta)$ rules and the $(\equiv_\beta)$ rule of ${\bf ETTC}_{n.d.\ }$ translate to {\bf ETTC} directly.
The elimination rules $(\nabla{\rm{E}})$ and
$(\parr{\rm{E}}_l)$ are emulated in {\bf ETTC} using cuts with {\bf ETTC}
derivable sequents $\vdash \delta^\alpha_\beta::\triangle^\beta_\alpha \overline A,A$ and
$\vdash\overline B\otimes A,\overline A, B$ respectively. The  $(\parr{\rm{E}_r})$ rule is similar to the $(\parr{\rm{E}}_l)$ rule.
The remaining elimination rules are emulated
very easily, using {\bf ETTC} rules for the corresponding dual connectives and the Cut rule. \end{proof}
\begin{cor}\label{ND=sequent calculus}
 A standard sequent  is derivable in  ${\bf ETTC}_{n.d.\ }$ iff its {sequent calculus translation}
is derivable ${\bf ETTC}$. $\Box$
\end{cor}

\subsection{Geometric representation}
Translating to  sequent calculus we can obtain geometric representation for n.d.\  sequents and derivations. Namely, the standard n.d.\  sequent
$\overrightarrow{X}:\Gamma\vdash t\overrightarrow{X}:F$ can be depicted as the ordered tensor sequent $\vdash t::F,\overline{\Gamma}$.

In this case, however, it seems natural to deform the graph of the sequent calculus translation, so that  indices corresponding to
 to the left and to the right of the turnstile in the n.d.\  sequent appear on {\it two} parallel vertical lines. That is, we depict indices of $F$ aligned vertically, from the top to the bottom on one line, and indices of $\overline\Gamma$ aligned {\it from the bottom to the top} on another line to the left. (This is the same as to say that the indices occurring in $F$ are depicted on the right vertical line, aligned from the top to the bottom, and those occurring in  $\Gamma$ depicted on the left vertical line, also ordered from the top to the bottom, but with opposite polarities.) The prescription for drawing labeled edges is the same as in the sequent calculus. An example of such a representation is shown in Figure~\ref{n.d.  sequent geometric representation}.

  In Figure~\ref{n.d.  geometric representation} we show schematically geometric representations of some n.d.\  rules, namely introduction/elimination for the $\parr$ connective and elimination rules for the binding operators.  Note that  $\parr$-introduction  rules do not change the graph, but change its pictorial representation by a continuous deformation.
\begin{figure}[htb]
\centering
%%%%%%%%%%%%%%%%%%%%%%%%%%%%%%%%%%%%%%%%%%%%%%%%%%%%
\subfloat[N.d.\  sequent \label{n.d.  sequent geometric representation}
 ]
{
\tikz%[scale=.5]
         {
                \begin{scope}[shift={(0,0)}]
                %%%%%%%%%%%%%%%%%%%%%%%%%%%%%%%%%%%%%%%%%%%%%%%%%%%%

                     \node[left] at(-.7,-.25){$ A$};
                     \node[left,font=\fontsize{23}{0}\selectfont] at(-.25,-.25){$\{$};
                     \node[left] at(-.7,-1.25){$ B$};
                     \node[left,font=\fontsize{23}{0}\selectfont] at(-.25,-1.25){$\{$};
                     \node[left] at(-.7,-2){$ C$};
                     \node[left,font=\fontsize{18}{0}\selectfont] at(-.25,-2){$\{$};

                     \node [left]at(0,0){$i$};
                     \node [left]at(0,-.5){$j$};
                     \node [left]at(0,-1){$j$};
                     \node [left]at(0,-1.5){$k$};
                     \node [left]at(0,-2){$l$};

                     \draw [fill] (0,0) circle [radius=0.05];
                     \draw [fill] (0,-.5) circle [radius=0.01];

                     \draw [fill] (0,-1) circle [radius=0.05];

                     \draw [fill] (0,-1.5) circle [radius=0.05];
                     \draw [fill] (0,-2) circle [radius=0.01];

                     \draw[thick,<-](0,-.5) to  [out=0,in=90] (.5,-.75) to  [out=-90,in=180](0,-1);

                    \draw [fill] (1.5,-0.5) circle [radius=0.01];
                     \draw [fill] (1.5,-1.) circle [radius=0.05];

                     \draw [fill] (1.5,-1.5) circle [radius=0.01];
                     \node [right]at(1.5,-.5){$s$};
                     \node [right]at(1.5,-1){$t$};
                     \node [right]at(1.5,-1.5){$k$};

                    \node[right,font=\fontsize{18}{0}\selectfont] at(1.75,-.5){$\}$};
                    \node[right] at(2.2,-0.5){$ F$};
                     \node[right,font=\fontsize{23}{0}\selectfont] at(1.75,-1.25){$\}$};
                     \node[right] at(2.2,-1.25){$ G$};

                     \draw[thick,->](0,0)--(0.5,0)--(1,-.5)--(1.5,-.5);
                    \draw[thick,->](0,-1.5)--(0.5,-1.5)--(1,-1.5)--(1.5,-1.5);
                    \draw[thick,<-](0,-2)--(0.5,-2)--(1,-1)--(1.5,-1);
                    \node at (.5,0.2){$\alpha$};
                    \node at (1.,-.8){$\beta$};

                \end{scope}
                \node at(2.9,-1.) {$\Leftrightarrow$};
                \node at(7.5,-1.) {$x^i_j:A^j_i,y^{jk}:B_{jk},z_l:C^l\vdash[\alpha]^s_i[\beta]^l_tx^i_jy^{jk}z_l:F_s\parr G^t_k$};
          }

}\\[.3cm]
%%%%%%%%%%%%%%%%%%%%%%%%%%%%%%%%%%%%%%%%%%%%%%%%%%%%
\subfloat[{Some n.d.\  rules}
 \label{n.d.  geometric representation}
 ]
{
$
\begin{array}{c}
%%%%%%%%%%%%%%%%%%%%%%%%%%%%%%%%%%%%%%%%%%%%%
{
 \hspace*{-0.5cm}\scalebox{.9}%[1]   % Was the vertical stretch by accident? If you insist on it, let us know
 {
  \tikz[scale=.5]
         {
             \begin{scope}[shift={(1,0)}]
                \begin{scope}[shift={(.5,0)}]
                     \draw[thick](0,-2)rectangle(2,-.5);
                     \node at(1,-1.25) {$t$};
                        \draw[thick,-](-1,-.8)--(0,-.8);

                     \node[left] at(-1,-.8) {$A$};%y=-1
                    \draw[thick,-](-1,-1.7)--(0,-1.7);
                      \node[left] at(-1,-1.7) {$\Gamma$};%y=-1.5

                     \draw[thick,-](2,-1.25)--(3,-1.25);
                     \node[right] at(3,-1.25) {$B$};
               % \draw(-.5,0)rectangle(3.5,-2.5);
                \end{scope}
                \node[below] at(6.,-1.5) {$\Longrightarrow$};
                \node[above] at(6.,-1.7) {$(\parr{\rm{I}}_l)$};
             \end{scope}

                \begin{scope}[shift={(10.25,0)}]%x=2.5
                     \draw[thick](0,-2)rectangle(2,-.5);
                     \node at(1,-1.25) {$t$};
                        \draw[thick,-](0,-.8)--(-1,-.8)
                        to[out=180,in=-90]
                    (-1.4,-.3)
                    to[out=90,in=180]
                    (-1,.2)
                    --(3,.2);
                      \node[right] at(3,.2) {$\overline{A}$};%y=-1.5;

                    \draw[thick,-](0,-1.7)--(-1,-1.7);
                     \node[left] at(-1,-1.7) {$\Gamma$};

                     \draw[thick,-](2,-1.25)--(3,-1.25);
                     \node[right] at(3,-1.25) {$B$};

                \end{scope}
         }
  }
 }
  %%%%%%%%%%%%%%%%%%%%%%%%%%%%%%%%%%%%%%%%%%%%%%%%%%
{
 \quad\quad\quad\scalebox{.9}%[1]
 {
    \tikz[scale=.5]
         {
            \begin{scope}[shift={(18.5,1.8)}]%x=2.5
                \begin{scope}[shift={(-1,-.5)}]
                    \draw[thick](0,0)rectangle(2,1);\node at(1,.5){$t$};
                     \draw[thick,-](2,.5)--(3,.5);

                     \node[right] at(3,.5) {$A$};
                     \draw[thick,-](-1,.5)--(0,.5);

                     \node[left] at(-1,.5) {$\Gamma$};
                \end{scope}

                \begin{scope}[shift={(-1,-.25)}]
                        \draw[thick](0,-2)rectangle(2,-.5);\node at(1,-1.25){$s$};
                    \draw[thick,-](2,-.8)--(3,-.8);

                     \node[right] at(3,-.6) {$\overline A$};%y=-1
                    \draw[thick,-](2,-1.7)--(3,-1.7);
                      \node[right] at(3,-1.9) {$B$};%y=-1.5

                     \draw[thick,-](-1,-1.25)--(0,-1.25);
                     \node[left] at(-1,-1.25) {$\Theta$};
                \end{scope}

                \node[below] at(4.5,-1.5) {$\Longrightarrow$};%-.5
                \node[above] at(4.5,-1.7) {$(\parr{\rm{E}}_l)$};
                \begin{scope}[shift={(1,0)}]

                    \draw[thick,-](9,.0)--(10,0)
                         to[out=0,in=90](10.5,-.5)
                         to[out=-90,in=0](10,-1)--(9,-1)
                         ;
                    \begin{scope}[shift={(0,-.5)}]
                                \draw[thick](7,0)rectangle(9,1);\node at(8,.5){$t$};

                         \draw[thick,-](6,.5)--(7,.5);

                         \node[left] at(6,.5) {$\Gamma$};
                    \end{scope}
                    \begin{scope}[shift={(0,-.25)}]
                                 \draw[thick](7,-2)rectangle(9,-.5);\node at(8,-1.25){$s$};
                        %\draw[thick,-](9,-1)--(10 ,-1);

                         %\node[right] at(8,-1) {$a^\bot$};
                        \draw[thick,-](9,-1.5)--(10,-1.5);
                          \node[right] at(10,-1.5) {$B$};

                         \draw[thick,-](6,-1.25)--(7,-1.25);
                         \node[left] at(6,-1.25) {$\Theta$};
                    \end{scope}

            \end{scope}

         \end{scope}
         }
    }
 }
 \\[.3cm]
% %%%%%%%%%%%%%%%%%%%%%%%%%%%%%%%%%%%%%%%%%%%%%
{
 \hspace*{-0.5cm}\scalebox{.9}%[1]
 {
  \tikz[scale=.5]
         {
             \begin{scope}[shift={(1,0)}]
                \begin{scope}[shift={(.5,0)}]
                     \draw[thick](0,-2)rectangle(2,-.5);
                     \node at(1,-1.25) {$t$};
                        \draw[thick,-](-1,-.8)--(0,-.8);

                     \node[left] at(-1,-.8) {$\Gamma$};%y=-1
                    \draw[thick,-](-1,-1.7)--(0,-1.7);
                      \node[left] at(-1,-1.7) {$A$};%y=-1.5

                     \draw[thick,-](2,-1.25)--(3,-1.25);
                     \node[right] at(3,-1.25) {$B$};
                \end{scope}
                \node[below] at(6.,-2.) {$\Longrightarrow$};
                \node[above] at(6.,-2.2) {$(\parr{\rm{I}}_r)$};
             \end{scope}

                \begin{scope}[shift={(10.25,0)}]%x=2.5
                     \draw[thick](0,-2)rectangle(2,-.5);
                     \node at(1,-1.25) {$t$};
                        \draw[thick,-](-1,-.8)--(0,-.8);

                     \node[left] at(-1,-.8) {$\Gamma$};%y=-1
                    \draw[thick,-](0,-1.7)--(-1,-1.7)
                    to[out=180,in=90]
                    (-1.4,-2.2)
                    to[out=-90,in=180]
                    (-1,-2.7)
                    --(3,-2.7);
                      \node[right] at(3,-2.7) {$\overline{A}$};%y=-1.5

                     \draw[thick,-](2,-1.25)--(3,-1.25);
                     \node[right] at(3,-1.25) {$B$};

                \end{scope}
         }
  }
 }
% %%%%%%%%%%%%%%%%%%%%%%%%%%%%%%%%%%%%%%%%%%%%%%%%%
{
 \quad\quad\quad\scalebox{.9}%[1]
 {
    \tikz[scale=.5]
         {
            \begin{scope}[shift={(18.5,1.8)}]%x=2.5
                \begin{scope}[shift={(-1,-1.75)}]
                    \draw[thick](0,0)rectangle(2,1);\node at(1,.5){$t$};
                     \draw[thick,-](2,.5)--(3,.5);

                     \node[right] at(3,.5) {$A$};
                     \draw[thick,-](-1,.5)--(0,.5);

                     \node[left] at(-1,.5) {$\Theta$};
                \end{scope}

                \begin{scope}[shift={(-1,1.5)}]
                        \draw[thick](0,-2)rectangle(2,-.5);\node at(1,-1.25){$s$};
                    \draw[thick,-](2,-.8)--(3,-.8);

                     \node[right] at(3,-1.9) {$\overline{A}$};%y=-1
                    \draw[thick,-](2,-1.7)--(3,-1.7);
                      \node[right] at(3,-.6) {$B$};%y=-1.5

                     \draw[thick,-](-1,-1.25)--(0,-1.25);
                     \node[left] at(-1,-1.25) {$\Gamma$};
                \end{scope}

                \node[below] at(4.5,-.5) {$\Longrightarrow$};
                \node[above] at(4.5,-.8) {$(\parr{\rm{E}}_r)$};
                \begin{scope}[shift={(1,0)}]

                    \draw[thick,-](9,.0)--(10,0)
                         to[out=0,in=90](10.5,-.5)
                         to[out=-90,in=0](10,-1)--(9,-1)
                         ;
                    \begin{scope}[shift={(0,-1.75)}]
                                \draw[thick](7,0)rectangle(9,1);\node at(8,.5){$t$};

                         \draw[thick,-](6,.5)--(7,.5);

                         \node[left] at(6,.5) {$\Theta$};
                    \end{scope}
                    \begin{scope}[shift={(0,1.5)}]
                                 \draw[thick](7,-2)rectangle(9,-.5);\node at(8,-1.25){$s$};
                        \draw[thick,-](9,-1)--(10 ,-1);

                         %\node[right] at(8,-1) {$a^\bot$};
                        \draw[thick,-](9,-1.5)--(10,-1.5);
                          \node[right] at(10,-1.) {$B$};

                         \draw[thick,-](6,-1.25)--(7,-1.25);
                         \node[left] at(6,-1.25) {$\Gamma$};
                    \end{scope}

            \end{scope}

         \end{scope}
         }
    }
 }
 \\[.3cm]
% %%%%%%%%%%%%%%%%%%%%%%%%%%%%%%%%%%%%%%%%%%
{
 %\scalebox{.6}[.8]
 {\quad\quad\quad
  \tikz[scale=.7]
         {
             \begin{scope}[shift={(1,0)}]
                \begin{scope}[shift={(3.25,0)}]
                     \draw[thick](0,-2)rectangle(2,-.5);
                     \node at(1,-1.25) {$t$};
                        \draw[thick,-](-1,-1.25)--(0,-1.25);

                     \node[left] at(-1,-1.25) {$\Gamma$};%y=-1

                     \draw[thick,-](2,-.7)--(3.3,-.7);
                     \draw[thick,-](2,-1.8)--(3.3,-1.8);
                     \node[right,font=\fontsize{19}{0}\selectfont] at(3.1,-1.25){$\}$};
                     \node[right] at(3.7,-1.25){$\nabla^\mu_\nu A$};
                     \begin{scope}[shift={(-.35,0)}]
                        %\draw [fill] (-.1,0) circle [radius=0.02];     % Removed dots that appear to be random
                        \draw [fill] (3.1,-.8) circle [radius=0.02];
                        \draw [fill] (3.1,-.9) circle [radius=0.02];
                        \draw [fill] (3.1,-1.15) circle [radius=0.02];
                        \draw [fill] (3.1,-1.25) circle [radius=0.02];
                        \draw [fill] (3.1,-1.35) circle [radius=0.02];
                        \draw [fill] (3.1,-1.6) circle [radius=0.02];
                        \draw [fill] (3.1,-1.7) circle [radius=0.02];
                    \end{scope}
                \end{scope}
                \node[below] at(10.5,-1.5) {$\Longrightarrow$};
                \node[above] at(10.5,-1.7) {$(\nabla{\rm{E}})$};

                \begin{scope}[shift={(14.75,0)}]
                     \draw[thick](0,-2)rectangle(2,-.5);
                     \node at(1,-1.25) {$t$};
                        \draw[thick,-](-1,-1.25)--(0,-1.25);

                     \node[left] at(-1,-1.25) {$\Gamma$};%y=-1

                     \draw[thick,-](2,-.7)--(3.3,-.7);
                     \draw[thick,-](2,-1.8)--(3.3,-1.8);

                         \node[right,font=\fontsize{19}{0}\selectfont] at(3.1,-1.25){$\}$};
                         \node[right] at(3.7,-1.25){$A$};
                         \begin{scope}[shift={(-.35,0)}]
                         %\draw [fill] (-.1,0) circle [radius=0.015];
                        \draw [fill] (3.1,-.8) circle [radius=0.015];
                        \draw [fill] (3.1,-.9) circle [radius=0.015];
                        \draw [fill] (3.1,-1.15) circle [radius=0.015];
                        \draw [fill] (3.1,-1.25) circle [radius=0.015];
                        \draw [fill] (3.1,-1.35) circle [radius=0.015];
                        \draw [fill] (3.1,-1.6) circle [radius=0.015];
                        \draw [fill] (3.1,-1.7) circle [radius=0.015];

                        \node at(3.4,-1.025){$\mu$};
                        \node at(3.4,-1.475){$\nu$};

                        \draw [fill] (3.1,-1.475) circle [radius=0.06];
                        \draw[thick,<-](3.2,-1.025) to  [out=180,in=90] (2.5,-1.25) to  [out=-90,in=180](3.1,-1.475);
                    \end{scope}
                \end{scope}
             \end{scope}
         }
  }
 }
\\[.3cm]
  %%%%%%%%%%%%%%%%%%%%%%%%%%%%%%%%%%%%%%%%%%%%%%%%%%
{
  %\scalebox{.6}[1]
 {
    \tikz[scale=.7]
         {
    \begin{scope}[shift={(-1.25,0)}]%[shift={(15,1.8)}]%x=2.5
                        \begin{scope}[shift={(-1,-1.5)}]
                                \draw[thick](0,-0.5)rectangle(2.,1.5);\node at(1,.5){$s$};

                                 \draw[thick,-](2,.5)--(3,.5);

                                 \node[right] at(3,.5) {$B$};

                                 \begin{scope}[shift={(-3.3,2.)}]
                                             \draw[thick,-](2,-.7)--(3.3,-.7);
                                             \draw[thick,-](2,-1.8)--(3.3,-1.8);

                                             \begin{scope}[shift={(-2.5,0)}]
                                                \node[right,font=\fontsize{19}{0}\selectfont] at(3.7,-1.25){$\{$};
                                                \node[right] at(3.1,-1.25){$A$};
                                             \end{scope}

                                             \begin{scope}[shift={(-.45,0)}]
                                                 %\draw [fill] (-.1,0) circle [radius=0.015];
                                                \draw [fill] (3.1,-.8) circle [radius=0.015];
                                                \draw [fill] (3.1,-.9) circle [radius=0.015];
                                                \draw [fill] (3.1,-1.15) circle [radius=0.015];
                                                \draw [fill] (3.1,-1.25) circle [radius=0.015];
                                                \draw [fill] (3.1,-1.35) circle [radius=0.015];
                                                \draw [fill] (3.1,-1.6) circle [radius=0.015];
                                                \draw [fill] (3.1,-1.7) circle [radius=0.015];

                                                \node[left] at(3.1,-1.025){$\mu$};
                                                \node [left]at(3.1,-1.475){$\nu$};

                                                \draw [fill] (3.1,-1.475) circle [radius=0.06];
                                                \draw[thick,<-](3.,-1.025) to  [out=0,in=90] (3.7,-1.25) to  [out=-90,in=0](3.1,-1.475);

                                            \end{scope}
                                 \end{scope}

            %%%%%%%%%%%%%%%%%%%%%%%   THETA %%%%%%%%%%%%%%%%%%%%%%%%%%%%%%%%%%%%%%%%%
                                 \draw[thick,-](-1.5,-.3)--(0,-.3);

                                 \node[left] at(-1.7,-.3) {$\Theta$};
            %%%%%%%%%%%%%%%%%%%%%%%%%%%%%%%%%%%%%%%%%%%%%%%%%%%%%%%%%%%%%%%%%%
                        \end{scope}

                        \begin{scope}[shift={(-1,2.25)}]
                                    \draw[thick](0,-2)rectangle(2,-.5);\node at(1,-1.25){$t$};

                                    \begin{scope}[shift={(0,0)}]%1.75
                                        \draw[thick,-](2,-.7)--(3.3,-.7);
                                         \draw[thick,-](2,-1.8)--(3.3,-1.8);
                                         \node[right,font=\fontsize{19}{0}\selectfont] at(3.1,-1.25){$\}$};
                                         \node[right] at(3.7,-1.25){$\triangle^\mu_\nu A$};
                                    \end{scope}
                                    \begin{scope}[shift={(-.35,0)}]
                                         %\draw [fill] (-.1,0) circle [radius=0.02];
                                        \draw [fill] (3.1,-.8) circle [radius=0.02];
                                        \draw [fill] (3.1,-.9) circle [radius=0.02];
                                        \draw [fill] (3.1,-1.15) circle [radius=0.02];
                                        \draw [fill] (3.1,-1.25) circle [radius=0.02];
                                        \draw [fill] (3.1,-1.35) circle [radius=0.02];
                                        \draw [fill] (3.1,-1.6) circle [radius=0.02];
                                        \draw [fill] (3.1,-1.7) circle [radius=0.02];
                                    \end{scope}

                                 \draw[thick,-](-1.7,-1.25)--(0,-1.25);
                                 \node[left] at(-1.7,-1.25) {$\Gamma$};
                        \end{scope}
            \end{scope}

    %%%                        ARROW!!!!!!!!!!!!!!!!!!!!!!
                \node[below] at(4.5,.3) {$\Longrightarrow$};
                \node[above] at(4.5,0) {$(\triangle{\rm{E}})$};
    %%%%%%%%%%%%%%%%%%%%%%%  AFTER ARROW %%%%%%%%%%%%%%%%%%%%%%%%%%%%%%%%%%%%%%%%%%%%%%%%%%%%%%%%%%%

    \begin{scope}[shift={(10,0)}]%[shift={(15,1.8)}]%x=2.5
                        \begin{scope}[shift={(-1,-2.5)}]
                                \draw[thick](0,-0.5)rectangle(2.,1.5);\node at(1,.5){$s$};

                                 \draw[thick,-](2,.5)--(3,.5);

                                 \node[right] at(3,.5) {$B$};

                                 \begin{scope}[shift={(-3.3,2.)}]
                                             \draw[thick,-](2,-.7)--(3.3,-.7);
                                             \draw[thick,-](2,-1.8)--(3.3,-1.8);

                                             \begin{scope}[shift={(-.45,0)}]
                                                 %\draw [fill] (-.1,0) circle [radius=0.015];
                                                \draw [fill] (3.1,-.8) circle [radius=0.015];
                                                \draw [fill] (3.1,-.9) circle [radius=0.015];
                                                \draw [fill] (3.1,-1.15) circle [radius=0.015];
                                                \draw [fill] (3.1,-1.25) circle [radius=0.015];
                                                \draw [fill] (3.1,-1.35) circle [radius=0.015];
                                                \draw [fill] (3.1,-1.6) circle [radius=0.015];
                                                \draw [fill] (3.1,-1.7) circle [radius=0.015];

                                            \end{scope}
                                 \end{scope}

            %%%%%%%%%%%%%%%%%%%%%%%   THETA %%%%%%%%%%%%%%%%%%%%%%%%%%%%%%%%%%%%%%%%%
                                 \draw[thick,-](-1.7,-.3)--(0,-.3);

                                 \node[left] at(-1.7,-.3) {$\Theta$};
            %%%%%%%%%%%%%%%%%%%%%%%%%%%%%%%%%%%%%%%%%%%%%%%%%%%%%%%%%%%%%%%%%%
                        \end{scope}

                        \begin{scope}[shift={(-1,2.25)}]
                                    \draw[thick](0,-2)rectangle(2,-.5);\node at(1,-1.25){$t$};

                                    \begin{scope}[shift={(0,0)}]%1.75
                                        \draw[thick,-](2,-.7)--(3.3,-.7)
                                        to  [out=0,in=90] (3.6,-1)--(3.6,-2.6)
                                         to  [out=-90,in=0](3.3,-2.8)--(-1.4,-2.8)
                                             to  [out=180,in=90](-1.7,-3.1)%--(-2,-3)
                                             to  [out=-90,in=180](-1.4, -3.45)--(0,-3.45)
                                             ;
                                         \draw[thick,-](2,-1.8)--(2.5,-1.8)
                                         to  [out=0,in=90] (2.8,-2.1)
                                             to  [out=-90,in=0](2.5,-2.4)--(-1.7,-2.4)
                                             to  [out=180,in=90](-2,-2.6)--(-2,-4.2)
                                             to  [out=-90,in=180](-1.7, -4.55)--(0,-4.55)
                                             ;

                                    \end{scope}
                                    \begin{scope}[shift={(-.35,0)}]
                                         %\draw [fill] (-.1,0) circle [radius=0.02];
                                        \draw [fill] (3.1,-.8) circle [radius=0.02];
                                        \draw [fill] (3.1,-.9) circle [radius=0.02];
                                        \draw [fill] (3.1,-1.15) circle [radius=0.02];
                                        \draw [fill] (3.1,-1.25) circle [radius=0.02];
                                        \draw [fill] (3.1,-1.35) circle [radius=0.02];
                                        \draw [fill] (3.1,-1.6) circle [radius=0.02];
                                        \draw [fill] (3.1,-1.7) circle [radius=0.02];
                                    \end{scope}

                                 \draw[thick,-](-1.7,-1.25)--(0,-1.25);
                                 \node[left] at(-1.7,-1.25) {$\Gamma$};
                        \end{scope}
            \end{scope}

         }
    }
 }
 \end{array}
 $
 }
 \caption{Geometric representation of n.d.\  sequents}
 \end{figure}

\subsection{Example: slash elimination rules of {\bf LC}}
\begin{figure}[ht]
\centering
\subfloat[
{Derivation}
\label{ND derivation}]
{
\begin{tikzpicture}
\node at (0,0){
\def\fCenter{\ \vdash\ }
\AxiomC{$y^{j}_m:\nabla^k_i(b^m_k\parr\overline{a}^i_j) \vdash y^{j}_m:\nabla^k_i(b^m_k\parr\overline{a}^i_j)$}
\RightLabel{$(\nabla{\rm {E}})$}
\UnaryInf$y^{j}_m:\nabla^k_i(b^m_k\parr\overline{a}^i_j)  \fCenter \delta^k_i y^{j}_m:b^m_k\parr\overline{a}^i_j $
\RightLabel{$(=)$}
\UnaryInf$ y^{j}_m:(b/a)^m_j \fCenter \delta^k_i y^{j}_m:b^m_k\parr\overline{a}^i_j$
\RightLabel{$(\alpha\eta_\leftarrow)$}
\UnaryInf$ y^{j}_m:(b/a)^m_j \fCenter  y^{j}_m:b^m_i\parr\overline{a}^i_j$
\AxiomC{$x^{i}_j:a^j_i \vdash x^{i}_j:a^j_i$}
\RightLabel{$(\parr{\rm {E}}_r)$}
\BinaryInf$ y^{j}_m:(b/a)^m_j,x^{i}_j:a^j_i \fCenter y^{j}_mx^{i}_j:b^m_i$
\DisplayProof
};
\end{tikzpicture}
}
\\[.3cm]
%%%%%%%%%%%%%%%%%%%%%%%%%%%
\subfloat[{Geometric representation}
\label{ND picture}
]
{
\tikz[scale=.5]
{
\begin{scope}[shift={(0,0)}]
    \begin{scope}[shift={(-2.5,0)}]

        \begin{scope}[shift={(0,-4)}]
             \draw[thick,<-](1,2)--(2,2);
             \draw [fill] (2,2) circle [radius=0.05];
             \draw[thick,->](1,1.3)--(2,1.3);
             \draw [fill] (1,1.3) circle [radius=0.05];
            \node[left]at (1,2){$j$};
            \node[right]at (2,2){$j$};
            \node[left]at (1,1.3){$i$};
            \node[right]at (2,1.3){$i$};

            \node[left] at (0,1.7){$a^j_i$};
            \node[font=\fontsize{25}{0}\selectfont] at (.2,1.7){$\{$};

            \node [right]at (3.0,1.7){$a^j_i$};
            \node[font=\fontsize{25}{0}\selectfont] at (2.8,1.7){$\}$};
        \end{scope}

        \begin{scope}%[shift={(0,-1)}]
                \draw[thick,<-](1,.4)--(2,.4);
              \draw[thick,->](1,-.3)--(2,-.3);
              \draw [fill] (2,.4) circle [radius=0.05];
              \draw [fill] (1,-.3) circle [radius=0.05];
              \node[left]at (1,.4){$m$};
            \node[right]at (2,.4){$m$};
            \node[left]at (1,-.3){$j$};
            \node[right]at (2,-.3){$j$};
            \node[left] at (-0.2,.05){$\nabla^k_i(b^m_k\parr\overline{a}_i^j)$};
            \node[font=\fontsize{25}{0}\selectfont] at (-.1,.05){$\{$};
            \node[font=\fontsize{25}{0}\selectfont] at (3.1,.05){$\}$};

            \node [right]at (3.2,.05){$\nabla^k_i(b^m_k\parr\overline{a}_i^j)$};
        \end{scope}

         \draw(-4.7,1.5)rectangle(7.7,-3.4);

    \end{scope}
    \node [below]at(6.25,-1.05) {$\Longrightarrow$};
    \node [above]at(6.25,-1.35) {$(\nabla{\rm {E}})$};
\end{scope}

\begin{scope}[shift={(14,0)}]%x=2.5
    \begin{scope}[shift={(-1.5,0)}]

        \begin{scope}[shift={(0,-4)}]
             \draw[thick,<-](1,2)--(2,2);
             \draw [fill] (2,2) circle [radius=0.05];
             \draw[thick,->](1,1.3)--(2,1.3);
             \draw [fill] (1,1.3) circle [radius=0.05];
            \node[left]at (1,2){$j$};
            \node[right]at (2,2){$j$};
            \node[left]at (1,1.3){$i$};
            \node[right]at (2,1.3){$i$};

            \node[font=\fontsize{20}{0}\selectfont] at (.2,1.7){$\{$};
            \node[left] at (0,1.7){$a^j_i$};

            \node[font=\fontsize{20}{0}\selectfont] at (2.8,1.7){$\}$};
            \node[right] at (3,1.7){$a^j_i$};
        \end{scope}

        \begin{scope}
                \draw[thick,<-](1,.4)
                to[out=0,in=-90](1.5,.7)
                to[out=90,in=0](2,1);
                \draw [fill] (2,1) circle [radius=0.05];

              \draw[thick,->](1,-.3)
              to[out=0,in=90](1.5,-.6)
                to[out=-90,in=180](2.15,-.9);
                \draw [fill] (1,-.3) circle [radius=0.05];
              \node[left]at (1,.4){$m$};
            \node[right]at (2,1){$m$};

            \node[right]at (2,.4){$k$};

            \draw[thick,->](2.15,-.3)
              to[out=180,in=-90](1.5,.05)
                to[out=90,in=180](2,.4);
                \draw [fill] (2.15,-.3) circle [radius=0.05];

            \node[right]at (2.15,-.3){$i$};
            \node[left]at (1,-.3){$j$};
            \node[right]at (2.15,-.9){$j$};

            \node[left] at (-0.2,.05){$\nabla^k_i(b^m_k\parr\overline{a}^i_j)$};
            \node[font=\fontsize{20}{0}\selectfont] at (-.1,.05){$\{$};

            \node[font=\fontsize{20}{0}\selectfont] at (3.1,.7){$\}$};
            \node[font=\fontsize{20}{0}\selectfont] at (3.1,.-.8){$\}$};

            \node[right] at (3.2,.7){$b^m_k$};
            \node[right] at (3.2,-.8){$\overline{a}^i_j$};
            \draw(-4.7,1.5)rectangle(4.45,-3.4);
        \end{scope}

    \end{scope}

\end{scope}

\begin{scope}[shift={(-2.5,-5.35)}]%x=2.5
    \begin{scope}%[shift={(-1.5,0)}]
        \node at(-2.2,-.9) {$=$};
            \begin{scope}[shift={(0,-4)}]
                         \draw[thick,<-](1,2)--(2,2);
                         \draw[thick,->](1,1.3)--(2,1.3);
                         \draw [fill] (2,2) circle [radius=0.05];
                         \draw [fill] (1,1.3) circle [radius=0.05];

                        \node[left] at (1,1.7){$a$};
                        %\draw[dashed,](2.5,2.2)--(2.8,2.2)--(2.8,1.2)--(2.5,1.2);
                        \node[right] at (2,1.7){$a$};
            \end{scope}

            \begin{scope}%[shift={(0,-1)}]
                    \draw[thick,<-](1,.4)
                    to[out=0,in=-90](1.5,.7)
                    to[out=90,in=0](2,1);
                    \draw [fill] (2,1) circle [radius=0.05];

                  \draw[thick,->](1,-.3)
                  to[out=0,in=90](1.5,-.6)
                    to[out=-90,in=180](2,-.9);
                    \draw [fill] (1,-.3) circle [radius=0.05];

                \draw[thick,->](2,-.3)
                  to[out=180,in=-90](1.5,.05)
                    to[out=90,in=180](2,.4);

                \node[left] at (1.,.05){$b/a$};
                \node[right] at (2,.7){$b$};
                \node[right] at (2,-.6){$\overline{a}$};
                \draw(-.6,1.5)rectangle(3.05,-3.4);
            \end{scope}

    \end{scope}

\end{scope}

\begin{scope}[shift={(-2.5,0)}]

     \begin{scope}[shift={(6.9,-5.35)}]%x=2.5
         \node [below]at(-2.1,-.9) {$\Longrightarrow$};
         \node [above]at(-2.1,-1.2) {$(\parr{\rm {E}}_r)$};
         \begin{scope}%[shift={(2,0)}]

                 \begin{scope}[shift={(0,-4)}]
                     \draw[thick,<-](1,2)--(2,2);
                     \draw[thick,-](1,1.3)--(2,1.3);
                     \draw [fill] (1,1.3) circle [radius=0.05];
                    \node[left] at (1,1.7){$a$};

                \end{scope}

                \begin{scope}%[shift={(0,-1)}]
                        \draw[thick,<-](1,.4)
                        to[out=0,in=-90](1.5,.7)
                        to[out=90,in=0](2,1);
                        \draw [fill] (2,1) circle [radius=0.05];

                      \draw[thick,-](1,-.3)
                      to[out=0,in=90](1.5,-.6)
                        to[out=-90,in=180](2,-.9)
                           to[out=0,in=90](2.3,-1.45)
                           to[out=-90,in=0](2,-2);
                           \draw [fill] (1,-.3) circle [radius=0.05];

                    \draw[thick,->]
                    (2,-2.7)
                    to[out=0,in=-90](2.7,-1.45)
                    to[out=90,in=0]
                    (2,-.3)
                      to[out=180,in=-90](1.5,.05)
                        to[out=90,in=180](2,.4);

                    \node[left] at (1,.05){$b/a$};
                    \node[right] at (2,.7){$b$};

                    \draw(-.6,1.5)rectangle(3.05,-3.4);
                \end{scope}
            \end{scope}

            %HERE
            \begin{scope}[shift={(8,0)}]%x=2.5
                \node at(-2.7,-.9) {$=$};
                \begin{scope}%[shift={(2,0)}]

                    \begin{scope}[shift={(0,-4)}]
                         \draw[thick,<-](1,2)--(2,2);
                         \draw[thick,-](1,1.3)--(2,1.3);
                         \draw [fill] (1,1.3) circle [radius=0.05];

                        \node[left] at (1,1.7){$a$};
                    \end{scope}

                    \begin{scope}%[shift={(0,-1)}]
                            \draw[thick,<-](1,.4)
                                                 to[out=0,in=90](2.5,-.25)
                            to[out=-90,in=180](3,-.9);
                            \draw [fill] (3,.-.9) circle [radius=0.05];

                          \draw[thick,-](1,-.3)
                               to[out=0,in=90](2.25,-1.15)
                               to[out=-90,in=0](2,-2);
                          \draw [fill] (1,-.3) circle [radius=0.05];

                        \draw[thick,->]
                        (2,-2.7)
                         to[out=0,in=-90](2.5,-1.55)
                            to[out=90,in=180](3,-1.4);

                        \node[left] at (1,.05){$b/a$};

                        \node[right] at (3,-1.15){$b$};
                        \draw(-.6,1.5)rectangle(3.9,-3.4);

                    \end{scope}
                \end{scope}
            \end{scope}
        \end{scope}
     \end{scope}

}
}
\caption{Slash elimination}
\label{ND example}
\end{figure}
%%%%%%%%%%%%%%%%%%%%%%%%%%%NOTATION FOR LAMBEK
Recall that we introduced a translation of {\bf LC} directional implications into the language of {\bf ETTC}, see (\ref{translating lambek types}).
%%%%%%%%%%%%%%%%%%%%%%%%%%%%%%%%%%%%%%%%%%%%%%%%%%%%%%%%%%%%%%%%%%%%%%%%%%%%%%%%%%%%%
\begin{prop}
In ${\bf ETTC}_{n.d.\ }$ there are derivable sequents
\[
y^j_m:(b/a)^m_j,x^i_j:a^j_i\vdash y^j_mx^i_j:b^m_i,\quad x^i_j:a^j_i,y^m_i:(a\setminus b)^m_j\vdash x^i_jy^m_i:b^m_j.
\]
\end{prop}
\begin{proof} Derivation of the first sequent together with is geometric representation is shown in Figure~\ref{ND derivation}. The case of the second one is similar. \end{proof}
\begin{cor}
The rules
\[
\cfrac
{
    \Gamma\vdash t:(b/a)^i_j\quad\Theta\vdash s:a^j_k
}
{
\Gamma,\Theta\vdash t s:b^i_k
}
(/ {\rm {E}})
\quad
\cfrac
{
\Gamma\vdash s:a^i_j\quad\Theta\vdash t:(a\backslash b)^j_k
}
{
\Gamma,\Theta\vdash s t:b^i_k
}
(\backslash{\rm{E}})
\]
 are admissible in ${\bf ETTC}_{n.d.\ }$. $\Box$
\end{cor}

\subsection{Grammars}
 \begin{figure}[htb]
\centering
%%%%%%%%%%%%%%%%%%%%%%%%%%%%%%%%%
\subfloat[{Axioms}
\label{Axioms}
]
{
$\begin{array}{c}
[{\rm{Mary}}]^i_j: np,\quad
[{\rm{John}}]^i_j: np,\quad
[{\rm{loves}}]^i_j:(np\backslash s)/np,
\\
\mbox{[{\rm{madly}}]}^i_j:
((np\backslash s)
\backslash(np\backslash s)),
\quad
\mbox{[{\rm{who}}]}^i_j:
(np\backslash np)/\Delta^u_t(s\parr\overline{np}^t_u)
\end{array}$
}\\
%%%%%%%%%%%%%%%%%%%%%%%%%%%
\subfloat[{Derivation}
\label{ND grammar derivation}
]
{
%$\mbox{ }$\\
 %  \begin{prooftree}
 \begin{tabular}{c}
    \def\fCenter{\mbox{\ $\vdash$\ }}
    \def\ScoreOverhang{.1pt}
    \Axiom$
            \fCenter [{\rm{loves}}]^k_l:(np\backslash s)/np
                $
     \Axiom$x^j_k:np\fCenter x^j_k:np$
    \RightLabel{$(/{\rm {E}})$}
    \insertBetweenHyps{\hskip 5pt}
    \BinaryInf$
        x^j_k:np \fCenter
          [{\rm{loves}}]^k_l x^j_k:np\backslash s
        $
     \Axiom$
            \fCenter [{\rm{madly}}]^i_j:
            (np\backslash s)\backslash(np\backslash s)
            $
    \RightLabel{$(\backslash{\rm {E}})$
    \insertBetweenHyps{\hskip 5pt}}
    \BinaryInf$
        x^j_k:np \fCenter
        [{\rm{loves}}]^k_l x^j_k[{\rm{madly}}]^i_j:
        np\backslash s
        $
    \DisplayProof
\\
\\
    \def\fCenter{\mbox{\ $\vdash$\ }}
    \def\ScoreOverhang{.1pt}
    \Axiom$
            \fCenter [{\rm{John}}]^l_m:
            np
            $
    \Axiom$
            x^j_k:np \fCenter
        [{\rm{loves}}]^k_l x^j_k[{\rm{madly}}]^i_j:
        np\backslash s
                $
     \RightLabel{$(\backslash{\rm {E}})$}
    \insertBetweenHyps{\hskip 5pt}
    \BinaryInf$
        x^j_k:np \fCenter
          [{\rm{John}}]^l_m[{\rm{loves}}]^k_l x^j_k[{\rm{madly}}]^i_j: s_i^m
        $
        \RightLabel{$(\equiv_\beta)$}
    \UnaryInf$
                    x^j_k:np \fCenter
          [{\rm{John}}\mbox{ }{\rm{loves}}]^k_m x^j_k[{\rm{madly}}]^i_j: s_i^m
        $
    \RightLabel{$(\parr{\rm {I}_r})$}
    \UnaryInf$
        \fCenter
            [{\rm{John}}\mbox{ }{\rm{loves}}]^k_m [{\rm{madly}}]^i_j :
            s^m_i\parr\overline{np}^j_k
        $
    \RightLabel{$(\triangle{\rm {I}})$}
    \UnaryInf$
    \fCenter
          \delta_k^j\cdot[{\rm{John}}\mbox{ }{\rm{loves}}]^k_m \cdot[{\rm{madly}}]^i_j:
        \triangle^k_j(
        s^m_i\parr\overline{np}^j_k
        )
        $
        \RightLabel{$(\equiv_\beta)$}
    \UnaryInf$
    \fCenter
          [{\rm{John}}\mbox{ }{\rm{loves}}\mbox{ }{\rm{madly}}]^i_m:
        \triangle^k_j(s^m_i\parr\overline{np}^j_k)
        $
  \DisplayProof
\\
\\
    \def\fCenter{\mbox{\ $\vdash$\ }}
    \def\ScoreOverhang{.1pt}
    \Axiom$
            \fCenter [{\rm{who}}]_u^m:
    (np\backslash np)/\triangle_j^k(s^m_u\parr\overline{np}^j_k)
            $
    \Axiom$
             \fCenter
        [{\rm{John}}~{\rm{loves}}~{\rm{madly}}]^i_m:
        \triangle^k_j(s^m_i\parr\overline{np}^j_k)
                $
    \insertBetweenHyps{\hskip 0pt}
    \RightLabel{$(/{\rm {E}})$}
    \BinaryInf$
         \fCenter
          [{\rm{who}}]^m_u[{\rm{John}}\mbox{ }{\rm{loves}}\mbox{ }{\rm{madly}}]_m^i:
        np\backslash np
        $
    \RightLabel{$(\equiv_\beta)$}
    \UnaryInf$
         \fCenter
          [{\rm{who}}\mbox{ }{\rm{John}}\mbox{ }{\rm{loves}}\mbox{ }{\rm{madly}}]_u^i:
        np\backslash np
        $
  \DisplayProof
\\
\\
    \def\fCenter{\mbox{\ $\vdash$\ }}
    \def\ScoreOverhang{.1pt}
    \Axiom$
            \fCenter [{\rm{Mary}}]^u_n:
            np
            $
    \Axiom$
         \fCenter
          [{\rm{who}}\mbox{ }{\rm{John}}~{\rm{loves}}~{\rm{madly}}]^i_u:
        np\backslash np
        $
    \RightLabel{$(\backslash{\rm {E}},\equiv_\beta)$}
    \insertBetweenHyps{\hskip 5pt}
    \BinaryInf$
        \fCenter
          [{\rm{Mary}}\mbox{ }{\rm{who}}\mbox{ }{\rm{John}}\mbox{ }{\rm{loves}}\mbox{ }{\rm{madly}}]^i_n:
         np
        $
    %\end{prooftree}
     \DisplayProof
\end{tabular}
}
\caption{Tensor grammar example}
\label{ND grammar example}
\end{figure}

\begin{defi} Given a terminal alphabet $T$, a {\it tensor lexical
 entry} (over $T$), is an $\alpha$-equivalence class of $\beta$-reduced lexical typing judgements (over $T$). A {\it tensor grammar} $G$ (over $T$) is a pair $G=(\mathit{Lex},s)$, where $\mathit{Lex}$, the {\it lexicon}, is a finite set of lexical entries (over $T$) and $s$ is an atomic type symbol of valency $(1,1)$.
  \end{defi}
  For a tensor grammar $G$ we  write $\vdash_G \sigma:A$ to indicate that the sequent $\vdash\sigma:A$ is derivable in ${\bf ETTC}_{n.d.\ }$ from elements of $\mathit{Lex}$.
 \begin{defi}
 The {\it language $L(G)$} of a tensor grammar $G$ is the set
 \be\label{tensor grammar language}
L(G)=\{w|~\vdash_G [w]^i_j:s^j_i\}.
\ee
\end{defi}
Since ${\bf ETTC}_{n.d.\ }$ derivability is closed under $\alpha$-equivalence, nothing will change if we require  lexicons to be finite only modulo $\alpha$-equivalence. In fact, it is more convenient to consider lexicons closed under $\alpha$-equivalence and consisting only of finitely many equivalence classes. Given a tensor grammar $G=(\mathit{Lex},s)$, we write $\mathit{Lex}_{\equiv_\alpha}$ for the closure of $\mathit{Lex}$ under $\alpha$-equivalence.
\begin{prop}\label{deduction theorem for grammars}
Given a tensor grammar $G=(\mathit{Lex},s)$ and a typing judgement $\Sigma$ of the form $t:F$, we have $\vdash_G\Sigma$ iff
there are lexical entries $\tau_{(1)}:A_{(1)},\ldots,\tau_{(n)}:A_{(n)}\in \mathit{Lex}_{\equiv_\alpha}$ satisfying
\be\label{orthogonality of axioms}
FI(A_{(i)})\perp FI(A_{(j)})\mbox{ for }i\not=j
\ee
and a term $t'$, which is the product of Kronecker deltas, such that
 \[
 \vdash_{\bf ETTC}t'::F,\overline{A_{(n)}},\ldots,\overline{A_{(1)}},\quad t'\cdot\tau_{(1)}\ldots\tau_{(n)}\equiv_\beta t.
 \]
\end{prop}
\begin{proof} Assume that $\vdash_G\Sigma$.

Since $\alpha$-equivalent n.d.\  sequents in ${\bf ETTC}_{n.d.\ }$ are derivable from each other, we may assume that the entries from $\mathit{Lex}_{\equiv_\alpha}$ used for deriving $\vdash\Sigma$ satisfy (\ref{orthogonality of axioms}). Then the statement is a
direct corollary of Proposition~\ref{extended deduction theorem} (``Deduction theorem'') and Corollary~\ref{ND=sequent calculus} on the relation between derivability in ${\bf ETTC}_{n.d.\ }$ and ${\bf ETTC}$. That $t'$ is the product of Kronecker deltas follows from  a simple observation that {\bf ETTC} is independent from terminal alphabets, hence $\vdash_{\bf ETTC}t'::\Gamma$ implies that $t'$ has no terminal symbols.

The other direction follows from Corollary~\ref{ND=sequent calculus} and Proposition~\ref{extended deduction theorem} in a similar way. \end{proof}

A linguistically motivated example of tensor grammar is given in Figure~\ref{ND grammar example}, where we use notation for {\bf LC} slashes as in (\ref{translating lambek types}). For better readability, we omit free indices in formulas, whenever  they are uniquely determined by free indices in terms.   Figure~\ref{Axioms} shows the lexicon, which we assume closed under $\alpha$-equivalence of typing judgements, and in Figure~\ref{ND grammar derivation} we derive the noun phrase ``Mary who John loves madly''.

 The grammar, obviously, is an extension of a one adapted from {\bf LC}. However, the lexical entry for ``who'', which allows medial extraction, is not a translation from {\bf LC}, as is manifested by the binding operator in the type.
Using the implicational notation, the above lexical entry can be written as
\[
\mbox{[{\rm{who}}]}^i_j:
(np\backslash np)/\triangle^u_t((np^u_t)\multimap s)).
\]
 It might be entertaining to reproduce the derivation in the geometric language.

\section{Correspondence with first order logic}
\subsection{Translating formulas}
\begin{figure}[htb]
\centering
%%%%%%%%%%%%%%%%%%%%%%%%%%%%%%%%%
\subfloat[{$\eta$-Long translation}
\label{eta-long translation}
]
{
$
\begin{array}{c}
||p(x_1,\ldots,x_n)||_{\eta\to}=p^{x^l_1\ldots x^l_k}_{x^r_{k+1}\ldots x^r_n},\mbox{ for }p\in\mathit{Lit}, v(p)=(k,n-k),
\\[.15cm]
||A\otimes B||_{\eta\to}=||A||_{\eta\to}\otimes||B||_{\eta\to},\quad||A\parr B||_{\eta\to}=||A||_{\eta\to}\parr||B||_{\eta\to},
\\[.15cm]
||\forall xA||_{\eta\to}=\nabla^{x^r}_{x^l}||A||_{\eta\to},\quad||\exists xA||_{\eta\to}=\triangle^{x^r}_{x^l}||A||_{\eta\to},
\\[.15cm]
||A_1,\ldots,A_n||_{\eta\to}=||A_1||_{\eta\to},\ldots,||A_n||_{\eta\to},
\\[.15cm]
||\vdash\Gamma||_{\eta\to}=(\vdash\pi(\Gamma)::||\Gamma||_{\eta\to}),\mbox{ where }\pi(\Gamma)=\prod\limits_{x\in FV(\Gamma)}\delta^{x^r}_{x_l}.
\end{array}
$
}\\[.25cm]
%%%%%%%%%%%%%%%%%%%%%%%%%%%%%%%%%
\subfloat[{$\eta$-Reduced translation}
\label{eta-reduced translation}
]
{
$
\begin{array}{c}
||p(x_1,\ldots,x_n)||=p^{x_1\ldots x_k}_{x_{k+1}\ldots x_n},\mbox{ for }p\in\mathit{Lit}, v(p)=(k,n-k),
\\[.15cm]
||A\otimes B||=||A||\otimes||B||,\quad||A\parr B||=||A||\parr||B||,
\\[.15cm]
||\forall xA||=\nabla^{u}_{v}||A||^{[\nicefrac{v/x}]}_{[\nicefrac{u/x}]},\quad||\exists xA||=\triangle^{u}_{v}||A||^{[\nicefrac{v/x}]}_{[\nicefrac{u/x}]},
\\[.15cm]
||A_1,\ldots,A_n||=||A_1||,\ldots,||A_n||,\quad||\vdash\Gamma||=(\vdash||\Gamma||).
\end{array}
$
}\caption{Translations of linguistic fragment to {\bf ETTC}}
\end{figure}
Given a linguistically marked first order language, we identify predicate symbols with literals   of the same valencies and define the {\it $\eta$-reduced} and {\it $\eta$-long} translations, denoted respectively as $||.||$ and $||.||_{\eta\to}$, of linguistically well-formed formulas and contexts to  tensor formulas and contexts. For the {$\eta$-reduced} translation we identify first order variables with indices, left occurrences with upper ones and right occurrences with lower ones. For the {$\eta$-long} translation, on the other hand, we identify free variable  {\it occurrences} with indices, so that for each  $x\in \mathit{Var}$ there are two indices $x^l,x^r\in\mathit{Ind}$, corresponding to its left and right occurrence respectively.
 \begin{defi} The {\it $\eta$-long translation} of  linguistically well-formed {\bf MLL1} formulas, contexts and sequents  to {\bf ETTC} is given in Figure~\ref{eta-long translation}. The {\it $\eta$-reduced translation} of linguistically well-formed {\bf MLL1} formulas, contexts and sequents to {\bf ETTC} is defined up to $\alpha$-equivalence of tensor formulas and is given in Figure~\ref{eta-reduced translation}, where it is assumed that the indices $u$, $v$ in the  translation of quantified formulas are such that they result in  well-defined expressions.
 \end{defi}
  The need to choose bound indices $u$, $v$ in the $\eta$-reduced translation of quantifiers makes the latter translation ambiguous. However,  different choices of indices result in $\alpha$-equivalent tensor formulas, which are provably equivalent in {\bf ETTC}.
  \begin{prop}\label{equivalence of translations}
  For any linguistically well-formed {\bf MLL1} sequent $\vdash\Gamma$, its ${\eta}$-reduced and $\eta$-long translations are $\alpha\eta$-equivalent, $||\vdash\Gamma||_{\eta\to}\equiv_{\alpha\eta}||\vdash\Gamma||$. $\Box$
  \end{prop}
  In terms of the geometric representation, the $\eta$-long translation $\vdash\pi(\Gamma):: ||\Gamma||_{\eta\to}$  of a linguistically well-formed {\bf MLL1} sequent $\vdash\Gamma$ is the occurrence net of $\vdash\Gamma$, encoded as a $\beta$-reduced $\eta$-long typing judgement (thinking both of the occurrence net and of the term $\pi(\Gamma)$ as bipartite graphs). The ${\bf ETTC}_{\eta\to}$ rules correspond to transformations of occurrence nets under linguistic derivation rules. The $(\forall)$ rule, which erases a link from the occurrence net, naturally corresponds to the $(\nabla)$ rule, and  the $(\exists')$ rule, which glues links together, to the $(\triangle)$ rule. The $(\exists)$ rule applied to a linguistically well-formed sequent erases a link in the same way as the $(\forall)$ rule. This corresponds to the situation when the $(\triangle)$ rule is applied to a sequent of the form $\vdash\delta^i_jt::\Theta$ and binds the pair of occurrences $(j,i)\in FI(\Theta)$,  introducing  the term $\delta^j_i\cdot\delta^i_jt\equiv_\beta t$, basically, erasing $\delta^i_j$ from the premise. After these observations we state the main lemma.
 \begin{lem}\label{FO2ETTC}
A linguistically well-formed sequent is derivable in ${\bf MLL1}$ iff its    $\eta$-long and $\eta$-reduced translations  are derivable in {\bf ETTC}.
Conversely, any {\bf ETTC} derivable sequent is $\alpha\beta\eta$-equivalent to a translation of an {\bf MLL1} derivable linguistically well-formed sequent.
\end{lem}
\begin{proof}
Given an {\bf MLL1} derivable linguistically well-formed sequent $\vdash\Gamma$, we prove by induction on linguistic derivation of $\vdash\Gamma$ that the tensor sequent $||\vdash\Gamma||_{\eta\to}$  is derivable in ${\bf ETTC}_{\eta\to}$, hence in {\bf ETTC}. Then Proposition~\ref{equivalence of translations} implies that $||\vdash\Gamma||$ is {\bf ETTC} derivable as well since $\alpha\eta$-equivalent sequents are {\bf ETTC} derivable from each other.

Given an {\bf ETTC} derivable tensor sequent $\vdash\Sigma$, we consider its $\beta$-reduced $\eta$-long form $\vdash\Sigma'$, which is derivable in ${\bf ETTC}_{\eta\to}$ by Lemma~\ref{conservativity of eta-long fragment}. By induction on ${\bf ETTC}_{\eta\to}$ derivation of $\vdash\Sigma'$ we prove that $\vdash\Sigma'$ is the translation of an ${\bf MLL1}$ derivable linguistically well-formed sequent.
 \end{proof}

  Restricting to  the intuitionistic fragments we obtain an analogous correspondence.
 Note that the $\eta$-reduced fragment of {\bf ETTC} and the linguistically well-formed fragment of {\bf MLL1} become literally the same thing, only written in different  notation (up to  bound indices/variables renaming).

\pagebreak
\subsection{Translating grammars}
\begin{defi} Given a
well-formed {\bf MILL1} grammar $G=(\mathit{Lex},s)$, its {\it tensor translation} is the tensor grammar $\widetilde G=(\widetilde{\mathit{Lex}},s)$,
where $\widetilde{\mathit{Lex}}=\{[w]_{r}^{l}:||A||~|(w,A)\in\mathit{Lex}\}$.
\end{defi}
\begin{lem}
For any well-formed {\bf MILL1} grammar $G$, its {tensor translation}  $\widetilde G$
 generates the same language: $L(G)=L(\widetilde G)$.
 \end{lem}
 \begin{proof}
Let $\mathit{Lex}$ be the lexicon of $G$.

Let $w\in L(\widetilde G)$, so that $\vdash_G [w]^r_l:s^l_r$. Obviously $w$ is the concatenation
\be\label{concatenation}
w=w_0\ldots w_{n-1}
\ee
 of all words occurring in lexical entries
used in the derivation of the latter sequent.

By Proposition~\ref{deduction theorem for grammars} there are lexical entries
\[
[w_\mu]^{i_\mu}_{j_\mu}:(a_{(\mu)})_{i_\mu}^{j_\mu}\in \widetilde{\mathit{Lex}}_{\equiv_\alpha},
\]
 where
\[
(a_{(\mu)})_i^j=||A_\mu||_{[\nicefrac{i}{l}]}^{[\nicefrac{j}{r}]},~(w_\mu,A_\mu)\in \mathit{Lex},~
 \mu=0,\ldots,n-1,
 \]
  satisfying (\ref{orthogonality of axioms}) and a term $t'$, which is a product of Kronecker deltas, such that
\[
\vdash_{\bf ETTC}t':: s^{l}_{r},\overline{(a_{(n-1)})}_{j_{n-1}}^{i_{n-1}},\ldots,\overline{(a_{(0)})}_{j_0}^{i_0}
\mbox{ and }t'[w_0]^{i_0}_{j_0}\ldots [w_{n-1}]^{i_{n-1}}_{j_{n-1}}\equiv_\beta [w]^r_l.
\]
 It follows from (\ref{concatenation}) that  $t'\equiv_\beta\delta_l^{j_0}\delta_{i_0}^{j_1}\cdots\delta_{i_{n-2}}^{j_{n-1}}\delta_{i_{n-1}}^r$.
and we have
 \[
\vdash_{\bf ETTC}\delta_l^{j_0}\delta_{i_0}^{j_1}\cdots\delta_{i_{n-2}}^{j_{n-1}}\delta_{i_{n-1}}^r:: s^{l}_{r},\overline{(a_{(n-1)})}_{j_{n-1}}^{i_{n-1}},\ldots,\overline{(a_{(0)})}_{j_0}^{i_0}.
\]
The above is just ($\alpha$-equivalent to) the $\eta$-long translation of
$$\vdash s[l;r],\overline{A_{n-1}}[r;c_{n-1}],\overline{A_{n-2}}[c_{n-1};c_{n-2}],\ldots, \overline{A_{(0)}}[c_0;l].$$
 The latter is identified as the image of a {\bf MILL1} sequent expressing, by definition, that $w\in L(G)$.
  The opposite inclusion is similar. \end{proof}

  We note that, compared with tensor grammars of this paper, (provisional) Definition \ref{first order grammar} of {\bf MILL1} grammars is too restrictive. It allows only lexical entries corresponding to single words. Tensor grammars, on the contrary, can have lexical entries corresponding to word tuples.

  In particular, we do not have an inverse translation from tensor grammars to {\bf MILL1} grammars. However, this does not mean that {\bf ETTC} is more expressive than {\bf MILL1}. In fact, Lemma~\ref{FO2ETTC} shows that {\bf ETTC} is {\it strictly equivalent} to the linguistic fragment of {\bf MLL1}, and the intuitionistic fragment of {\bf ETTC} equivalent to the linguistic fragment of {\bf MILL1}. We discussed in Section~\ref{first order grammar section} that the simple definition of {\bf MILL1} grammars should be generalized to allow more complex lexical entries. Tensor grammars of this paper can be considered as such a generalization, only expressed in a different syntax. The advantage of the {\bf ETTC} syntax, arguably, is that {\bf ETTC} contains the term component for explicit manipulations with strings. Using this component we managed to give a concise definition for tensor grammars. Translating this definition back to {\bf MILL1} might be cumbersome.

\section{Conclusion}
We defined and studied the system of {\it extended tensor type calculus} ({\bf ETTC}),  a non-trivial notationally enriched variation of the system previously introduced in \cite{Slavnov_tensor}. We presented a cut-free sequent calculus formalisation and a natural deduction version of the system. We also discussed geometric representation of {\bf ETTC} derivations. We showed that the system of {\bf ETTC} is strictly equivalent to the fragment of first order linear logic relevant for language modeling, that is, to representing categorial grammars, used in \cite{MootPiazza}, \cite{Moot_extended}, \cite{Moot_comparing}, \cite{Moot_inadequacy}. In this way we provided the fragment in question with an alternative syntax,  intuitive geometric representation and an {\it intrinsic} deductive system, which has been absent.

We leave for future research open questions such as: what are proof-nets for {\bf ETTC}, is there a complete semantics, what is the complexity of the calculus? We also think that the system can be further enriched/modified to go beyond first order linear logic. This is the subject of  on-going work.

\bibliographystyle{alphaurl}
\bibliography{FO_generating_bibliography}

 \end{document}